\documentclass[11pt]{article}
\usepackage[final]{neurips_2023} 

\usepackage{natbib}
\usepackage{graphicx}
\usepackage[dvipsnames]{xcolor}

\newcommand{\Expthin}[1]{\mathbb{E}[#1]}

\usepackage{nicematrix}

\newcommand{\squeeze}{\textstyle}

\usepackage[utf8]{inputenc} 
\usepackage[T1]{fontenc}    
\usepackage{hyperref}       
\usepackage{url}            
\usepackage{booktabs}       
\usepackage{amsfonts}       
\usepackage{nicefrac}       
\usepackage{microtype}      
\usepackage{xcolor}         
\hypersetup{
  colorlinks   = true, 
  urlcolor     = blue, 
  linkcolor    = {blue!50!black}, 
  citecolor   = {blue!50!black} 
}
\usepackage[flushleft]{threeparttable}
\usepackage{colortbl} 
\usepackage{cellspace}
\usepackage{multirow}

\usepackage{cancel}

\usepackage[colorinlistoftodos,bordercolor=orange,backgroundcolor=orange!20,linecolor=orange,textsize=scriptsize]{todonotes}

\newcommand*\colourcheck[1]{%
	\expandafter\newcommand\csname gcmark\endcsname{\textcolor{#1}{\ding{52}}}%
}
\newcommand*\colourxmark[1]{%
	\expandafter\newcommand\csname rxmark\endcsname{\textcolor{#1}{\ding{55}}}%
}
\definecolor{mygreen}{HTML}{02862a}
\definecolor{myred}{HTML}{9a0000}
\colourcheck{mygreen}
\colourxmark{myred}
\definecolor{linen}{HTML}{FAF0E6} 
\definecolor{darkteal}{RGB}{0, 110, 110}

\usepackage{xspace}
\newcommand{\algname}[1]{{\sf {\color{darkteal}\footnotesize #1}}\xspace}
\newcommand{\algnamesmall}[1]{{\sf {\color{darkteal}\scriptsize #1}}\xspace}

\input{preamble.tex}
\usepackage{algorithm}
\usepackage{algpseudocode}

\title{Momentum Provably Improves Error Feedback!}

\allowdisplaybreaks
\author{Ilyas Fatkhullin \\ ETH AI Center \& ETH Zurich \and \textbf{Alexander Tyurin}\\ KAUST\thanks{King Abdullah University of Science and Technology, Thuwal, Saudi Arabia.} \and \textbf{Peter Richt\'{a}rik} \\KAUST }

\begin{document}

\maketitle

\begin{abstract}

Due to the high communication overhead when training machine learning models in a distributed environment, modern algorithms invariably rely on lossy communication compression. However, when untreated, the errors caused by compression propagate, and can lead to severely unstable behavior, including exponential divergence. Almost a decade ago, \citet{Seide2014} proposed an error feedback (EF) mechanism, which we refer to as \algname{EF14}, as an immensely effective heuristic for mitigating this issue. However, despite steady algorithmic and theoretical  advances in the EF field in the last decade, our understanding is far from complete. In this work we address one of the most pressing issues. In particular, in the canonical nonconvex setting, all known variants of EF rely on very large batch sizes to converge, which can be prohibitive in practice. We propose a surprisingly simple fix which removes this issue both theoretically, and in practice: the application of Polyak's momentum to the latest incarnation of EF due to \citet{EF21}  known as \algname{EF21}. Our algorithm, for which we coin the name  \algname{EF21-SGDM}, improves the communication and sample complexities of previous error feedback algorithms under standard smoothness and bounded variance assumptions, and does not require any further strong assumptions such as bounded gradient dissimilarity. Moreover, we propose a double momentum version of our method that improves the complexities even further. Our proof seems to be novel even when compression is removed from the method, and as such, our proof technique is of independent interest in the study of nonconvex stochastic optimization enriched with Polyak's momentum.



\end{abstract}


\tableofcontents
\newpage 

\section{Introduction}
Since the practical utility of modern machine learning models crucially depends on our ability to train them on large quantities of training data,  it is  imperative  to perform the training in a distributed storage and compute environment. In federated learning (FL) \citep{FEDLEARN,FL-big}, for example, data is naturally stored in a distributed fashion across a large number of clients (who capture and own the data in the first place), and the goal is to train a single machine learning model from the wealth of all this distributed data, in a private fashion, directly on their devices.

\phantom{XX} {\bf 1.1 Formalism.}
We consider the problem of collaborative training of a single model by several clients in a data-parallel fashion. In particular, 
we aim to solve the  \textit{distributed nonconvex  stochastic optimization problem} 
\begin{equation}\label{eq:problem}
\squeeze	\min \limits_{x \in \R^d } \sb{ f(x) \eqdef \fr{1}{n} \sum \limits_{i = 1}^{n} f_i(x)}, \qquad f_i(x) \eqdef \Expu{\xi_i\sim \cD_i}{f_i(x, \xi_i)},  \qquad i = 1, \ldots, n, 
\end{equation}
where $n$ is the number of clients, $x\in \R^d$ represents the parameters of the model we wish to train, and $f_i(x)$ is the (typically nonconvex) loss  of  model parameterized by the vector $x$  on the data $\cD_i$ owned by client $i$. Unlike most works in federated learning, we do not assume the datasets to be similar, i.e., we allow the distributions $\cD_1, \dots, \cD_n$ to be arbitrarily different. 

We are interested in the fundamental problem of finding an approximately stationary point of $f$ in expectation, i.e., we wish to find a (possibly random) vector $\hat{x}\in \R^d$ such that $\Exp{\|\nabla f(\hat{x}) \|} \leq \varepsilon$. In order to solve this problem, we assume that the $n$ clients communicate via an orchestrating server. Typically, the role of the server is to first perform aggregation of the messages obtained from the workers, and to subsequently broadcast the aggregated information back to the workers. Following an implicit assumption made in virtually all theoretically-focused papers on communication-efficient training, we also assume that the speed of client-to-workers broadcast is so fast (compared to speed of workers-to-client communication) that the cost associated with broadcast can be neglected\footnote{While this is a reasonable assumption in many practical situations~\citep{DIANA,FL-big}, some works consider the regime when the server-to-workers broadcast cannot be neglected~\citep{Cnat,DoubleSqueeze,Artemis2020,D-DIANA,EF21BW_2021,Gruntkowska_EF21_P_2022}.}. 

\phantom{XX} {\bf 1.2 Aiming for communication and computation efficiency at the same time.}
In our work, we pay attention to two key aspects of efficient distributed training---{\em communication cost} and {\em computation cost} for finding an approximate stationary point $\hat{x}$. The former refers to the number of bits that need to be communicated by the workers to the server, and the latter refers to the number of stochastic gradients that need to be sampled by each client. The rest of the paper can be summarized as follows: {\em We pick one of the most popular communication-efficient gradient-type methods (the \algname{EF21} method of \citet{EF21} -- the latest variant of error feedback pioneered by \citet{Seide2014})   and modify it in a way which provably preserves its communication complexity, but massively improves its computation/sample complexity, both theoretically and in practice. 
}



\section{Communication Compression, Error Feedback, and Sample Complexity}



Communication compression techniques such as {\em quantization} \citep{alistarh2017qsgd,Cnat} and {\em sparsification} \citep{Seide2014,beznosikov2020biased} are known to be immensely powerful for reducing the communication footprint of gradient-type\footnote{For Newton-type methods, see \citep{Islamov_Newton_3PC_2022} and references therein.} methods. Arguably the most studied, versatile and practically useful class of compression mappings are {\em contractive} compressors. 

 \begin{definition}[Contractive compressors]\label{def:contractive_compressor}
 	We say that a (possibly randomized) mapping $\cC: \R^{d} \to \R^{d}$ is a  contractive compression operator if there exists a constant $0<\alpha\leq 1$ such that 
 	\begin{eqnarray}\label{eq:b_compressor}
 		\Exp{\|\cC(x) - x\|^{2}} \leq \rb{1 - \alpha} \|x\|^{2}, \qquad \forall x\in \R^d.
 	\end{eqnarray}
 \end{definition}
 
Inequality \eqref{eq:b_compressor} is satisfied by a vast array of compressors considered in the literature, including numerous variants of sparsification operators \citep{Alistarh-EF-NIPS2018,Stich-EF-NIPS2018}, quantization operators \citep{alistarh2017qsgd,Cnat}, and  low-rank  approximation~\citep{PowerSGD,FedNL} and more~\citep{beznosikov2020biased,UP2021}. The canonical examples 
are i) the Top$K$ sparsifier, which preserves the $K$  largest   components  of $x$ in magnitude and sets all remaining coordinates to zero~\citep{Stich-EF-NIPS2018}, and ii) the (scaled) Rand$K$ sparsifier, which preserves a subset of $K$ components of $x$ chosen uniformly at random and sets all remaining coordinates to zero \citep{beznosikov2020biased}. In both cases, \eqref{eq:b_compressor} is satisfied with $\alpha = \nicefrac{K}{d}$.

\subsection{Brief history of error-feedback}
 When greedy contractive compressors, such as Top$K$, are used in a direct way to compress the local gradients in distributed gradient descent (\algname{GD}), the resulting method may diverge exponentially, even on strongly convex quadratics~\citep{beznosikov2020biased}.  Empirically, instability caused by such a naive application of greedy compressors was observed much earlier, and a fix was proposed in the form of the \textit{error feedback} (EF) mechanism by \citet{Seide2014}, which we henceforth call \algname{EF14} or \algname{EF14-SGD} (in the stochastic case).\footnote{In Appendix~\ref{sec:appendix_literature_momentum}, we provide a more detailed discussion on theoretical develepments for this method.}
 To the best of our knowledge, the best {\em sample complexity} of \algname{EF14-SGD} for finding a stationary point in the distributed nonconvex setting is given by \citet{Koloskova2019DecentralizedDL}:  after $\cO (  G  \al^{-1} \varepsilon^{-3}  +  \sigma^2 n^{-1} \varepsilon^{-4} )$ samples\footnote{Here $\sigma^2$ is the bound on the variance of stochastic gradients at each node, see Assumption~\ref{as:BV}. 
When referring the sample complexity we count the number of stochastic gradients used only at one of the $n$ nodes rather than by all nodes in total. This is a meaningful notion because the computations are done in parallel.}, \algname{EF14-SGD} finds a point $x$ such that $\Expthin{\norm{\nabla f(x) }} \leq \varepsilon$, where $\alpha$ is the contraction parameter (see Definition~\ref{def:contractive_compressor}). However, such an analysis has two important deficiencies. First, in the deterministic case (when exact gradients are computable by each node), the analysis only gives the suboptimal $\cO(\varepsilon^{-3})$ {\em iteration complexity}, which is suboptimal compared to vanilla (i.e., non-compressed) gradient descent, whose iteration complexity is $\cO(\varepsilon^{-2})$. Second, their analysis relies heavily  on additional strong assumptions, such as the {\em  bounded gradient} (BG) assumption, $\Expthin{\|\nabla f_i(x, \xi_i) \|^2 } \leq G^2$ for all $x\in \R^d$, $i\in [n]$, $\xi_i \sim \cD_i$, or the bounded gradient similarity (BGS) assumption,  
 $ \suminn \| \nabla f_i(x) - \nabla f(x) \|^2 \leq G^2$ for all $x\in \R^d$. Such assumptions are restrictive and sometimes even unrealistic. In particular, both BG and BGS might not hold even in the case of convex quadratic functions.\footnote{For example, one can consider $f_i(x) = x^{\top} A_i x$ with $A_i \in \R^{d\times d}$, for which BG or BGS assumptions hold only in the trivial cases: matrices $A_i$ are all zero or all equal to each other (homogeneous data regime).}
 Moreover, it was recently shown that nonconvex analysis of  stochastic gradient methods using a BG assumption may hide an exponential dependence on the smoothness constant in the complexity~\citep{Yang_Two_Sides_2023}.  
 
In 2021, these issues were {\em partially} resolved by \citet{EF21}, who propose a modification of the EF mechanism, which they call \algname{EF21}. They address both deficiencies of the original \algname{EF14} method: i) they removed the BG/BGS assumptions, and improved the iteration complexity to $\cO(\varepsilon^{-2})$ in the full gradient regime. Subsequently, the \algname{EF21} method was modified in several directions, e.g., extended to bidirectional compression, variance reduction and proximal setup~\citep{EF21BW_2021}, generalized from contractive to three-point compressors~\citep{3PC} and adaptive compressors~\citep{Makarenko_AdaCGD_2022}, modified from dual (gradient) to primal (model)  compression~\citep{Gruntkowska_EF21_P_2022} and from centralized to decentralized setting~\citep{Zhao_BEER_2022}. For further work, we refer to \citep{Wang_CD_ADAM-AMSGrad_2022,Dorfman_DoCoFL_2023,Islamov_Newton_3PC_2022}. 
   
\subsection{Key issue: error feedback has an unhealthy appetite for samples!}\label{subsec:EF_Stoch}
Unfortunately, the current theory of \algname{EF21} with \emph{stochastic gradients} has weak sample complexity guarantees. In particular, \citet{EF21BW_2021} extended the \algname{EF21-GD} method, which is the basic variant of \algname{EF21} using full gradient at the clients, to  \algname{EF21-SGD}, which uses a ``large minibatch'' of stochastic gradients instead. They obtained   $\cO(  \fr{1}{\al \varepsilon^{2}} +  \fr{\sigma^2}{\al^3 \varepsilon^{4}} )$ sample complexity for their method. Later, \citet{Zhao_BEER_2022}  improved this result slightly\footnote{The result was obtained under a more general setting of decentralized optimization over a network.}  to $\cO (  \fr{1}{\al \varepsilon^{2}} +  \fr{\sigma^2}{\al^2 \varepsilon^{4}} )$, shaving off one $\alpha$ in the stochastic term. However, it is easy to notice several issues in these results, which generally feature the fundamental challenge of combining biased gradient methods with stochastic gradients.

\begin{figure}
	\centering
	\begin{subfigure}{.5\textwidth}
		\centering
		\includegraphics[width=0.85\linewidth]{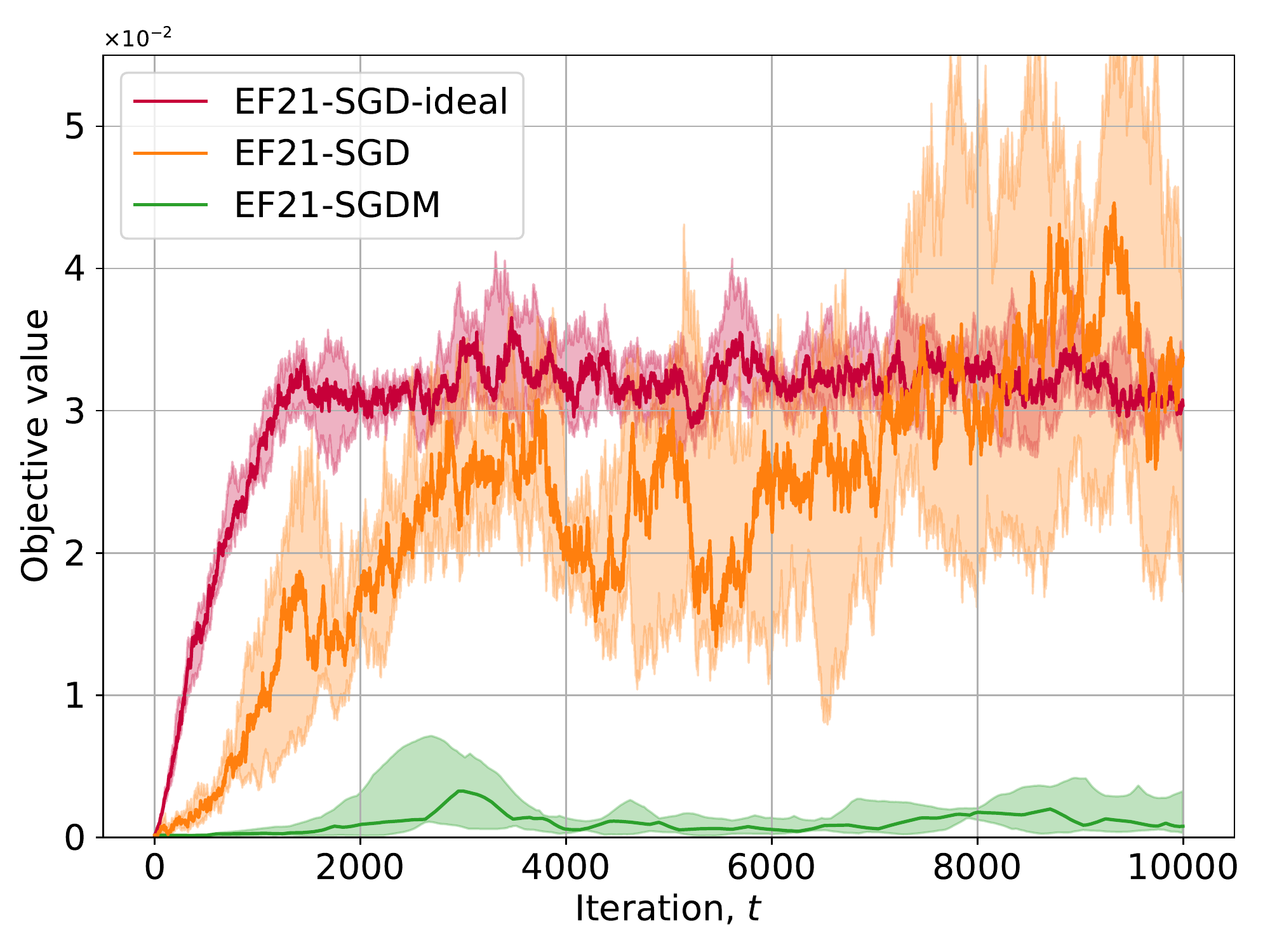}
		\caption{Divergence for $n=1$. }
		\label{fig:diverg_const_sz}
	\end{subfigure}%
	\begin{subfigure}{.5\textwidth}
		\centering
		\includegraphics[width=0.85\linewidth]{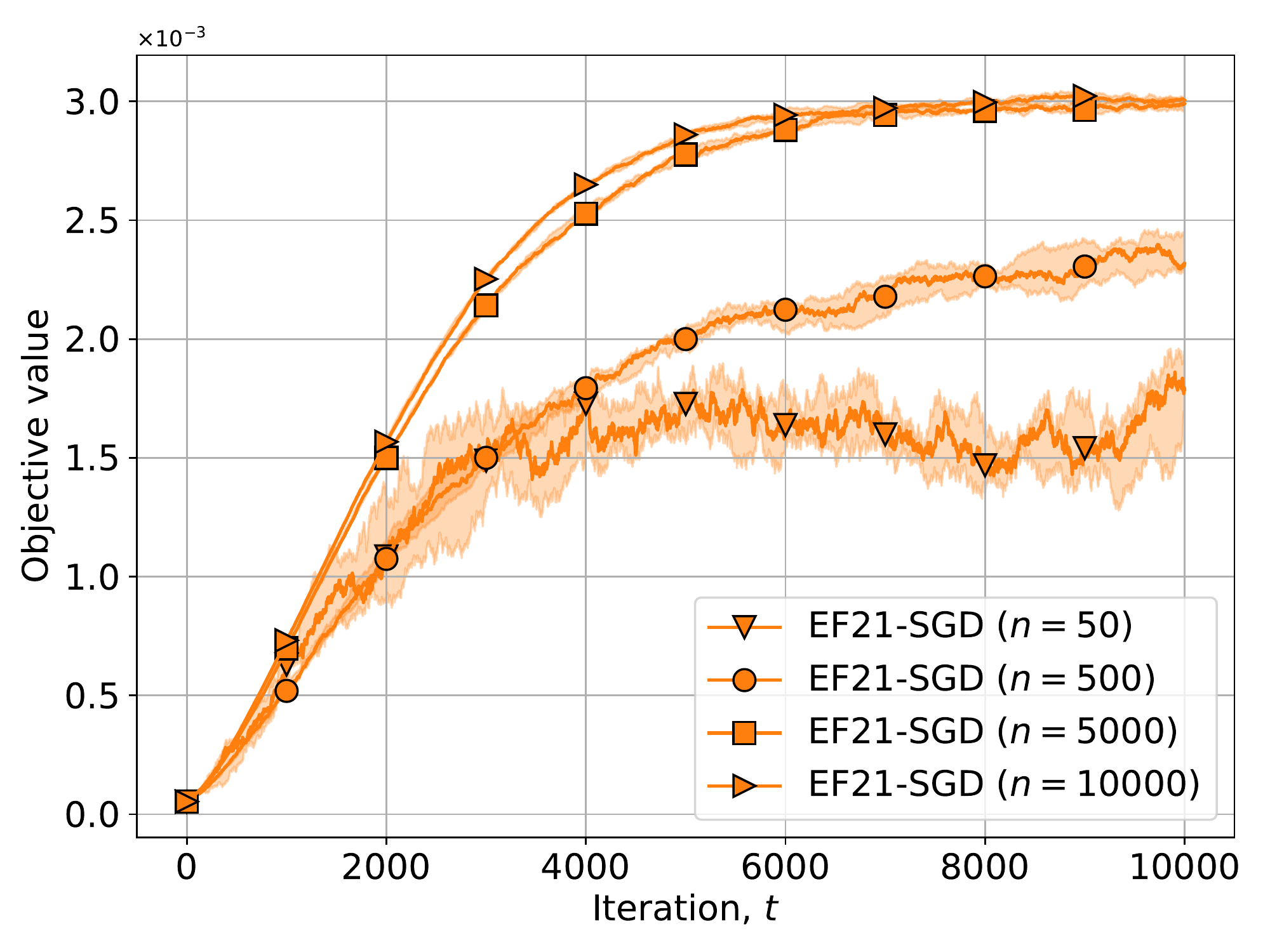}
		\caption{No improvement with $n$.}
		\label{fig:no_speedup_const_sz}
	\end{subfigure}
	\caption{\footnotesize Divergence of \algnamesmall{EF21-SGD} on the quadratic function $f(x)=\frac{1}{2}\sqnorm{x}$, $x\in \R^2$, using the Top$1$ compressor. See the proof of Theorem~\ref{thm:non_convergence_ef21_like_SGD} for details on the construction of the noise $\xi$; we use $\sigma = 1$, $B = 1$. The starting point  is $x^0  = (0,-0.01)^{\top}$. Unlike \algnamesmall{EF21-SGD}, our method \algnamesmall{EF21-SGDM} does not suffer from divergence and is stable near optimum. Figure~\ref{fig:no_speedup_const_sz} shows that when increasing the number of nodes $n$, \algnamesmall{EF21-SGD} applied with $B=1$ does not improve, and, moreover, diverges from the optimum even faster. All experiments use constant parameters $\gamma = \eta = \nicefrac{0.1}{\sqrt{T}} = 10^{-3}$; see Figure~\ref{fig:divergence_var} for diminishing parameters.  Each method is run $10$ times and the plot shows the median performance alongside the $25\%$ and $75\%$ quantiles. }
	\label{fig:divergence}
\end{figure}

\textbf{$\bullet$ Mega-batches.} These works require all clients to sample ``mega-batches'' of stochastic gradients/datapoints in each iteration, of order $\cO(\varepsilon^{-2})$, in order to control the variance coming from stochastic gradients. In Figure~\ref{fig:divergence}, we find that, in fact, a batch-free (i.e., with mini-batch size $B=1$) version of \algname{EF21-SGD} diverges even on a very simple quadratic function. We also observe a similar behavior when a small batch $B > 1$ is applied.  This implies that there is a fundamental flaw in the \algname{EF21-SGD} method itself, rather ``just'' a problem of the theoretical analysis. While mega-batch methods are common in optimization literature, smaller batches are often preferred whenever they ``work''. For example, the time/cost required to obtain such a large number of samples at each iteration might be unreasonably large compared to the communication time, which is already reduced using compression. Moreover, when dealing with medical data, large batches might simply be unavailable~\citep{Rieke_Dig_Health_FL_2020}. In certain applications, such as federated reinforcement learning (RL) or multi-agent RL, it is often intractable to sample more than one trajectory of the environment in order to form a gradient estimator \citep{Mitra_TD_EF_2023,Doan_FT_TD_4_MARL_2019,Jin_FedRL_Env_Heterogeneity_2022,Khodadadian_FedRL_Linear_speed_2022}. 
Further, a method using a mega-batch at each iteration effectively follows the gradient descent (\algname{GD}) dynamics instead of the dynamics of (mini-batch) \algname{SGD}, which may hinder the training and generalization performance of such algorithms since it is both empirically  \citep{Keskar_OnLargeBatch_DL_2017,Kleinberg_When_SGD_Escape_Local_2018} and theoretically \citep{Kale_SGD_Role_of_Implicit_Reg_2021} observed that mini-batch \algname{SGD} is superior to mega-batch \algname{SGD} or \algname{GD} in a number of machine learning tasks.

\textbf{$\bullet$ Dependence on $\alpha$.} The total sample complexity results derived by \citet{EF21BW_2021,Zhao_BEER_2022} suffer from poor dependence on the contraction parameter $\alpha$. 
Typically, EF methods are used with the  Top$K$ sparsifier, which only communicates $K$ largest entries in magnitude. In this case, $\alpha = \nicefrac{K}{d}$, and the stochastic part of sample complexity scales quadratically with dimension.

\textbf{$\bullet$ No improvement with $n$.} The stochastic term in the sample complexity of \algname{EF21-SGD} does {\em not} improve when increasing the number of nodes. However, the opposite behavior is typically desired, and is present in several latest non-EF methods based on {\em unbiased} compressors, such as \algname{MARINA} \citep{MARINA} and \algname{DASHA} \citep{DASHA_2022}.  We are not aware of any distributed algorithms utilizing the Top$K$ compressor  achieving  linear speedup in $n$ in the stochastic term without relying on restrictive BG or BGS assumptions.

These observations motivate our work with the following central questions:
\begin{quote}
\em
   Can we design a batch-free distributed \algname{SGD} method utilizing contractive communication compression (such as Top$K$) without relying on restrictive BG/BGS assumptions? Is it possible to improve over the current state-of-the-art $\cO\rb{  \alpha^{-1} \varepsilon^{-2} +  \sigma^2 \alpha^{-2} \varepsilon^{-4} }$ sample complexity under the standard smoothness and bounded variance assumptions? 
\end{quote} 
We answer both questions in the affirmative by incorporating a momentum update into \algname{EF21-SGD}.  

\subsection{Mysterious effectiveness of momentum in nonconvex optimization}
  An immensely popular modification of \algname{SGD} (and its distributed variants) is the use of {\em momentum}. This technique, initially inspired by the developments in convex optimization \citep{Polyak_Some_methods_1964}, is often applied in machine learning for stabilizing convergence and speeding up the training. In particular, momentum is an important part of an immensely popular and empirically successful line of adaptive methods for deep learning, including \algname{ADAM}~\citep{ADAM} and a plethora of variants. The classical \algname{SGD} method with Polyak (i.e., heavy ball) momentum (\algname{SGDM}) reads:
 \begin{equation}
 	x^{t+1} = x^t - \gamma v^t , \qquad	v^{t+1} = (1-\eta) v^{t} + \eta \nabla f(x^{t+1}, \xi^{t+1}),
 \label{eq:HB}\end{equation}	
where $\gamma>0$ is a learning rate and $\eta>0$ is the momentum parameter.  

We provide a concise walk through the key theoretical developments in the analysis of \algname{SGDM} in stochastic nonconvex optimization in Appendix~\ref{sec:appendix_literature_momentum}; and only mention the most relevant works here.  The most closely related  works to ours are \citep{DIANA}, \citep{CSER}, and \citep{EF21BW_2021}, which analyze momentum together with communication compression. The analysis in \citep{DIANA,CSER} requires BG/BGS assumption, and does not provide any theoretical improvement over the variants without momentum. Finally, the analysis of \citet{EF21BW_2021} is only established for deterministic case, and it is unclear if its extension to stochastic case can bring any convergence improvement over \algname{EF21-SGD}.  Recently, several other works attempt to explain the benefit of momentum~\citep{Plattner_SGDM_Thesis_2022}; some  consider structured nonconvex problems~\citep{Wang_Quickly_Finding_HB_2021}, and others focus on generalization \citep{Jelassi_Mom_Improves_Generalization_2022}.

\phantom{XX} {\bf Summary of contributions.}
 Despite the vast amount of work trying to explain the benefits of momentum,  there is no work obtaining any theoretical improvement over vanilla \algname{SGD} in the smooth nonconvex setting under the standard assumptions of smoothness and bounded variance.  

  \begin{table*}[t]
 	\caption{\footnotesize Summary of related works on distributed error compensated SGD methods using a Top$K$ compressor under Assumptions~\ref{as:main} and~\ref{as:BV}. The goal is to find an $\varepsilon$-stationary point of a smooth nonconvex function of the form \eqref{eq:problem}, i.e., a point $x$ such that $\Exp{\norm{\nabla f(x)}} \leq \varepsilon$. "\textbf{Communication complexity}": the total \# of communicated bits if the method is applied with sufficiently large batch-size; see Table~\ref{table:related-works} for batch-size. "\textbf{Asymptotic sample complexity}": the total \# of samples required at each node to find an $\varepsilon$-stationary point for batch-size $B = 1$ in the regime $\varepsilon\rightarrow 0$.
 		"\textbf{No extra assumptions}": \gcmark means that  no additional assumption is required. 
 	}
 	\label{table:related-works_small}
 	\centering
	\scriptsize
 	\begin{threeparttable}
 		\begin{tabular}{|c|c|c|c|c|}
 			\hline
 			\bf Method & \bf \makecell{Communication complexity}  & \bf \makecell{Asymptotic \\sample  complexity}  & \bf \makecell{Batch-free?}   &\bf \makecell{No extra  assumptions?}   \\
 			\hline
 			\makecell{\algnamesmall{EF14-SGD}\\ \citep{Koloskova2019DecentralizedDL}}  &  $ \fr{K {\color{myred}{G}} }{\alpha \color{myred}{\varepsilon^3}}$  & \makecell{$  \fr{  \sigma^2}{n \varepsilon^4}$} & \gcmark  &  \makecell{\rxmark}\tnote{{\color{blue}(a)}}   \\
 			\hline	
 			\makecell{\algnamesmall{NEOLITHIC}\\ \citep{huang2022lower}}  &\makecell{$ \fr{K  }{ \alpha \varepsilon^{2} }\color{myred}{\log\left( \frac{ G }{ \varepsilon }  \right)} $}\tnote{{\color{blue}(b)}}   & \makecell{$ \fr{  \sigma^2}{n \varepsilon^4}$} & \rxmark   &  \makecell{\rxmark}\tnote{{\color{blue}(c)}}  \\
 			\hline
 			\makecell{\algnamesmall{EF21-SGD}\\ \citep{EF21BW_2021}}  & $ \fr{K }{ \alpha \varepsilon^{2} } $ & \makecell{ $ \fr{ \sigma^2}{{\color{myred}\alpha^3} \varepsilon^4}$}\tnote{{\color{blue}(d) }}
 			&  \makecell{\rxmark} &  \makecell{\gcmark}  \\
 			\hline
 			\makecell{\algnamesmall{BEER}\\ \citep{Zhao_BEER_2022}}  & $ \fr{K }{ \alpha \varepsilon^{2} }$ &   \makecell{$  \fr{ \sigma^2}{ {\color{myred} \alpha^2} \varepsilon^4}$}\tnote{{\color{blue}(d) }}
 			& \makecell{\rxmark}  &  \makecell{\gcmark}  \\
 			\hline
 			\cellcolor{linen}\begin{tabular}{c}\algnamesmall{EF21-SGDM} \, (Corollary~\ref{cor:EF21-M_conv}) \\
 				\algnamesmall{EF21-SGD2M} \, (Corollary~\ref{cor:EF21-DM_conv}) \end{tabular}  & \cellcolor{linen} $ \fr{ K  }{\al \varepsilon^2} $ & \cellcolor{linen} $   \fr{  \sigma^2    }{n \varepsilon^4}$ &\cellcolor{linen} \gcmark   &  \cellcolor{linen} \gcmark \\
 			\hline
 		\end{tabular}
 		\begin{tablenotes}
 			\scriptsize
 			 \item [{\color{blue}(a)}] Analysis requires a bound of the second moment of the stochastic gradients, i.e., $\Exp{\sqnorm{\nabla f_i(x, \xi_i)}} \leq G^2$ for all $x \in \R^d$. 
 			\item [{\color{blue}(b)}] This complexity is achieved by using a large mini-batch and communicating $\left\lceil \nicefrac{K}{\alpha} \right\rceil$ coordinates per iteration, see Appendix~\ref{sec:appendix_literature_momentum}.
 			\item [{\color{blue}(c)}] Analysis requires a bounded gradient disimilarity assumption, i.e.,  $
 			\suminn \sqnorm{\nabla f_i(x) - \nabla f(x)} \leq G^2  
 			$ for all $x \in \R^d$. 
 			 \item [{\color{blue}(d)}] Analysis requires a batch-size at least $B \geq \fr{\sigma^2}{ \alpha^2 \varepsilon^2}$ for \algnamesmall{EF21-SGD} and $B \geq \fr{\sigma^2}{ \alpha \varepsilon^2}$ for \algnamesmall{BEER}.
 			
 		\end{tablenotes}
 	\end{threeparttable}
 \end{table*} 

\phantom{X}$\bullet$ First, we establish a {\em negative result} for a simplified/idealized version of \algname{EF21-SGD}, which shows that this algorithm does not converge with constant batch-size, and that a mega-batch of order $\Omega(\sigma^2 \varepsilon^{-2})$ is required. This provides a strong indication that \algname{EF21-SGD} method is inherently sensitive to stochastic gradients, which is also confirmed by our numerical simulations. 

\phantom{X}$\bullet$  We propose a simple fix for this problem by incorporating {\em momentum} step into \algname{EF21-SGD}, which leads to our one-batch  \algname{EF21-SGDM}  algorithm. By leveraging our {\em new Lyapunov function construction and new analysis}, we establish $\cO\left(\alpha^{-1} \varepsilon^{-2} + \sigma^2 \varepsilon^{-4} \right)$ sample complexity in the single node case. 

\phantom{X}$\bullet$  We extend our algorithm to the distributed setting and derive an improved sample complexity result compared to other methods using the Top$K$ compressor without resorting to the BG/BGS assumptions. In particular,  \algname{EF21-SGDM}  achieves {\em asymptotically optimal} $\cO\rb{\sigma^2 n^{-1} \varepsilon^{-4}}$ \textit{sample complexity}. Moreover, when  \algname{EF21-SGDM}  is applied with large enough batch size, we prove that it reaches the \textit{optimal communication complexity} $\cO\rb{K \alpha^{-1} \varepsilon^{-2}}$; see Tables~\ref{table:related-works_small} \& \ref{table:related-works} for more details. 

\phantom{X}$\bullet$  Finally, we propose a {\em double momentum} variant of  \algname{EF21-SGDM}, and find that it further improves the sample complexity of  \algname{EF21-SGDM}  in the non-asymptotic regime. 
 	



We highlight that, interestingly, we {\em prove that momentum helps}: \algname{EF21-SGDM} is theoretically better compared to its non-momentum variant -- large-batch \algname{EF21-SGD}. We believe that our new technique can be extended in many ways, e.g., to dealing with other biased compressors and other (biased) communication saving techniques such as lazy aggregation of gradients, model compression, bidirectional compression, partial participation, decentralized training, adaptive compression; other important optimization techniques involving biased updates such as proximal \algname{SGD} with momentum, gradient clipping and adaptive step-size schedules. We also hope that our proof techniques can be useful to establish linear speedup for other classes of distributed methods, e.g, algorithms based on local training 
\algname{ProxSkip}/\algname{Scaffnew} \citep{Mishchenko_Proxskip_2022}.

Additionally, we extend our results to the class of absolute compressors in Appendix~\ref{sec:appendix_abs} and study the variance reduced variant of error feedback in Appendix~\ref{sec:appendix_STORM}. 

 \section{Main Results}\label{sec:main_results}
Throughout the paper we work under the following standard assumptions.  
  \begin{assumption}[Smoothness and lower boundedness]\label{as:main}	 
 	We assume that $f$ has $L$-Lipschitz gradient, i.e., $\norm{\nabla f(x) - \nabla f(y)} \leq L\norm{x - y}$ for all $x, y\in \R^d$, and each $f_i$ has $L_i$-Lipschitz gradient, i.e., $\norm{\nabla f_i(x) - \nabla f_i(y)} \leq L_i\norm{x - y}$ for all $i \in [n],$ $x, y\in \R^d$. We denote $\wL^2 \eqdef \suminn L_i^2$.	Moreover, we assume that $f$ is lower bounded, i.e., $f^* \eqdef \inf_{x\in \R^d} f(x)>-\infty $.
 \end{assumption}
 
 \begin{assumption}[Bounded variance (BV)] \label{as:BV} There exists $\sigma>0$ such that 
 	\begin{eqnarray}\label{eq:BV}
 		\Exp{ \sqnorm{\nabla f_{i}(x, \xi_i) - \nabla f_i(x)} } \leq \sigma^2, \qquad \forall x\in \R^d, \qquad \forall i\in [n],
 	\end{eqnarray}	
 where $ \xi_i \sim \cD_i$ are i.i.d.\ random samples for each $i \in [n]$.
 \end{assumption}

\subsection{A deeper dive into the issues \algname{EF21} has with stochastic gradients}
As remarked before, the current analysis of \algname{EF21} in the stochastic setting requires each client to sample a mega-batch in each iteration, and it is not clear how to avoid this. In order to understand this phenomenon, we propose to step back and examine an ``idealized'' version of \algname{EF21-SGD}, which we call \algname{EF21-SGD-ideal}, defined by the update rules \eqref{eq:ef21_like_sgd_1} + \eqref{eq:ef21_like_sgd_2}:
\begin{subequations}
	\begin{alignat}{2}
	\label{eq:ef21_like_sgd_1}
              x^{t+1} &=  x^t - \gamma g^t, \quad  g^t &&\squeeze = \frac{1}{n}\sum \limits_{i=1}^n g_i^t  \\
		\label{eq:ef21_like_sgd_2}
		\text{\algname{EF21-SGD-ideal:}}\qquad g_i^{t+1} &= \textcolor{blue}{ \nabla f_i(x^{t+1})} &&+ \cC\left( \nabla f_i(x^{t+1}, \xi_i^{t+1}) - \textcolor{blue}{ \nabla f_i(x^{t+1}) } \right) \tag{\theequation a} , \\
		\label{eq:ef21_sgd}
		\text{\algname{EF21-SGD:}}\qquad g_i^{t+1} &= \qquad \textcolor{blue}{ g_i^t } &&+ \cC\left( \nabla f_i(x^{t+1}, \xi_i^{t+1}) -  \qquad \textcolor{blue}{ g_i^t } \right). \tag{\theequation b}
	\end{alignat}
\end{subequations}

Compared to \algname{EF21-SGD}, given by \eqref{eq:ef21_like_sgd_1} + \eqref{eq:ef21_sgd}, we replace the previous state $g_i^{t}$ by the {\em exact gradient} at the current iteration. 
Since \algname{EF21-SGD}  heavily relies on the approximation $g_i^{t} \approx \nabla f_i(x^{t+1})$, and according to the proof of convergence of \algname{EF21-SGD}, such discrepancy tends to zero as $t\rightarrow \infty$, this change can only improve the method. While we admit this is  a conceptual algorithm only (it does not lead to any communication or sample complexity reduction in practice)\footnote{This is because full gradients would need to be computed and communicated for its implementation. Notice also that if $\sigma= 0$, this method becomes the exact distributed gradient descent.}, it  serves us well to illustrate the drawbacks of \algname{EF21-SGD}. We now establish the following negative result for \algname{EF21-SGD-ideal}.
\begin{theorem}\label{thm:non_convergence_ef21_like_SGD}
	Let $L$, $\sigma >0$, $0<\gamma\leq \nicefrac{1}{L}$ and $n=1.$ There exists a convex, $L$-smooth function $f:\R^2\to \R$, a contractive compressor $\cC(\cdot)$ satisfying Definition~\ref{def:contractive_compressor}, and an unbiased stochastic gradient with bounded variance $\sigma^2$ such that if the method \algname{EF21-SGD-ideal} (\eqref{eq:ef21_like_sgd_1} + \eqref{eq:ef21_like_sgd_2}) is  run with step-size $\gamma$, then for all $T \geq 0$ and for all $x^0 \in \{(0,  x_{(2)}^{0})^{\top} \in \R^2 \,|\, x_{(2)}^0 < 0\},$ we have
	$$
	\squeeze \Exp{\sqnorm{ \nabla f(x^T) } } \geq  \frac{1}{60} \min\left\{\sigma^2, \sqnorm{ \nabla f(x^0) }\right\} . 
	$$
	Fix $0 < \varepsilon \leq \nicefrac{L}{\sqrt{60}}$ and $x^0 = (0,  -1)^{\top}.$ Additionally assume that $n \geq 1$ and the variance of unbiased stochastic gradient is controlled by $\nicefrac{\sigma^2}{B}$ for some $B\geq1$. If $B < \frac{ \sigma^2}{60 \varepsilon^2}$, then we have $\Exp{\norm{\nabla f(x^T)}} > \varepsilon $ for all $T \geq 0$. 
\end{theorem}	

The above theorem implies that the method \eqref{eq:ef21_like_sgd_1}, \eqref{eq:ef21_like_sgd_2}, does not converge with small batch-size (e.g., equal to one) for any fixed step-size choice.\footnote{In fact, the example can be easily extended to the case of polynomially decaying step-size. } Moreover, in distributed setting with $n$ nodes, a mini-batch of order $B = \Omega\rb{ \nicefrac{\sigma^2}{ \varepsilon^2}}$ is required for convergence. Notice that this batch-size is independent of $n$, which further implies that a linear speedup in the number of nodes $n$ cannot be achieved for this method. While we only prove these negative results for an "idealized" version of \algname{EF21-SGD} rather than for the method itself, in Figures~\ref{fig:diverg_const_sz} and \ref{fig:diverg_var_sz}, we empirically verify that \algname{EF21-SGD} also suffers from a similar divergence on the same problem instance provided in the proof of Theorem~\ref{thm:non_convergence_ef21_like_SGD}. Additionally, Figures~\ref{fig:no_speedup_const_sz} and~\ref{fig:no_speedup_var_sz} illustrate that the situation does not improve for \algname{EF21-SGD} when increasing $n$.

\subsection{Momentum for avoiding mega-batches}
Let us now focus on the single node setting\footnote{In this case, we can drop index $i$ everywhere and write $g_i^t = g^t$,\, $\xi_i^t = \xi^t$ for all $t\geq0$.} and try to fix the divergence issue shown above. As we can learn from Theorem~\ref{thm:non_convergence_ef21_like_SGD}, the key reason for non-convergence of \algname{EF21-SGD} is that even if the state vector $g^{t}$ sufficiently approximates the current gradient, i.e., $g^{t} \approx \nabla f (x^{t+1})$, the design of this method cannot guarantee that the quantity
$
\sqnorm{ g^{t} - \nabla f (x^t) } \approx \sqnorm{ \cC\rb{ \nabla f (x^t, \xi^{t}) - \nabla f (x^t) } }  
$
is small enough. Indeed, the last term above can be bounded by $2 (2 - \alpha) \sigma^2$ in expectation, but it is not sufficient as formally illustrated in Theorem~\ref{thm:non_convergence_ef21_like_SGD}. 
To fix this problem, we propose to modify our ``idealized'' \algname{EF21-SGD-ideal} method so that the compressed difference can be controlled and made arbitrarily small, which leads us to another (more advanced) conceptual algorithm, 
\begin{align}
	\text{\algname{EF21-SGDM-ideal:}}&
	\begin{split}
		v^{t+1} &= \textcolor{mygreen}{\nabla f(x^{t+1}) } + \eta (\nabla f(x^{t+1}, \xi^{t+1})  - \textcolor{mygreen}{\nabla f(x^{t+1}) }  ) , \\
		g^{t+1} &= \textcolor{blue}{ \nabla f(x^{t+1}) } + \cC\rb{  v^{t+1}   - \textcolor{blue}{ \nabla f(x^{t+1}) } } . 
	\end{split} 
	\label{eq:EF21-SGDMv0} 
\end{align}
In this method, instead of using $v^{t+1} = \nabla f(x^{t+1}, \xi^{t+1}) $ as in \algname{EF21-SGD-ideal}, we introduce a correction, which allows to control variance of the difference $\nabla f(x^{t+1}, \xi^{t+1}) - \nabla f(x^{t+1})$. This allows us to derive the following convergence result. Let $\delta_0  \eqdef f(x^0) - f^*$.
\begin{proposition}\label{prop:ef21_m_SGDv0}
	Let Assumptions~\ref{as:main}, \ref{as:BV} hold, and let $\cC$ satisfy Definition~\ref{def:contractive_compressor}. Let $g^0 = 0$ and the step-size in method~\eqref{eq:ef21_like_sgd_1}, \eqref{eq:EF21-SGDMv0} be set as $\gamma \leq \nicefrac{1}{L}$. Let $\hat x^T$ be sampled uniformly at random from the iterates of the method. Then for any $\eta> 0$ after $T$ iterations, we have
	$
	\Exp{\sqnorm{\nabla f(\hat x^T)}} \leq \fr{2 \delta_0 }{\gamma T} + 4 \eta^2 \sigma^2  .
	$
\end{proposition}	

Notice that if $\eta = 1$, then algorithm \algname{EF21-SGDM-ideal}~\eqref{eq:ef21_like_sgd_1}, \eqref{eq:EF21-SGDMv0} reduces to \algname{EF21-SGD-ideal} method \eqref{eq:ef21_like_sgd_1}, \eqref{eq:ef21_like_sgd_2}, and this result shows that the lower bound for the batch-size established in Theorem~\ref{thm:non_convergence_ef21_like_SGD} is tight, i.e., $B = \Theta (\nicefrac{\sigma^2}{\varepsilon^2})$ is necessary and sufficient\footnote{This follows by replacing $\sigma^2$ in the batch free algorithm by $\nicefrac{\sigma^2}{B}$ if the batch-size of size $B>1$ is used. } for convergence. For $\eta<1$, the above theorem suggests that using a small enough parameter $\eta$, the variance term can be completely eliminated. This observation motivates us to design a practical variant of this method. Similarly to the design of \algname{EF21} mechanism (from \algname{EF21-SGD-ideal}), we propose to do this by replacing the exact gradients $\nabla f(x^{t+1})$ by state vectors $v^{t}$ and $g^{t}$ as follows:
\begin{align}
\text{\algname{EF21-SGDM:}}&
	\begin{split}
		v^{t+1} &= \textcolor{mygreen}{ v^{t}}  + \eta (\nabla f(x^{t+1}, \xi^{t+1}) - \textcolor{mygreen}{ v^{t}}  ) , \\
		g^{t+1} &= \textcolor{blue}{ g^{t} } + \cC\rb{  v^{t+1}   - \textcolor{blue}{g^{t}} } 
	\end{split} 		
 \label{eq:EF21-SGDM}
\end{align}

 \begin{theorem}\label{thm:ef21-sgdm-one-node}
Let Assumptions~\ref{as:main}, \ref{as:BV} hold, and let $\cC$ satisfy Definition~\ref{def:contractive_compressor}. Let method \eqref{eq:ef21_like_sgd_1}, \eqref{eq:EF21-SGDM} be run with $g^0 = v^0 = \nabla f(x^0)$, and $\hat x^T$ be sampled uniformly at random from the iterates of the method. Then for all $\eta \in (0, 1]$ with $\gamma \leq \gamma_0 = \min\cb{ \fr{\al}{20 L}, \fr{\eta}{7 L}},$ we have $ \Exp{  \sqnorm{\nabla f(\hat x^T) }  } \leq \cO(\fr{\delta_0}{\gamma T}  + \eta \sigma^2  ).$ The choice $\eta = \min\cb{1,  \rb{ \fr{ L \delta_0   }{\sigma^2 T}}^{\nfr{1}{2} }}$, $\gamma = \gamma_0$ results in 
$
	\squeeze	\Exp{  \sqnorm{\nabla f(\hat x^T) }  } \leq 	\squeeze \cO \big( \fr{L \delta_0}{\alpha  T }  + \big( \fr{  L \delta_0 \sigma^2 }{ T } \big)^{\nfr{1}{2}}  \big).
$
\end{theorem}
Compared to Proposition~\ref{prop:ef21_m_SGDv0}, where $\eta$ can be made arbitrarily small, Theorem~\ref{thm:ef21-sgdm-one-node} suggests that there is a trade-off for the choice of $\eta \in (0, 1]$ in algorithm~\eqref{eq:ef21_like_sgd_1}, \eqref{eq:EF21-SGDM}. The above theorem implies that in single node setting  \algname{EF21-SGDM}  has $\cO(\frac{L}{\alpha \varepsilon^2} + \frac{L \sigma^2}{\varepsilon^4})$ sample complexity. 
For $\alpha = 1$, this result matches with the sample complexity of \algname{SGD} and is known to be unimprovable under Assumptions~\ref{as:main} and~\ref{as:BV} \citep{arjevani2019lower}. Moreover, when $\alpha = 1$, our sample complexity matches with previous analysis of momentum methods in \citep{Liu_Improved_analysis_SGDM_2020} and \citep{Defazio_Mom_PrimalAvg_2021}. However, even in this single node ($n=1$), uncompressed ($\alpha = 1$) setting our analysis is different from the previous work, in particular, our choice of momentum parameter and the Lyapunov function are different, see Appendix~\ref{sec:appendix_literature_momentum} and~\ref{sec:appendix_mom_simple}. For $\alpha < 1$, the above result matches with sample complexity of \algname{EF14-SGD} (single node setting) \citep{Stich_Delayed_2019}, which was recently shown to be optimal \citep{huang2022lower} for biased compressors satisfying Definition~\ref{def:contractive_compressor}. However, notice that the extension of the analysis by \citet{Stich_Delayed_2019} for \algname{EF14-SGD} to distributed setting meets additional challenges and it is unclear whether it is possible without imposing additional BG or BGS assumptions as in  \citep{Koloskova2019DecentralizedDL}. We revisit this analysis in Appendix~\ref{sec:revisiting_EF14} to showcase the difficulty of removing BG/BGS. In the following we will demonstrate the benefit of our  \algname{EF21-SGDM}  method by extending it to distributed setting without imposing any additional assumptions.

\subsection{Distributed stochastic error feedback with momentum} 
 Now we are ready to present a distributed variant of  \algname{EF21-SGDM} , see Algorithm~\ref{alg:EF21-M}. Letting $\delta_t  \eqdef f(x^t)  - f^*$, our convergence analysis of this method relies on the monotonicity of the following Lyapunov function:
 \begin{eqnarray}\label{eq:Lyapunov}
	\squeeze \Lambda_t \eqdef \delta_t +  \frac{\gamma}{\alpha n}\sum \limits_{i=1}^{n} \sqnorm{ g_i^t - v_i^t } + \frac{\gamma \eta}{\alpha^2 n }\sum\limits_{i=1}^{n}\sqnorm{ v_i^t - \nabla f_i(x^t)} + \frac{\gamma}{\eta} \sqnorm{\sum\limits_{i=1}^n (v_i^t - \nabla f_i(x^t))}.
\end{eqnarray}

 \begin{algorithm}[H]
 	\centering
 	\caption{ \algname{EF21-SGDM}  (Error Feedback  2021 Enhanced with Polyak Momentum)}\label{alg:EF21-M}
 	\begin{algorithmic}[1]
 		\State \textbf{Input:} starting point $x^{0}$, step-size $\gamma>0$, momentum $\eta \in (0, 1],$ initial batch size $B_{\textnormal{init}} $
		\State Initialize $v_i^{0} = g_i^{0} = \frac{1}{B_{\textnormal{init}}} \sum_{j=1}^{B_{\textnormal{init}}} \nabla f_{i}(x^0, \xi_{i, j}^{0})$ for  $ i = 1, \ldots,  n$; $g^{0} = \frac{1}{n} \sum_{i=1}^n  g_i^{0}$
 		\For{$t=0,1, 2, \dots , T-1 $}
		 \State Master computes $x^{t+1} = x^t - \gamma g^t$ and broadcasts $x^{t+1}$ to all nodes \label{line:EF21_HB_master_step}
 		\For{{\bf all nodes $i =1,\dots, n$ in parallel}}
 		\State \textcolor{mygreen}{ Compute momentum estimator $v_i^{t+1} = (1-\eta) v_i^{t}+ \eta  \nabla f_{i}(x^{t+1}, \xi_{i}^{t+1})  $ } \label{alg:eq:sgdm}
 		\State Compress $c_i^{t+1} =  \cC(  v_i^{t+1} - g_i^{t}  )$ and send $c_i^{t+1} $ to the master
 		\State Update local state $g_i^{t+1} = g_i^{t} +  c_i^{t+1}$
 		\EndFor
 		\State Master computes $g^{t+1} = \frac{1}{n} \sum_{i=1}^n  g_i^{t+1}$ via  $g^{t+1} = g^{t} + \frac{1}{n} \sum_{i=1}^n c_i^{t+1} $
 		\EndFor
 	\end{algorithmic}
 \end{algorithm}
  
{\bf $\bullet$ Convergence of  \algname{EF21-SGDM}  with contractive compressors.} We obtain the following result:

 \begin{theorem}\label{thm:main-distrib}
 	Let Assumptions~\ref{as:main} and \ref{as:BV} hold. Let $\hat x^T$ be sampled uniformly at random from the $T$ iterates of the method. Let  \algname{EF21-SGDM} (Algorithm~\ref{alg:EF21-M}) be run with a contractive compressor. For all $\eta \in (0, 1]$ and $B_{\textnormal{init}} \geq 1$, with $\gamma \leq \min\cb{ \fr{\al}{20 \wL}, \fr{\eta}{7 L}},$ we have
	\begin{align}
	\squeeze	\label{eq:main-distrib:general}
		\Exp{  \sqnorm{\nabla f(\hat x^T) }  } \leq \cO\rb{\fr{\Lambda_0}{\gamma T} + \fr{\eta^3 \sigma^2 }{\al^2}  +  \fr{\eta^2 \sigma^2 }{\al} + \frac{\eta \sigma^2}{n}},
	\end{align}
	where $\Lambda_0$ is given by \eqref{eq:Lyapunov}. 
	Choosing the batch size $B_{\textnormal{init}} =  \left\lceil\frac{\sigma^2}{L \delta_0}\right\rceil $, and stepsize 
$\gamma = \min\cb{ \fr{\al}{20 \wL}, \fr{\eta}{7 L}}$,  and momentum $ \eta = \min\cb{1, \rb{ \fr{ L \delta_0 \al^2 }{\sigma^2 T}}^{\nfr{1}{4} }, \rb{ \fr{ L \delta_0 \al }{\sigma^2 T}}^{\nfr{1}{3} }, \rb{ \fr{ L \delta_0 n  }{\sigma^2 T}}^{\nfr{1}{2} }, \frac{\alpha \sqrt{L \delta_0 B_{\textnormal{init}}}}{\sigma}},
$ \footnote{In Appendix~\ref{sec:appendix_mom_simple}, we show how to deal with time varying $\gamma_t$ and $\eta_t$.}
we get
 		\begin{eqnarray*}
 \squeeze		 \Exp{  \sqnorm{\nabla f(\hat x^T) }  } &\leq &\squeeze \cO\rb{ \fr{\wL \delta_0}{\alpha  T } +  \rb{\fr{   L \delta_0 \sigma^{2/3} }{\al^{2/3}  T } }^{\nfr{3}{4}} +  \rb{\fr{   L \delta_0 \sigma }{\sqrt{\al}  T } }^{\nfr{2}{3}}  + \rb{\fr{  L \delta_0 \sigma^2}{ n T } }^{\nfr{1}{2}} }  .
		\end{eqnarray*}
 \end{theorem}
 \begin{remark}
 	Note that using large initial batch size $B_{\text{init}} > 1$ is not necessary for convergence of \algname{EF21-SGDM}. If we set $B_{\text{init}} = 1$, the above theorem still holds by replacing $\delta_0$ with $\Lambda_0$. 
 \end{remark}
 \begin{remark}\label{rem:after_main_thm}
 	In the single node setting ($n=1$), the above result recovers the statement of Theorem~\ref{thm:ef21-sgdm-one-node} (with the same choice of parameters) since by Young's inequality $
 		\rb{\fr{   L \delta_0 \sigma^{2/3} }{\al^{2/3}  T } }^{\nfr{3}{4}} \leq \frac{1}{2}\frac{L \delta_0}{\alpha T} + \frac{1}{2}\left(\frac{L\delta_0 \sigma^2}{T}\right)^{\nicefrac{1}{2}} $, 
 		$\rb{\fr{   L \delta_0 \sigma }{\sqrt{\al}  T } }^{\nfr{2}{3}} \leq \frac{1}{3}\frac{L \delta_0}{\alpha T} + \frac{2}{3} \left(\frac{L\delta_0 \sigma^2}{T}\right)^{\nicefrac{1}{2}}$, and $ \wt L = L $.
 \end{remark}

{\bf $\bullet$ Recovering previous rates in case of full gradients.} Compared to the iteration complexity $\cO (  \fr{ L_{\textnormal{max}} G}{ \al \varepsilon^{3}}  )$ of \algname{EF14}  \citep{Koloskova2019DecentralizedDL}, our result, summarized in 
  \begin{corollary}\label{cor:EF21-M_conv-determ}
 	If $\sigma = 0$,  then $\Exp{\norm{\nabla f(\hat{x}^T) }} \leq \varepsilon$ after $T = \cO\rb{ \fr{\wL  }{\al \varepsilon^2}   }$ iterations. 
 \end{corollary}	
is better by an order of magnitude, and does not require the BG assumption. The result of Corollary~\ref{cor:EF21-M_conv-determ} is the same as for \algname{EF21} method  \citep{EF21}, and \algname{EF21-HB} method \citep{EF21BW_2021}. Notice, however, that even in this deterministic setting ($\sigma=0$)  \algname{EF21-SGDM}  method is different from \algname{EF21} and \algname{EF21-HB}: while the original \algname{EF21} does not use momentum, \algname{EF21-HB} method incorporates momentum on the server side to update $x^t$, which is different from our Algorithm~\ref{alg:EF21-M}, where momentum is applied by each node. This iteration complexity $\cO\rb{ \fr{1 }{\al \varepsilon^2}   }$ is optimal in both $\alpha$ and $\varepsilon$. The matching lower bound was recently established by~\citet{huang2022lower} for smooth nonconvex optimization in the class of centralized, zero-respecting algorithms with contractive compressors. 

{\bf $\bullet$ Comparison to previous work.} 
Our sample complexity\footnote{Note that the initial batch size contributes to the sample complexity only an additive constant independent of $\varepsilon$. Moreover, $B_{\textnormal{init}} =  \left\lceil\frac{\sigma^2}{L \delta_0}\right\rceil \leq \left\lceil\frac{2 \sigma^2}{\varepsilon^2}\right\rceil$ since, otherwise, $\norm{\nabla f(x^0)}^2 \leq 2 L \delta_0 \leq \varepsilon^2$, and $x^0$ is a solution. In the following, we ignore the dependece on $B_{\textnormal{init}}$ for a fair comparison with other works. }  in 
 \begin{corollary}\label{cor:EF21-M_conv}
 $\Exp{\norm{\nabla f(\hat{x}^T) }} \leq \varepsilon$ after $T = \cO\rb{ \fr{\wL  }{\al \varepsilon^2}   + \frac{L \sigma^{2/3}}{\alpha^{2/3} \varepsilon^{8/3}} + \fr{ L \sigma    }{  \alpha^{\nfr{1}{2}} \varepsilon^{3}}  +  \fr{ L \sigma^2    }{n \varepsilon^4} }$ iterations.
 \end{corollary}	
strictly improves over the complexity $\cO ( \fr{G L_{\textnormal{max}} }{ \al \varepsilon^{3}}  +  \fr{L_{\textnormal{max}} \sigma^2}{ n \varepsilon^{4}} ) $ of \algname{EF14-SGD} by \citet{Koloskova2019DecentralizedDL}, even in case when $G < +\infty$. Notice that it always holds that $\sigma \leq G$. 
 If we assume that $G \approx \sigma$, our three first terms in the complexity improve the first term from \citet{Koloskova2019DecentralizedDL} by the factor of $\nicefrac{\varepsilon}{\sigma},$ $(\nicefrac{\varepsilon \alpha}{\sigma})^{1/3},$ or $\alpha^{1/2}.$ 
 Compared to the \algname{BEER} algorithm of \citet{Zhao_BEER_2022}, with sample complexity $\cO ( \frac{ L_{\textnormal{max}}}{\alpha \varepsilon} + \frac{L_{\textnormal{max}} \sigma^2 }{\alpha^2 \varepsilon^4} )$, the result of Corollary~\ref{cor:EF21-M_conv} is strictly better in terms of $\alpha$, $n$, and the smoothness constants.\footnote{$L_{\textnormal{max}}  \eqdef \max_{i\in[n]} L_i$. Notice that  $L \leq \wt L \leq L_{\textnormal{max}}$ and the inequalities are strict in heterogeneous setting. } In addition, we remove the large batch requirement for convergence compared to \citep{EF21BW_2021,Zhao_BEER_2022}. Moreover, notice that Corollary~\ref{cor:EF21-M_conv} implies that  \algname{EF21-SGDM}  achieves asymptotically optimal sample complexity $\cO (\frac{L \sigma^2 }{n \varepsilon^4} )$ in the regime $\varepsilon \rightarrow 0$. 
 
\subsection{Further improvement using {\em double} momentum!}
 Unfortunately, in the non-asymptotic regime, our sample complexity does not match with the lower bound in all problem parameters simultanuously due to the middle term $\frac{L \sigma^{2/3}}{\alpha^{2/3} \varepsilon^{8/3}} + \fr{ L \sigma    }{  \alpha^{\nfr{1}{2}} \varepsilon^{3}} $, which can potentially dominate over $\fr{ L \sigma^2    }{  n \varepsilon^{4}} $ term for large enough  $n$ and $\varepsilon$, and small enough $\alpha$ and $\sigma$. 
 We propose a {\em double-momentum} method, which can further improve the middle term in the sample complexity of  \algname{EF21-SGDM}. We replace the momentum estimator $v_i^t$ in line \ref{alg:eq:sgdm} of Algorithm~\ref{alg:EF21-M} by the following two-step momentum update
 \begin{align}\label{eq:double_mom}
 	\text{\algname{EF21-SGD2M:}}\quad
 v_i^{t+1} = (1-\eta) v_i^{t}+ \eta  \nabla f_{i}(x^{t+1}, \xi_{i}^{t+1}) , \quad
  u_i^{t+1} = (1-\eta) u_i^{t}+ \eta  v_i^{t+1} . 
  \end{align}

We formally present this method in  Algorithm~\ref{alg:EF21-DM} in Appendix~\ref{sec:appendix_DM}. Compared to \algname{EF21-SGDM} (Algorithm~\ref{alg:EF21-M}), the only change is that instead of compressing $v_i^t - g_i^t$, in \algname{EF21-SGD2M}, we compress $u_i^t - g_i^t$, where $u_i$ is a two step (double) momentum estimator. The intuition behind this modification is that a double momentum estimator $u_i^t$ has richer "memory" of the past gradients compared to $v_i^t$. Notice that for each node, \algname{EF21-SGD2M} requires to save $3$ vectors ($v_i^t$, $u_i^t$, $g_i^t$) instead of $2$ in \algname{EF21-SGDM} ($v_i^t$, $g_i^t$) and \algname{EF14-SGD} ($e_i^t$, $g_i^t$).\footnote{See Appendix~\ref{sec:revisiting_EF14} for details on \algname{EF14-SGD}. In contrast, \algname{EF21-SGD} needs to save only one vector ($g_i^t$). } When interacting with biased compression operator $\cC(\cdot)$, such effect becomes crucial in improving the sample complexity. For \algname{EF21-SGD2M}, we derive

 \begin{corollary}\label{cor:EF21-DM_conv}
Let $v_i^t$ in Algorithm~\ref{alg:EF21-M} be replaced by $u_i^t$ given by \eqref{eq:double_mom} (Algorithm~\ref{alg:EF21-DM} in Appendix~\ref{sec:appendix_DM}). 
Then with appropriate choice of $\gamma$ and $\eta$ (given in Theorem~\ref{thm:EF21-DM-SGD}), 
we have $\Exp{\norm{\nabla f(\hat{x}^T) }} \leq \varepsilon$ after $T = \cO\rb{ \fr{\wL \delta_0 }{\al \varepsilon^2}    + \fr{L \delta_0 \sigma^{2/3}     }{  \alpha^{\nfr{2}{3}} \varepsilon^{8/3}}  +  \fr{ L \delta_0 \sigma^2     }{n \varepsilon^4} }$ iterations.
\end{corollary}	
 


\section{Experiments}
\label{sec:experiments_main}
We consider a nonconvex logistic regression problem: $f_i(x_1, \dots, x_{c}) = -\frac{1}{m} \sum_{j=1}^m\log(\exp(a_{ij}^\top x_{y_{ij}}) / \sum_{y=1}^c \exp(a_{ij}^\top x_{y}))$ with a nonconvex regularizer $h(x_1, \dots, x_c) = \lambda \sum_{y = 1}^c \sum_{k = 1}^l [x_{y}]_k^2 / (1 + [x_{y}]_k^2)$
with $\lambda = 10^{-3},$ where $x_1, \dots, x_{c} \in \R^{l},$ $[\cdot]_k$ is an indexing operation of a vector, $c \geq 2$ is the number of classes, $l$ is the number of features, $m$ is the size of a dataset, $a_{ij} \in \R^{l}$ and $y_{ij} \in \{1, \dots, c\}$ are features and labels. The datasets used are \textit{MNIST} (with $l =  784$, $m = 60 \, 000$, $c = 10$) and \textit{real-sim} (with $l = 20 \, 958$, $m = 72\,309$, $c = 2$) \citep{lecun2010mnist, chang2011libsvm}. The dimension of the problem is $d = (l + 1) c$, i.e., $d = 7\,850$ for \textit{MNIST} and $d = 41\,918$ for \textit{real-sim}. In each experiment, we show relations between the total number of transmitted coordinates and gradient/function values. The stochastic gradients in each algorithm are replaced by a mini-batch estimator $\frac{1}{B} \sum_{j = 1}^{B} \nabla f_i(x, \xi_{i j})$ with the same $B\geq 1$ in each plot. Notice that all methods (except for \algname{NEOLITHIC})\footnote{For \algname{NEOLITHIC}, we use the parameter $R = \lceil d/K \rceil$ following the requirement in their Theorem 3. Experiments in \cite{huang2022lower} use a heuristic choice $R=4$, and thus can show faster convergence. } calculate the same number of samples at each communication round, thus the dependence on the number of samples used will be qualitatively the same.  In all algorithms, the step sizes are fine-tuned from a set $\{2^k\,|\, k \in [-20, 20]\}$ and the Top$K$ compressor is used to compress information from the nodes to the master. For  \algname{EF21-SGDM} , we fix momentum parameter $\eta = 0.1$ in all experiments. Prior to that, we tuned $\eta \in \{0.01, 0.1\}$ on the independent dataset \textit{w8a} (with $l = 300$, $m = 49 \, 749$, $c = 2$). We omit \algname{BEER} method from the plots since it showed worse performance than \algname{EF21-SGD} in all runs. 

\subsection{Experiment 1: increasing batch-size}
In this experiment, we use \textit{MNIST} dataset and fix the number of transmitted coordinates to $K = 10$ (thus $\alpha \geq \nicefrac{K}{ d} \approx 10^{-3} $), and set $n = 10$. Figure~\ref{fig:log_reg_batch_increase} shows convergence plots for $B\in \cb{1 ,32,128}$.  \algname{EF21-SGDM} and its double momentum version \algname{EF21-SGD2M} have fast convergence and show a significant improvement when increasing batch-size compared to \algname{EF14-SGD}. In contrast, \algname{EF21-SGD} suffers from poor performance for small $B$, which confirms our observations in previous sections. \algname{NEOLITHIC} has order times slower convergence rate due to the fact that it sends $\left\lceil \nicefrac{d}{K} \right\rceil$ compressed vectors in each iteration, while other methods send only one.
\begin{figure}[h]
\centering
\begin{subfigure}{.33\textwidth}
	\centering
	\includegraphics[width=0.95\textwidth]{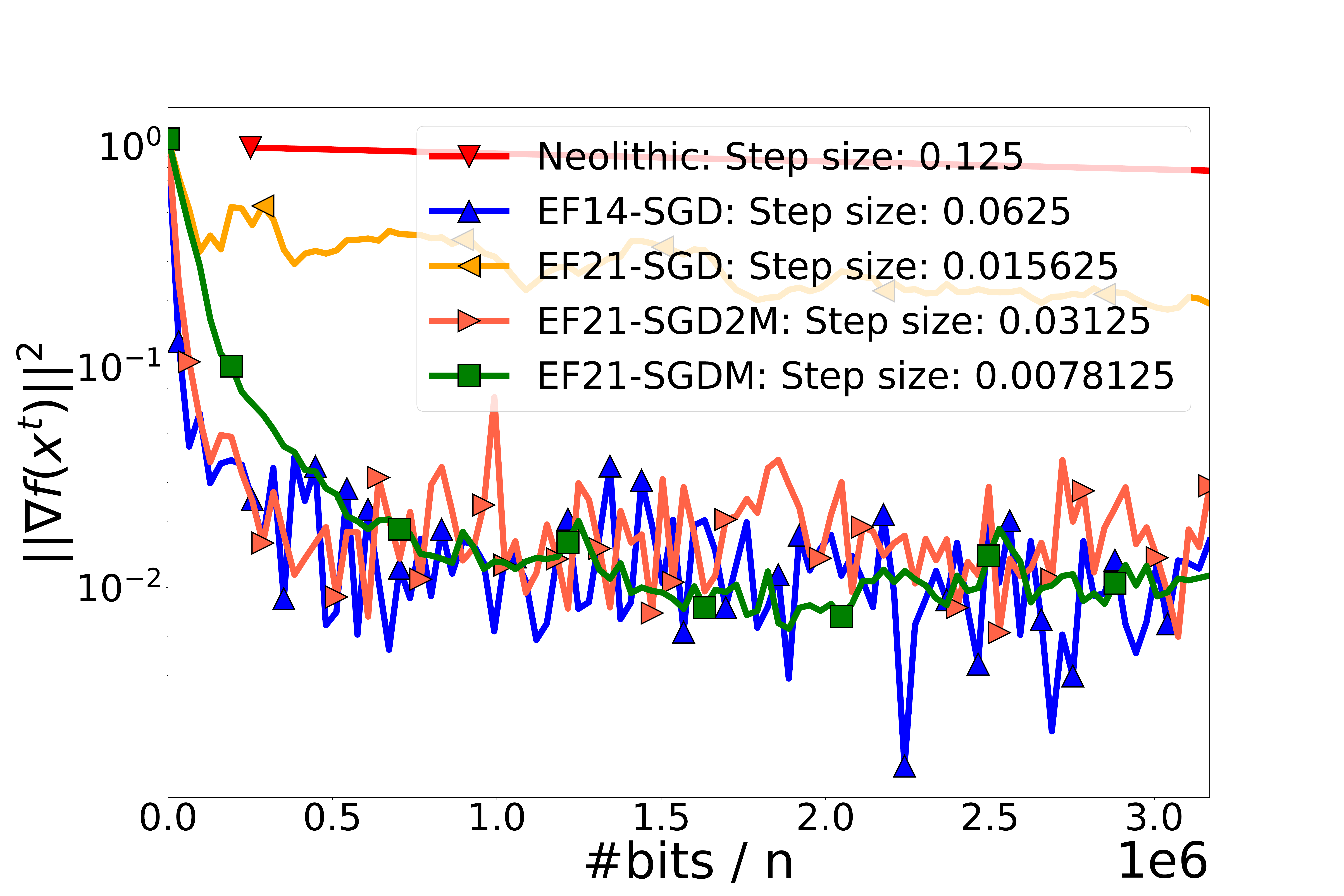}
	\caption{ $B = 1$}
\end{subfigure}\hfill
\begin{subfigure}{.33\textwidth}
	\centering
	\includegraphics[width=0.95\textwidth]{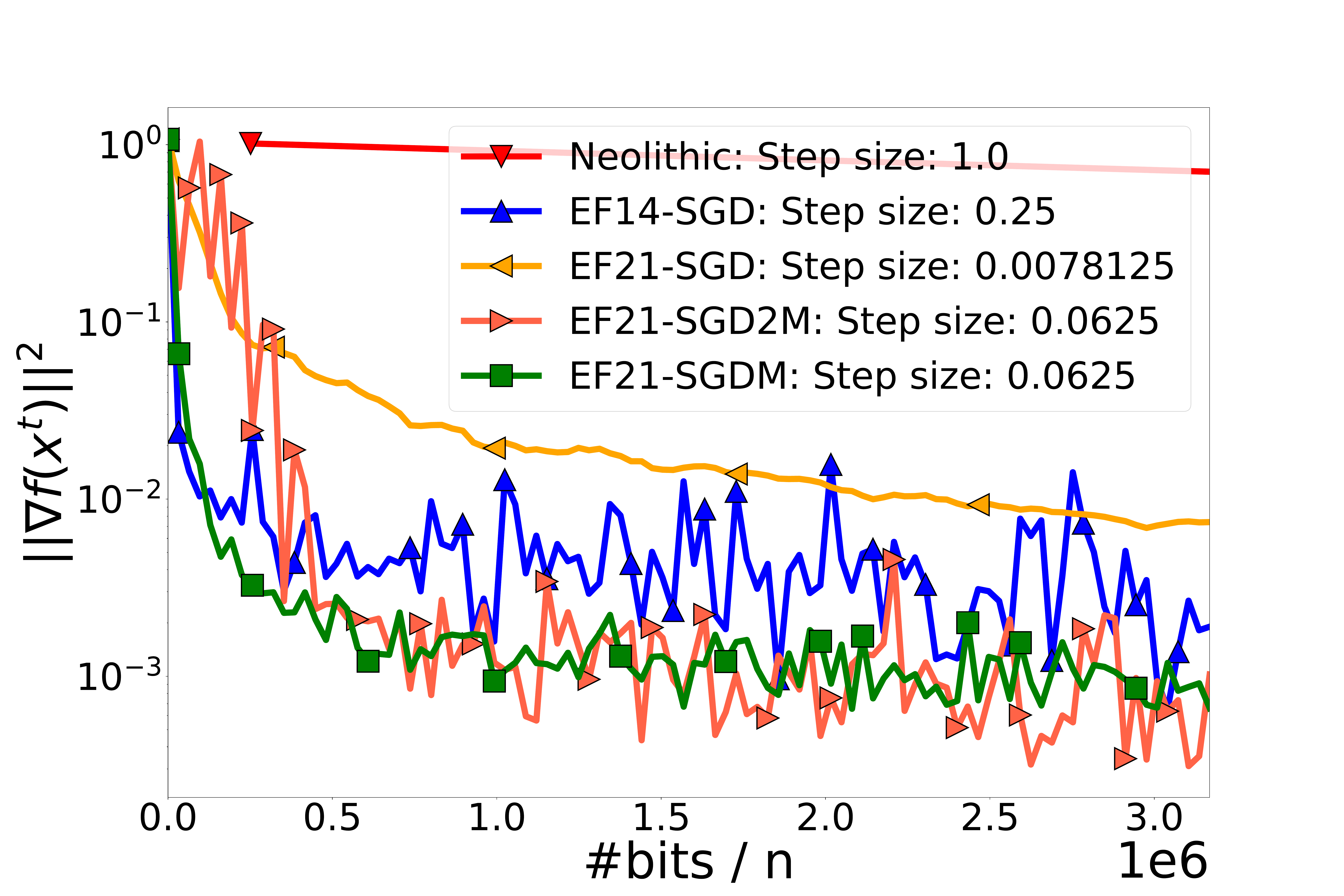}
	\caption{ $B = 32$}
\end{subfigure}\hfill
\begin{subfigure}{.33\textwidth}
	\centering
	\includegraphics[width=0.95\textwidth]{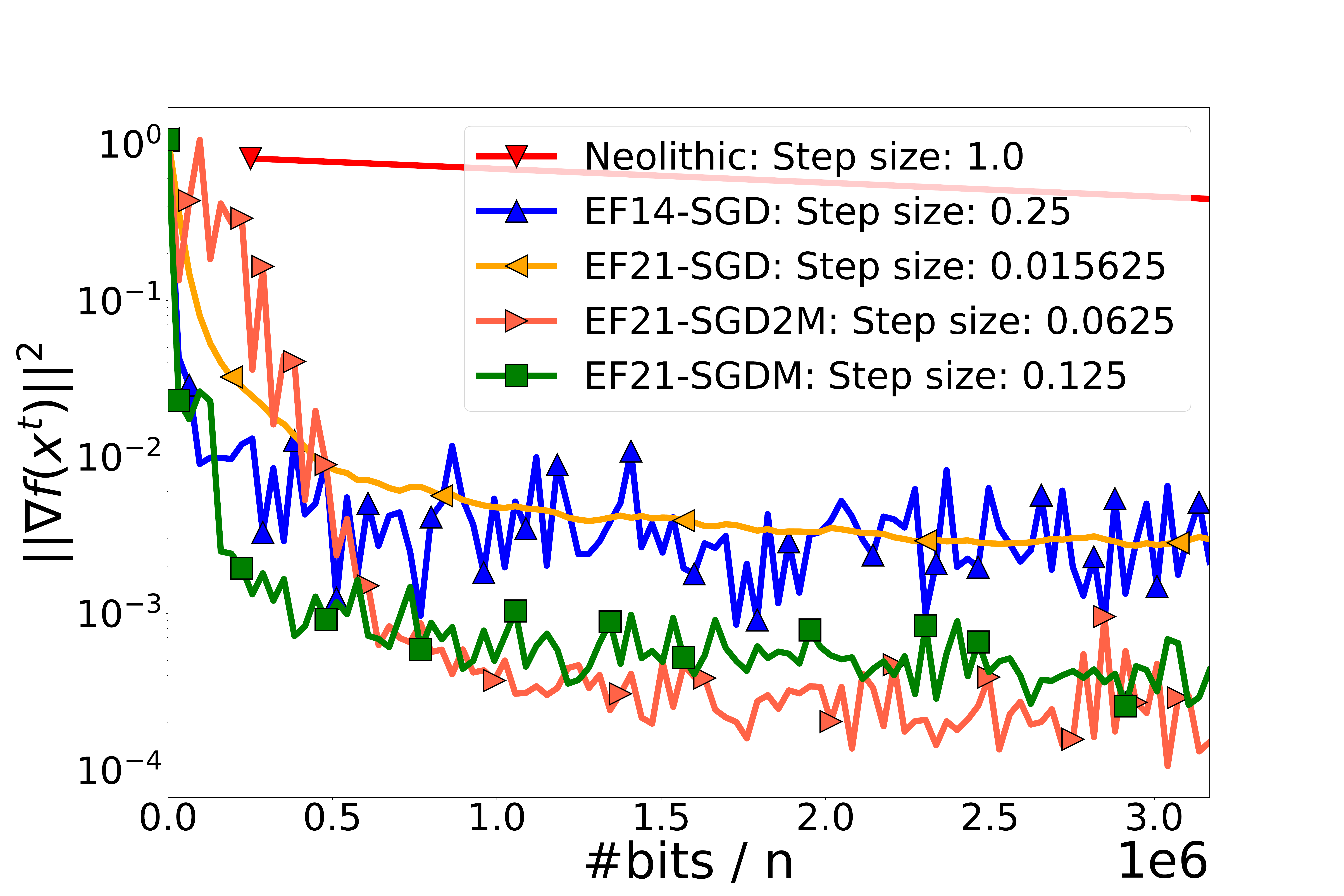}
	\caption{ $B = 128$}
\end{subfigure}\hfill
	\caption{Experiment on \textit{MNIST} dataset with $n = 10$, and Top$10$ compressor.}
\label{fig:log_reg_batch_increase}
\end{figure}

\subsection{Experiment 2: improving convergence with $n$}
This experiment uses \textit{real-sim} dataset, $K = 100$ (thus $\alpha \geq \nicefrac{K}{d} \approx 2 \cdot 10^{-3} $), and with $B = 128 \ll m $. We vary the number of nodes within $n \in \cb{1, 10, 100}$, see Figure~\ref{fig:log_reg_real_sim_mini_batch}. In this case, \algname{EF21-SGDM} and \algname{EF21-SGD2M} have much faster convergence compared to other methods for all $n$. Moreover, the proposed algorithms show a significant improvement when $n$ increases. We also observe that on this task, \algname{EF21-SGD2M} performs slightly worse than \algname{EF21-SGDM} for $n = 10, 100$ , but it is still much faster than other other methods. 

\begin{figure}[h]
	\centering
	\begin{subfigure}{.33\textwidth}
		\centering
		\includegraphics[width=0.95\textwidth]{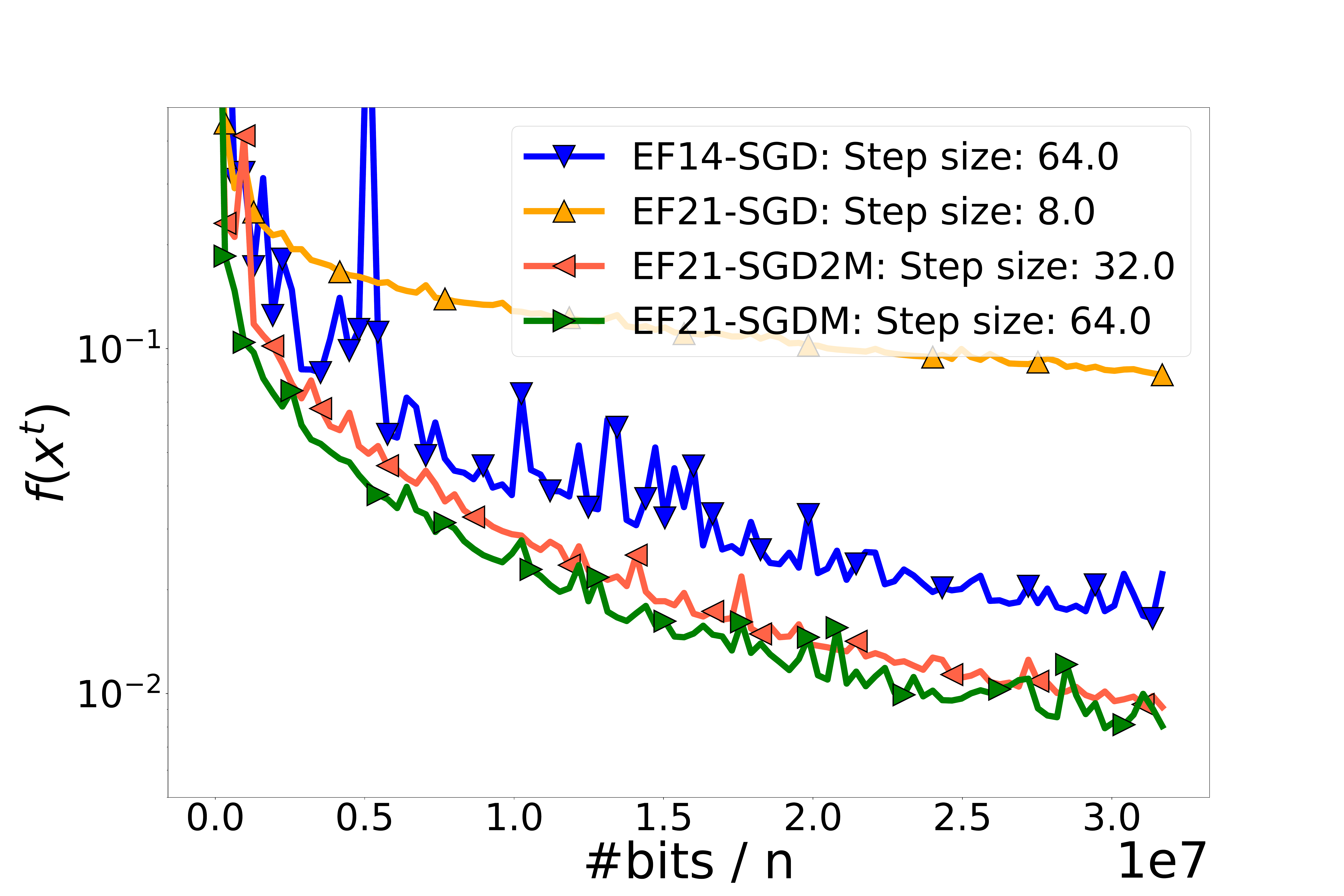}
		\caption{$n = 1$}
	\end{subfigure}
	\begin{subfigure}{.33\textwidth}
		\centering
		\includegraphics[width=0.95\textwidth]{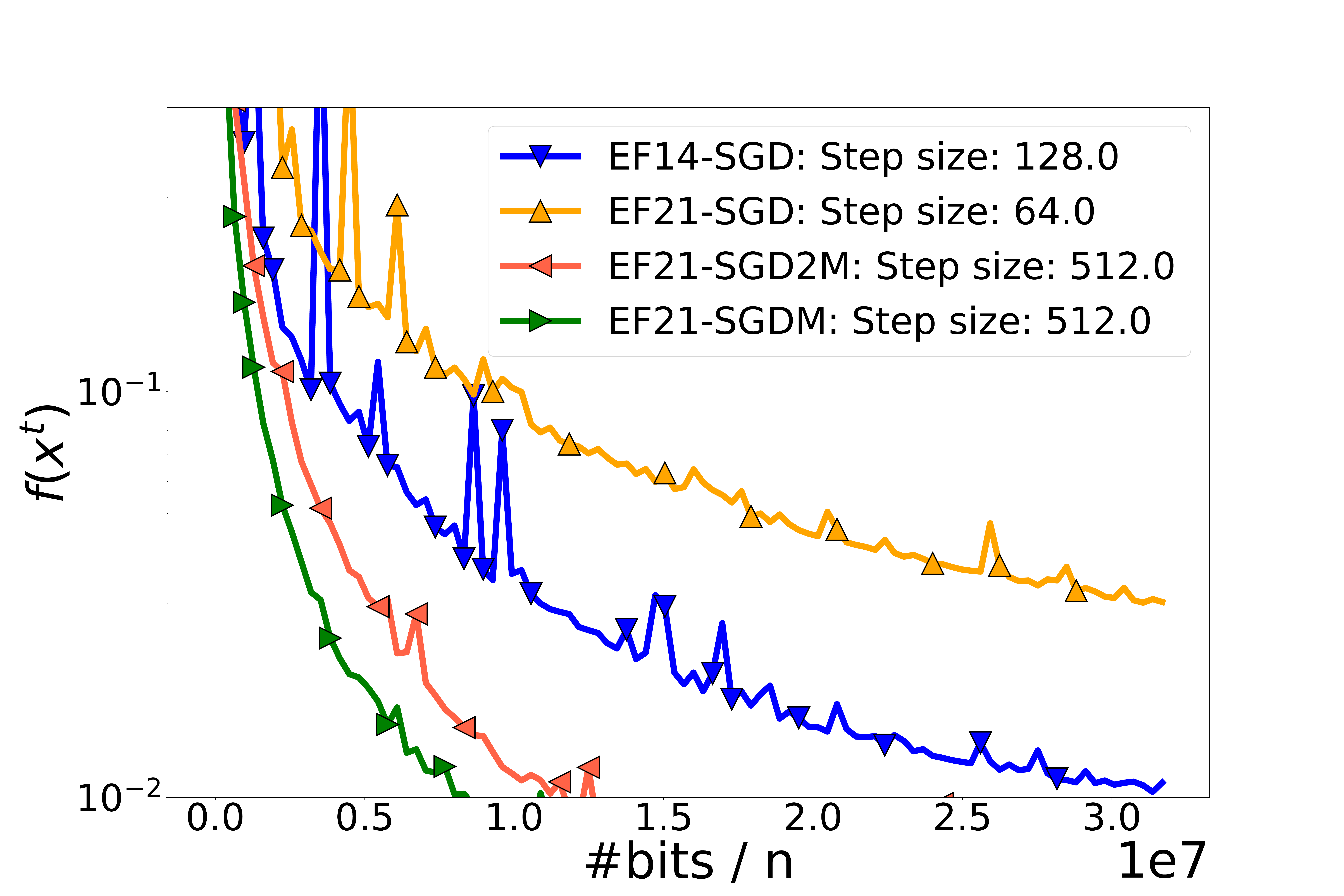}
	\caption{ $n = 10$}
	\end{subfigure}\hfill
	\begin{subfigure}{.33\textwidth}
		\centering
		\includegraphics[width=0.95\textwidth]{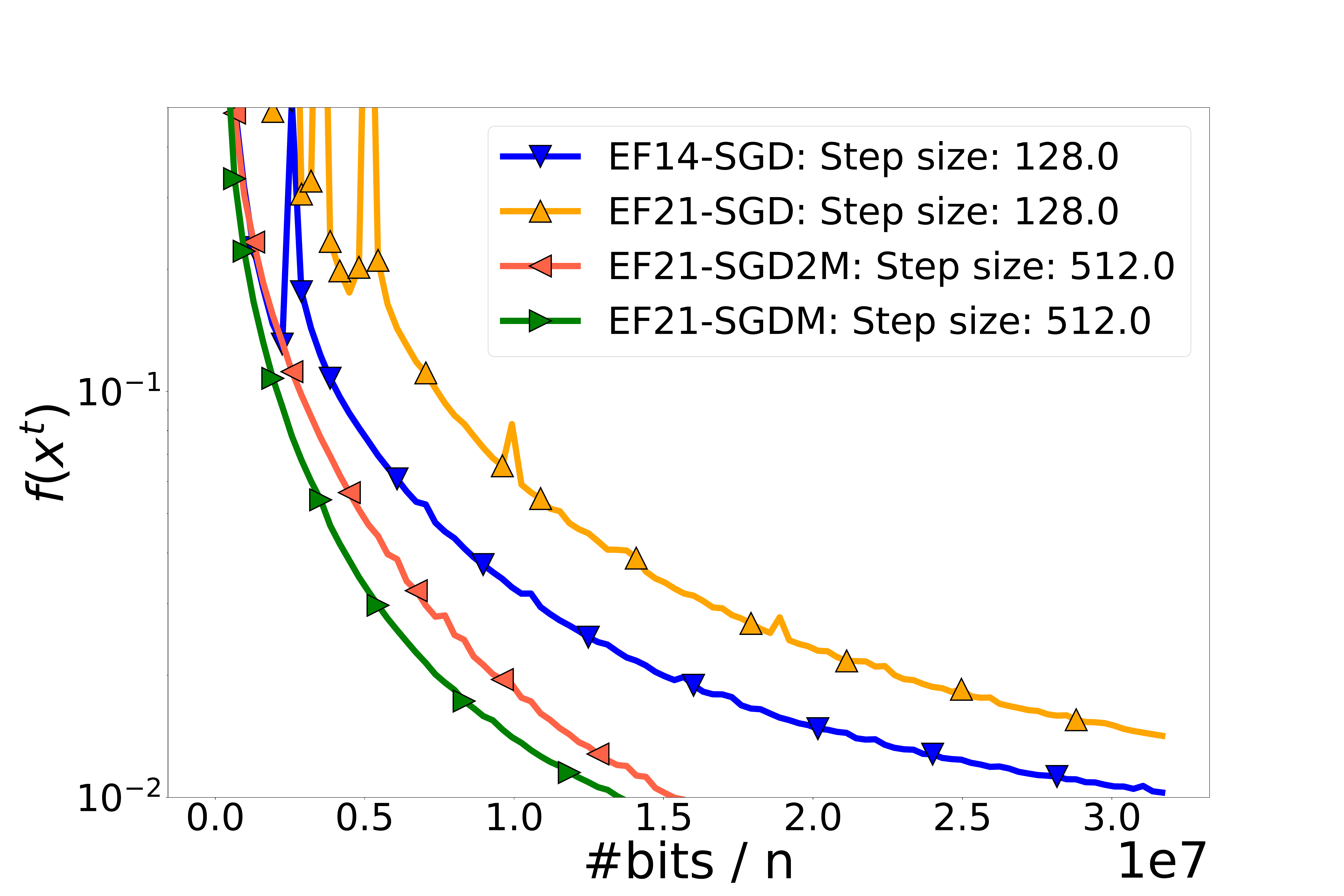}
		\caption{  $n = 100$}
	\end{subfigure}\hfill
\caption{Experiment on \textit{real-sim} dataset with batch-size $B = 128$, and Top$100$ compressor.}
\label{fig:log_reg_real_sim_mini_batch}
	\end{figure}

In Section~\ref{sec:add_experiments}, we present extra simulations with different parameters for above experiments. Additionally, we inlclude experiemnts on simple quadratic problems and perform training of larger image recognition models. In all cases, \algname{EF21-SGDM} and \algname{EF21-SGD2M} outperform other algorithms. 


\section*{Acknowledgments and Disclosure of Funding}
 This work of I. Fatkhullin was supported by ETH AI Center doctoral fellowship. The work of P. Richt\'{a}rik and A. Tyurin was supported by the KAUST Baseline Research Scheme (KAUST BRF) and the KAUST Extreme Computing Research Center (KAUST ECRC); P. Richt\'{a}rik was also supported by the SDAIA-KAUST Center of Excellence in Data Science and Artificial Intelligence (SDAIA-KAUST AI).

\bibliographystyle{plainnat}
\bibliography{bibliography.bib}

\clearpage
\appendix


\clearpage

\section{More on Contractive Compressors, Error Feedback and Momentum}\label{sec:appendix_literature_momentum}

\begin{table}[t]
	\caption{\footnotesize Extended summary of related works on distributed error compensated SGD methods using a Top$K$ compressor under Assumptions~\ref{as:main} and~\ref{as:BV}. The goal is to find an $\varepsilon$-stationary point of a smooth nonconvex function of the form \eqref{eq:problem}, i.e., a point $x$ such that $\Exp{\norm{\nabla f(x)}} \leq \varepsilon$. "\textbf{Comm. compl.}" reports the total number of communicated bits if the method is applied with batch-size equal to "Batch-size" at each node. When Top$K$ compressor is applied, then $\alpha \geq K / d$, and the comm. compl. of error compensated methods can be reduced by a factor of $\alpha d / K$. "\textbf{Batch-size for comm. compl.}" means the batch-size for achieving the reported "Comm. compl.". "\textbf{Asymp. sample compl.}" reports the asymptotic sample complexity of the algorithm with batch-size $B = 1$ in the regime $\varepsilon\rightarrow0$, i.e., the total number of samples required at each node to achieve $\varepsilon$-stationary point.
		"\textbf{Batch free}" marks with \gcmark\, if the analysis ensures convergence with batch-size equal to $1$. 
		"\textbf{No extra assump.}" marks with \gcmark \, if  no additional assumption beyond Assumptions~\ref{as:main} and \ref{as:BV} is required for analysis. We denote $L_{\textnormal{max}}  \eqdef \max_{i\in[n]} L_i$. Notice that it always holds $L \leq \wt L \leq L_{\textnormal{max}}$ and these inequalities only become equalities in the homogeneous case. It could be that $\alpha L / \wt L \ll 1$ making the batch-size of \algnamesmall{EF21-SGDM} and \algnamesmall{EF21-SGD2M} much smaller than those of \algnamesmall{EF21-SGD} and \algnamesmall{BEER}. Symbol $\vee$ denotes the maximum of two scalars. }
	\label{table:related-works}
	\centering
	\footnotesize
	\begin{threeparttable}
		\begin{tabular}{|c|c|c|c|c|c|}
			\hline
			\bf Method & \bf \makecell{Comm. \\compl.}  &\bf \makecell{ Batch-size  for \\
			comm. compl.}  & \bf \makecell{Asymp. \\sample \\ compl.}  & \bf \makecell{Batch\\ free?}   &\bf \makecell{No extra \\ assump.? }   \\
			\hline
			\makecell{\algnamesmall{EF14-SGD} \\ \citep{Koloskova2019DecentralizedDL}}  &  $ \fr{K {\color{myred}{G}} L_{\textnormal{max}} }{\alpha \color{myred}{\varepsilon^3}}$ & \makecell{$  \fr{\alpha \sigma^2}{n \varepsilon G}$\textsuperscript{{\color{blue}(*) }} }   & \makecell{$  \fr{ L_{\textnormal{max}} \sigma^2}{n \varepsilon^4}$} & \gcmark  &  \makecell{\rxmark}\textsuperscript{{\color{blue}(a)}}   \\
			\hline	
			\makecell{\algnamesmall{NEOLITHIC}\\ \citep{huang2022lower}}  &$ \fr{K  L_{\textnormal{max}}}{ \alpha \varepsilon^{2} }\color{myred}{\log\left( \frac{ G }{ \varepsilon }  \right)} $ \textsuperscript{{\color{blue}(b)}}  & \makecell{ $ \frac{ \sigma^2}{n \varepsilon^2  } \vee \frac{1}{\alpha} \log\left( \frac{ G }{ \varepsilon }  \right)   $\textsuperscript{{\color{blue}(*) }}} & \makecell{$ \fr{ L_{\textnormal{max}} \sigma^2}{n \varepsilon^4}$} & \rxmark   &  \makecell{\rxmark}\textsuperscript{{\color{blue}(c)}}  \\
			\hline
			\makecell{\algnamesmall{EF21-SGD}\\ \citep{EF21BW_2021}}  & $ \fr{K \wt L}{ \alpha \varepsilon^{2} } $ & \makecell{$  \fr{ \sigma^2}{{\color{myred}\alpha^2} \varepsilon^2}$} & $ \fr{\wt L \sigma^2}{{\color{myred}\alpha^3} \varepsilon^4}$\textsuperscript{{\color{blue}(d) }} 
			&  \makecell{\rxmark} &  \makecell{\gcmark}  \\
			\hline
			\makecell{\algnamesmall{BEER}\\ \citep{Zhao_BEER_2022}}  & $ \fr{K L_{\textnormal{max}}}{ \alpha \varepsilon^{2} }$ & $\fr{\sigma^2}{ {\color{myred}\alpha} \varepsilon^2}$ &  $  \fr{ L_{\textnormal{max}} \sigma^2}{ {\color{myred} \alpha^2} \varepsilon^4}$\textsuperscript{{\color{blue}(d) }}
			& \makecell{\rxmark}  &  \makecell{\gcmark}  \\
			\hline
			\makecell{\algnamesmall{EF21-SGDM} \\ Corollary~\ref{cor:EF21-M_conv}}  & \cellcolor{linen} $ \fr{ K\wL  }{\al \varepsilon^2} $ & \cellcolor{linen} $  \frac{\alpha L}{\wt L }\fr{ \sigma^2}{ n \varepsilon^2} \vee \frac{\alpha L^2}{\wt L^2}\fr{ \sigma^2}{ \varepsilon^2}   $ & \cellcolor{linen} $   \fr{ L \sigma^2    }{n \varepsilon^4}$ &\cellcolor{linen} \gcmark   &  \cellcolor{linen} \gcmark \\
			\hline
			\makecell{\algnamesmall{EF21-SGD2M}\\  Corollary~\ref{cor:EF21-DM_conv} }   & \cellcolor{linen} $ \fr{ K\wL  }{\al \varepsilon^2} $ & $ \cellcolor{linen} \frac{\alpha L}{\wt L }\fr{ \sigma^2}{ n \varepsilon^2} \vee \frac{\alpha L^3}{\wt L^3}\fr{ \sigma^2}{ \varepsilon^2}   $ & \cellcolor{linen}$   \fr{ L \sigma^2    }{n \varepsilon^4}$ &\cellcolor{linen} \gcmark  &  \cellcolor{linen} \gcmark\\
			\hline
		\end{tabular}
		\begin{tablenotes}
			\scriptsize
			\item [{\color{blue}(a)}] Analysis requires a bound of the second moment of the stochastic gradients, i.e., $\Exp{\sqnorm{\nabla f_i(x, \xi_i)}} \leq G^2$ for all $x \in \R^d$. 
			\item [{\color{blue}(b)}] This complexity is achieved by using a large mini-batch and communicating $\left\lceil \nicefrac{K}{\alpha} \right\rceil \approx d$ coordinates per iteration.
			\item [{\color{blue}(c)}] Analysis requires a bounded gradient disimilarity assumption, i.e.,  $
			\suminn \sqnorm{\nabla f_i(x) - \nabla f(x)} \leq G^2  
			$ for all $x \in \R^d$. 
			\item [{\color{blue}(d)}] Analysis requires a batch-size at least $B \geq \fr{\sigma^2}{ \alpha^2 \varepsilon^2}$ for \algnamesmall{EF21-SGD} and $B \geq \fr{\sigma^2}{ \alpha \varepsilon^2}$ for \algnamesmall{BEER}.
			\item [{\color{blue}(*)}] For a fair comparison, we take the (minimal) batch-size for these methods which guarantees the reported communication complexity. 
		\end{tablenotes}
	\end{threeparttable}
\end{table}

\paragraph{Greedy vs uniform.}
In our work, we specifically focus on the class of contractive compressors satisfying Definition~\ref{def:contractive_compressor}, which contains a greedy Top$K$ compressor as a special case. Note that Top$K$ is greedy in that it minimizes the error $\|{\rm Top}K(x) - x\|^{2}$ subject to the sparsity constraint $\|\cC(x)\|_0 \leq K$, where $\|u\|_0$ counts the number of nonzero entries in $u$. In practice, greediness is almost always\footnote{Greediness is not useful, for example, when $\cD_i=\cD_j$ for all $i,j$ and when Top$K$ is applied to the full-batch gradient $\nabla f_i(x)$ by each client. However, situations of this type arise rarely in practice.} very useful, translating into excellent empirical performance, especially when compared to the performance of the Rand$K$ sparsifier. On the other hand, it appears to be very hard to formalize these practical gains theoretically\footnote{No theoretical results of this type exist for $n>1$ .}. In fact, while greedy compressors such as Top$K$ outperform their randomized cousins such as Rand$K$ in practice, and often by huge margins \citep{lin2018deep}, the theoretical picture is exactly reversed, and the theoretical communication complexity of gradient-type methods based on randomized compressors~\citep{alistarh2017qsgd,DIANA,DIANA2,ADIANA,MARINA} is much better than of those based on greedy compressors~\citep{Koloskova2019DecentralizedDL,EF21,EF21BW_2021,3PC}. The key reason behind this is the fact that popular randomized compressors such as Rand$K$ become {\em unbiased} mappings after appropriate scaling (e.g., $\Expthin{\frac{d}{K}{\rm Rand}K(x)}\equiv x$), and that the inherent randomness is typically drawn {\em independently} by all clients. This leads to several key simplifications in the analysis, and consequently, to theoretical gains over methods that do not compress, and over methods that compress greedily. Further improvements are possible when the randomness is {\em correlated} in an appropriate way \citep{szlendak2021permutation}.

Due to the superior empirical properties of greedy contractive compressors, and our desire to push this very potent line of work further, in this paper we work with the general class of compressors satisfying Definition~\ref{def:contractive_compressor}, and do not invoke any additional restrictive assumptions. For example, we do not assume $\cC$ can be made unbiased after scaling.

\paragraph{Error Feedback.}
The first theoretical analysis of \algname{EF14} was presented in the works of \citet{Stich-EF-NIPS2018,Alistarh-EF-NIPS2018} and further revisited in convex case in \citep{Karimireddy_SignSGD,beznosikov2020biased,Lin_EC_SGD,EC-Katyusha} and analysis was extended to nonconvex setting in \citep{Stich_Delayed_2019}. Later, in nonconvex case, various extensions and combinations of \algname{EF14} with other optimization techniques were considered and analyzed, which include bidirectional compression \citep{DoubleSqueeze}, decentralized training \citep{Koloskova2019DecentralizedDL,Singh_SquarmSGD_2021}, server level momentum \citep{CSER}, client level momentum \citep{Zheng_EF_SGD_Nesterov_2019}, combination with adaptive methods \citep{Li_COMP_AMS_2022}. To our knowledge, the best sample complexity for finding a stationary point for this method (including its momentum and adaptive variants) in the distributed nonconvex setting is given by~\citet{Koloskova2019DecentralizedDL}, which is  $\cO (  \fr{G}{ \al \varepsilon^{3}}  +  \fr{\sigma^2}{ n \varepsilon^{4}} )$. More recently, \citet{huang2022lower} propose a modification of \algname{EF14} method achieving $\cO \rb{  \fr{1}{ \al \varepsilon^{2}} \log(\frac{G}{\varepsilon}  ) +  \fr{\sigma^2}{ n \varepsilon^{4}} }$ sample complexity by using the BGS assumption. When applied with Top$K$ compressor, this method requires to communicate $\wt \cO\rb{ \nicefrac{K}{\alpha}  } $ coordinates at every iteration. This makes it impractical since when the effective $\alpha$ is unknown and is set to $\alpha = K / d$, it means that their method communicates all $d$ coordinates at every iteration, mimicking vanilla \algname{(S)GD} method. Moreover, their algorithm uses an additional subroutine and applies updates with a large batch-size of samples of order $\cO(\frac{1}{\alpha} \log\left(\frac{G}{\varepsilon} \right) )$, making the algorithm less practical and difficult to implement. It is worth to mention, that error feedback was also analyzed for other classes of compressors such as absolute (see Definition~\ref{def:absolute_compressor}) or additive compressors (i.e., $\cC(x+y ) = \cC(x) + \cC(y)$ for all $x, y\in\R^d$) \citep{DoubleSqueeze,xu2022detached}, which do not include Top$K$ sparsifier. 

\paragraph{Momentum.}
The first convergence analysis of gradient descent with momentum was proposed by B.T. Polyak in his seminal work~\citep{Polyak_Some_methods_1964} studying the benefit of multi-step methods. The proof technique proposed in this work is based on the analysis of the spectral norm of a certain matrix arising from the dynamics of a multi-step process on a quadratic function. Unfortunately, such technique is restricted to the case of strongly convex quadratic objective and the setting of full gradients. Later~\citet{Zavriev_heavy_ball_1993} prove an asymptotic convergence of this method in nonconvex deterministic case without specifying the rate of convergence. 

To our knowledge, the first non-asymptotic analysis of \algname{SGDM} in the smooth nonconvex setting is due to \citet{Yu_LinearSpeedup_Com_Efficient_SGDM_2019}. Their analysis, however, heavily relies on BG assumption. Later, \citet{Liu_Improved_analysis_SGDM_2020} provide a refined analysis of \algname{SGDM}, removing the BG assumption and improving the dependence on the momentum parameter $\eta$. Notice that the analysis of \citet{Liu_Improved_analysis_SGDM_2020} and the majority of other works relies on some variant of the following Lyapunov function: 
\begin{equation}\label{eq:HB_lyapunov_dynamic}
	\Lambda_t  \eqdef f(z^t) - f^* +  \sum_{\tau=0}^t c_{\tau} \sqnorm{x^{t-\tau} - x^{t-1-\tau} } ,
\end{equation}
 where $\cb{z^t}_{t\geq0}$ is some auxiliary sequence (often) different from the iterates $\cb{x^t}_{t\geq0}$, and $\cb{c_{\tau}}_{\tau\geq0}$ is a diminishing non-negative sequence. This approach is motivated by the dynamical system point of view at Polyak's heavy ball momentum, where the two terms in \eqref{eq:HB_lyapunov_dynamic} are interpreted as the potential and kinetic energy of the system~\citep{sebbouh2019nesterov}. In contrast, the Lyapunov function used in this work is conceptually different even in the single node ($n=1$), uncompressed ($\alpha=1$) setting. Later, \citet{Defazio_Mom_PrimalAvg_2021} revisit the analysis in \citep{Liu_Improved_analysis_SGDM_2020} through the lens of primal averaging and provide insights on why momentum helps in practice.  
The momentum is also used for stabilizing adaptive algorithms such as normalized SGD \citep{Cutkosky_MomNSGD_2020}. In particular, it was shown that by using momentum, one can ensure convergence without large batches for normalized SGD (while keeping the same sample complexity as a large batch normalized SGD). However, their analysis is specific to the normalized method, which allows using the function value as a Lyapunov function. High probability analysis of momentum methods was investigated in \citep{Cutkosky_High_Prob_Tails_2021,Li_Orabona_High_Prob_2020}. In the distributed setting, \citep{Yu_LinearSpeedup_Com_Efficient_SGDM_2019,Karimireddy_Mime_2021} extend the analysis of \algname{SGDM} under BGS assumption. Later \citep{Takezawa_Mom_Tracking_2022,Gao_Distr_Stoch_Grad_Tracking_2023} remove this assumption providing a refined analysis based on the techniques developed in \citep{Liu_Improved_analysis_SGDM_2020}. However, the algorithms in these works do not apply any bandwidth reduction technique such as communication compression.

We would like to mention that understanding the behavior of \algname{SGDM} in convex case also remains an active area of research  \citep{Ghadimi_Glob_HB_Convex_2014,Yang_Unified_momentum_2016,Sebbouh_AS_Conv_SHB_2021,Li_On_Last_Iterate_Mom_2022,Xiao_Conv_SGD_type_2022} . 
 

\clearpage
\section{Variance Reduction Effect of SGDM and Comparison to STORM} 
Notice that the choice of our Lyapunov function $\Lambda_t$ \eqref{eq:Lyapunov}, which is used in the analysis of  \algname{EF21-SGDM}  implies that the gradient estimators $g_i^t$ and $v_i^t$ improve over the iterations, i.e., 
$$g_i^t \rightarrow  \nabla f_i(x^t),  \qquad v_i^t \rightarrow \nabla f_i(x^t) \qquad \text{for } t\rightarrow \infty .$$
 This comes in contrast with the behavior of SGD, for which the gradient estimator $v_i^t = \nabla f_i(x^t, \xi_i^{t})$ does not necessarily tend to zero over iterations. Such effect of asymptotic improvement of the estimation error of the gradient estimator is reminiscent to the analogous effect known in the literature on variance reduction (VR) methods. In particular, the classical momentum step \ref{alg:eq:sgdm} of Algorithm~\ref{alg:EF21-M} may be contrasted with a \algname{STORM} variance reduced estimator proposed by \citet{Cutkosky_STORM_2019}, which updates the gradient estimator via
\begin{eqnarray}\label{eq:storm_est}
	w_i^{t+1} =   \nabla f_{i}(x^{t+1}, \xi_{i}^{t+1})  +  (1-\eta) ( w_i^{t} - \nabla f_i(x^{t}, \xi_i^{t+1}) )  , \quad w_i^{0} = \nabla f_i(x^0, \xi_i^0)
\end{eqnarray}
It is known that the class of VR methods (and \algname{STORM}, in particular) can show faster asymptotic convergence in terms of $T$ (or $\varepsilon$) compared to \algname{SGD} and \algname{SGDM} under some additional assumptions. However, we would like to point out the important differences (and limitations) of \eqref{eq:storm_est} compared to the classical Polyak's momentum used on line \ref{alg:eq:sgdm} of Algorithm~\ref{alg:EF21-M}. First, the estimator $w_i^{t+1}$ is different from the momentum update rule $v_i^{t+1}$ in that it is unbiased for any $t\geq0$, i.e., $\Exp{w_i^{t+1} - \nabla f_i(x^{t+1}) } = 0$,\footnote{ Notice that $\Exp{w_i^0 - \nabla f_i(x^0) } = 0$. Let $\Exp{w_i^t - \nabla f_i(x^t) } = 0 $ hold, then $$\Exp{w_i^{t+1} - \nabla f_{i}(x^{t+1}) } =  (1-\eta)  \Exp{ w_i^{t}    -  \nabla f_i(x^{t} , \xi_i^{t+1} )  } = 0 . $$ }
which greatly facilitates the analysis of this method. Notice that, in particular, in the deterministic case ($\sigma=0$, $\alpha = 1$), the method with estimator \eqref{eq:storm_est} reduces to vanilla gradient descent with $w_i^{t+1} = \nabla f_i(x^{t+1})$. Second, the computation of $w_i^{t+1}$ requires access to two stochastic gradients $\nabla f_{i}(x^{t+1}, \xi_{i}^{t+1}) $ and $\nabla f_{i}(x^{t}, \xi_{i}^{t+1}) $ under the same realization of noise $\xi_i^{t+1}$ at each iteration, and requires the additional storage of control variate $x^{t}$. This is a serious limitation, which can make the method impractical or even not implementable for certain applications such as federated RL \citep{Mitra_TD_EF_2023}, multi-agent  RL \citep{Doan_FT_TD_4_MARL_2019} or operations research problems~\citep{Chen_Network_Revenue_2022}. Third, the analysis of variance reduced methods such as \algname{STORM} requires an additional assumptions such as individual smoothness of stochastic functions (or its averaged variants) (Assumption~\ref{as:ind_smoothness}), i.e., $\norm{\nabla f_{i}(x,\xi_i) - \nabla f_{i}(y, \xi_i)} \le \ell_i\norm{x - y} $ for all $x,y\in \R^d$, $\xi_i \sim \cD_i$, $i\in[n]$, while our  \algname{EF21-SGDM}  only needs  smoothness of (deterministic) local functions $f_i(x)$. While this assumption is satisfied for some loss functions in supervised learning, it can also be very limiting. Even if  Assumption~\ref{as:ind_smoothness} is satisfied, the constant $\wt \ell$ (which always satisfies $\wt \ell \geq \wt L$) can be much larger than $\wt L$ canceling the speed-up in terms of $T$ (or $\varepsilon$). For completeness, we provide the sample complexity analysis of our error compensated method combined with estimator \eqref{eq:storm_est}, which is deferred to Appendix~\ref{sec:appendix_STORM}.

\clearpage
\section{Additional Experiments and Details of Experimental Setup}
\label{sec:add_experiments}

\paragraph{Divergence of \algname{EF21-SGD} with time-varying step-sizes.} We complement our Figure~\ref{fig:divergence} in the main part of the paper, which shows divergence of \algname{EF21-SGD} \citep{EF21BW_2021} with small (constant) step-size. Here, in Figure~\ref{fig:divergence_var}, we see that the similar divergence is observed when using time varying step-sizes $\gamma_t = \frac{0.1}{\sqrt{t+1}}$.  Also, \algname{EF21-SGD} with time-varying step-size does not improve convergence when $n$ is increased. 

\begin{figure}
	\centering
	\begin{subfigure}{.5\textwidth}
		\centering
		\includegraphics[width=0.9\linewidth]{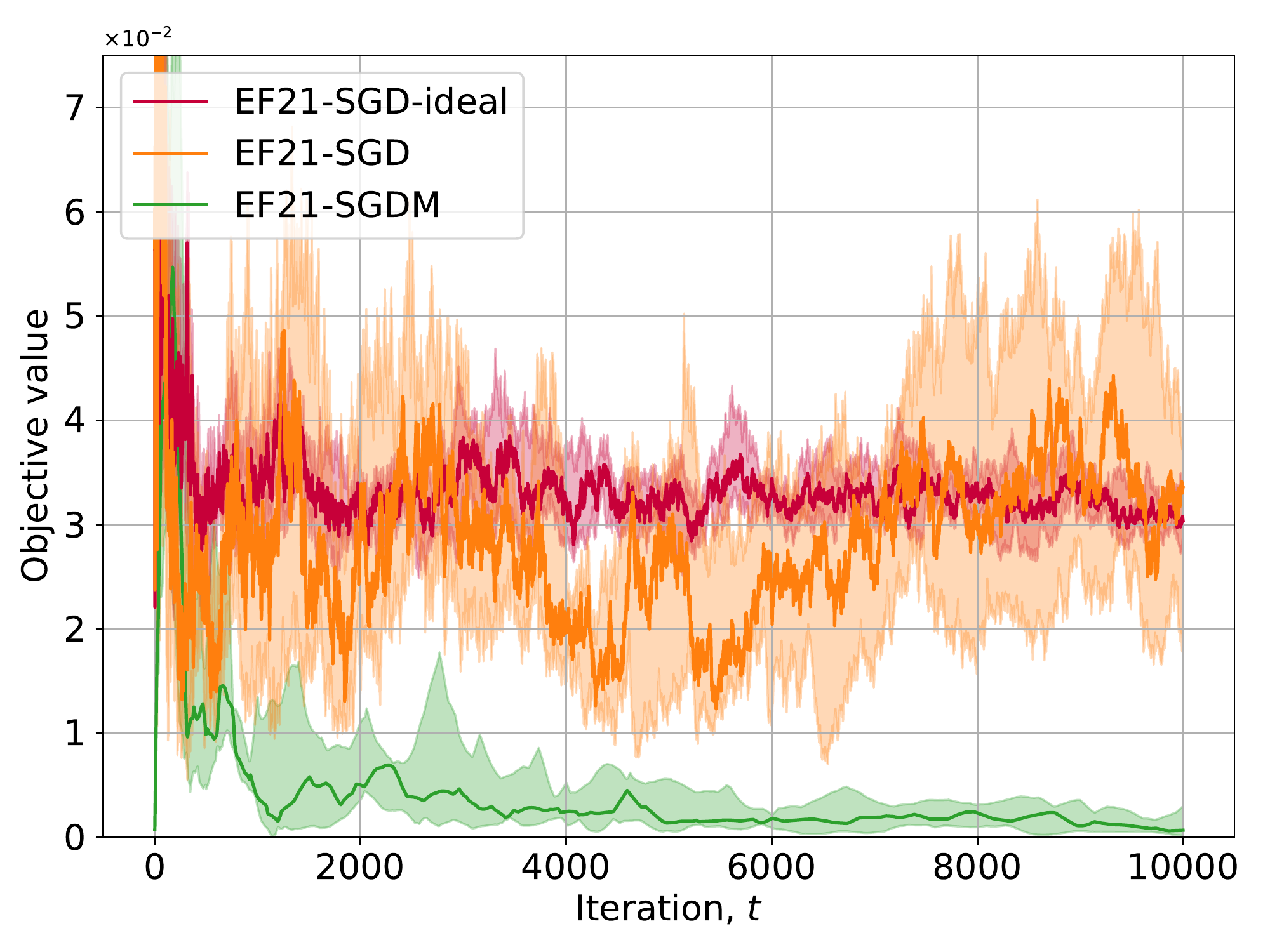}
		\caption{Divergence in single node setting, $n=1$. }
		\label{fig:diverg_var_sz}
	\end{subfigure}%
	\begin{subfigure}{.5\textwidth}
		\centering
		\includegraphics[width=0.9\linewidth]{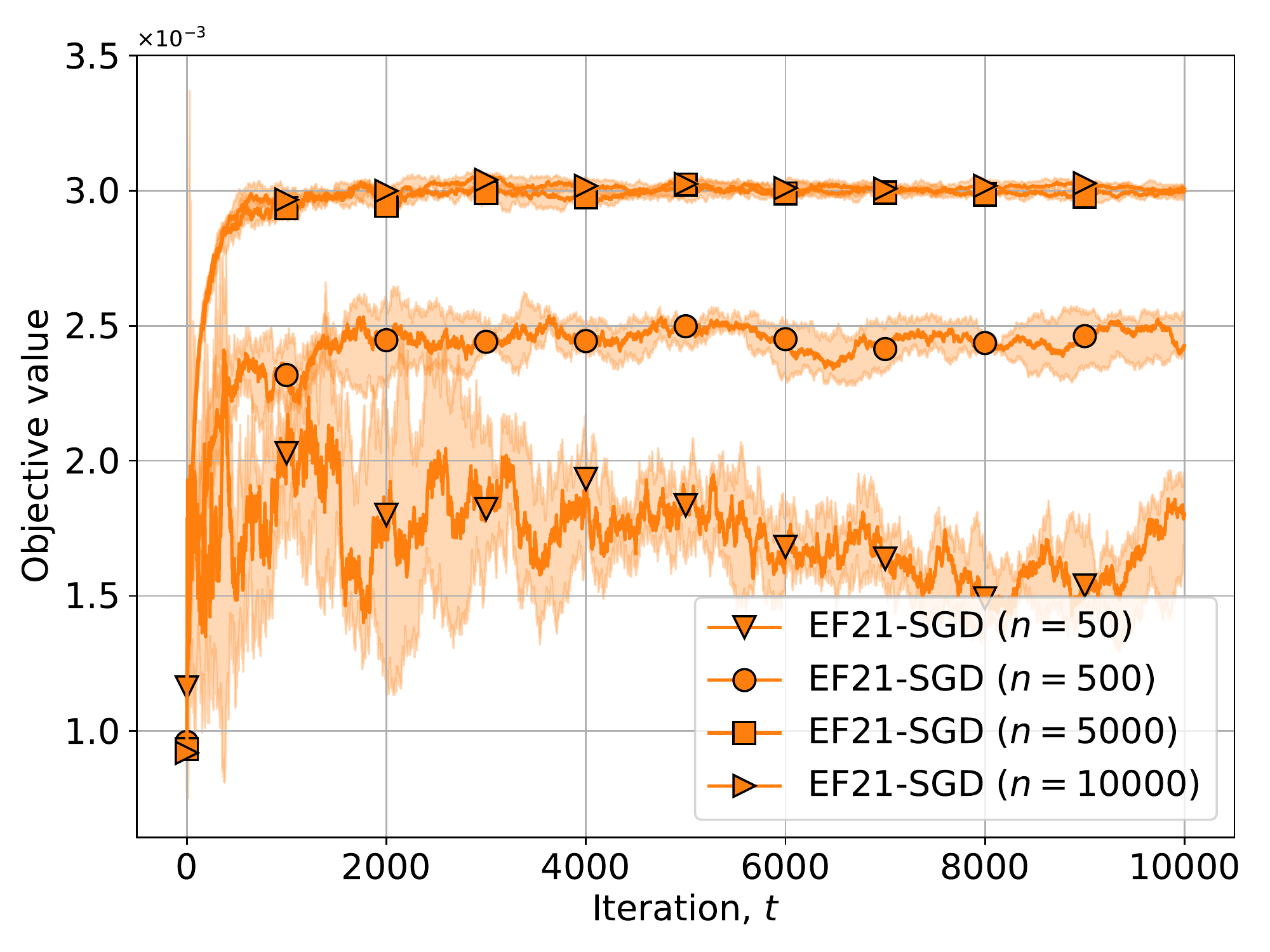}
		\caption{No improvement with $n$.}
		\label{fig:no_speedup_var_sz}
	\end{subfigure}
	\caption{Divergence of \algname{EF21-SGD} on a quadratic function $\frac{1}{2}\sqnorm{x}$ with Top$1$ compressor. See the proof of Therem~\ref{thm:non_convergence_ef21_like_SGD} for details on the construction of noise $\xi$, we use $\sigma = 1$, $B = 1$. The starting point for all methods is $x^0  = (0,-0.01)^{\top}$. 
	Unlike Figure~\ref{fig:divergence}, these experiments use time varying step-sizes and momentum parameters $\gamma_t = \eta_t = \frac{0.1}{\sqrt{t+1}} $. Each method is run $10$ times and the plot shows the median performance alongside the $25\%$ and $75\%$ quantiles. }
	\label{fig:divergence_var}
\end{figure}

\paragraph{Implementation Details.}

The experiments were implemented in Python 3.7.9. The distributed environment was emulated on machines with Intel(R) Xeon(R) Gold 6248 CPU @ 2.50GHz. In all experiments with \textit{MNIST}, we split the dataset across nodes by labels to simulate the heterogeneous setting.

\subsection{Extra plots for experiments 1 and 2}
In Figures~\ref{fig:log_reg_batch_increase_n_100} and \ref{fig:log_reg_real_sim_mini_batch_B_1}, we provide extra experiments for the setup from Section~\ref{sec:experiments_main}.

\begin{figure}[h]
	\centering
	\begin{subfigure}{.33\textwidth}
		\centering
		\includegraphics[width=\textwidth]{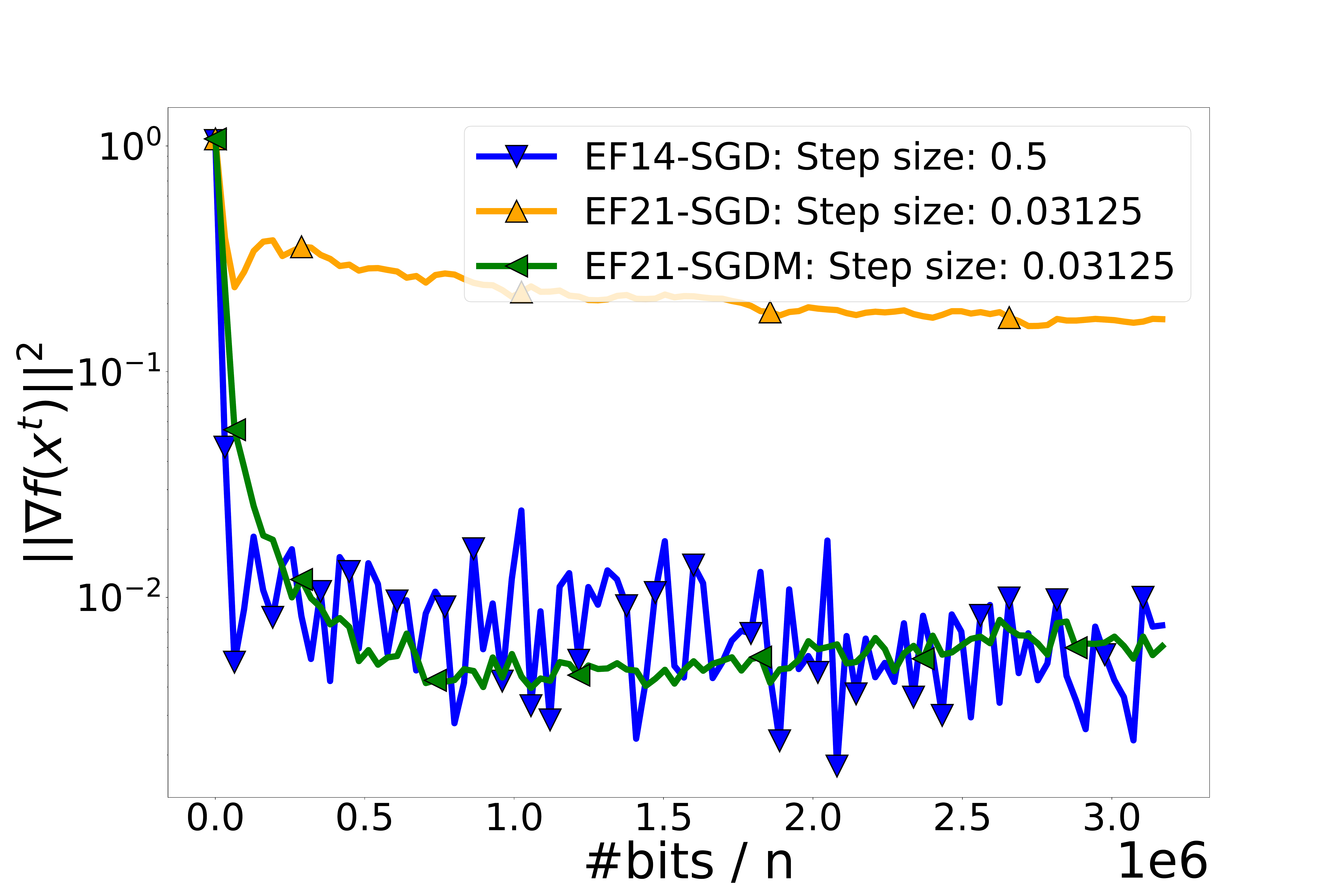}
		\caption{ $B = 1$}
	\end{subfigure}\hfill
	\begin{subfigure}{.33\textwidth}
		\centering
		\includegraphics[width=\textwidth]{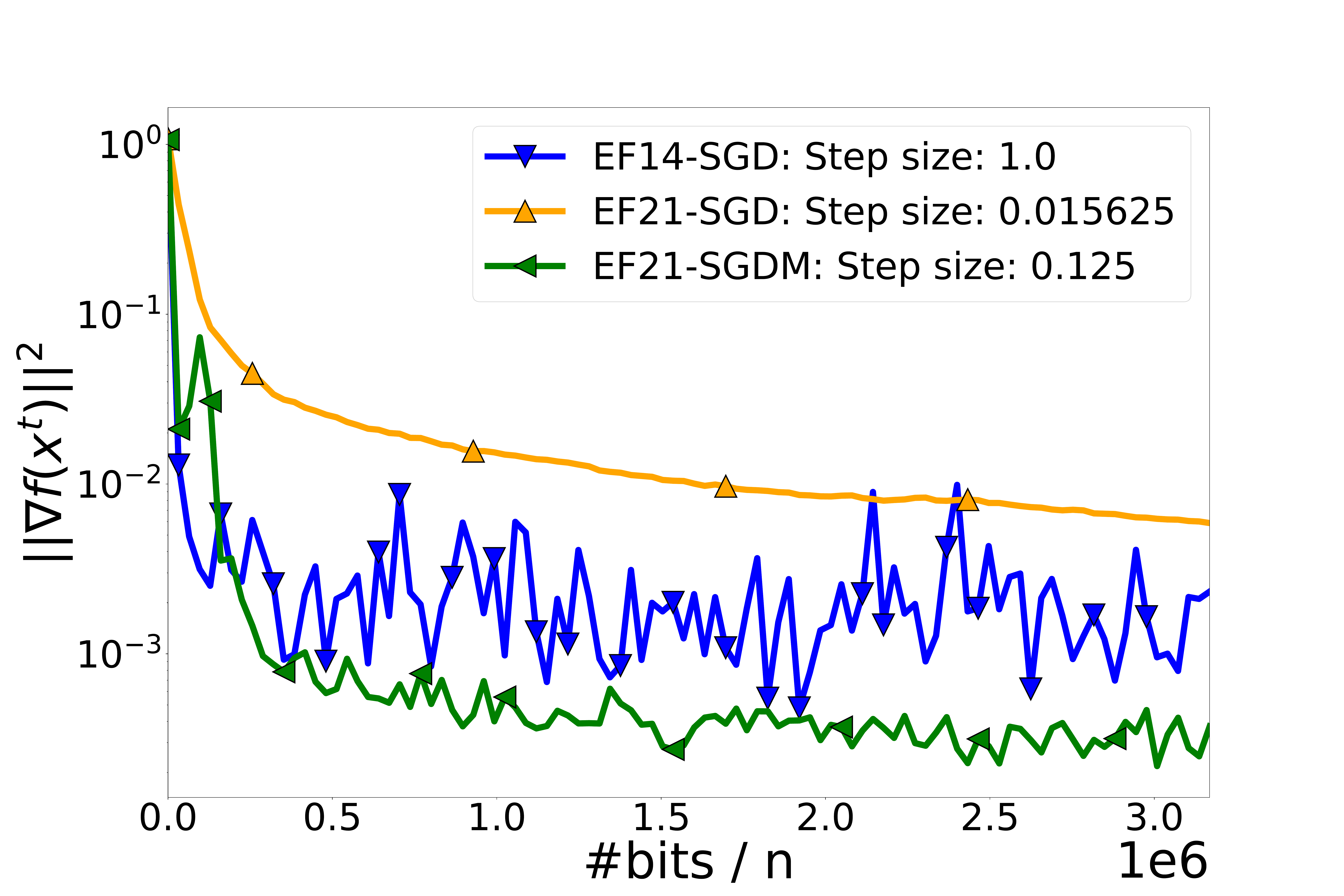}
		\caption{ $B = 32$}
	\end{subfigure}\hfill
	\begin{subfigure}{.33\textwidth}
		\centering
		\includegraphics[width=\textwidth]{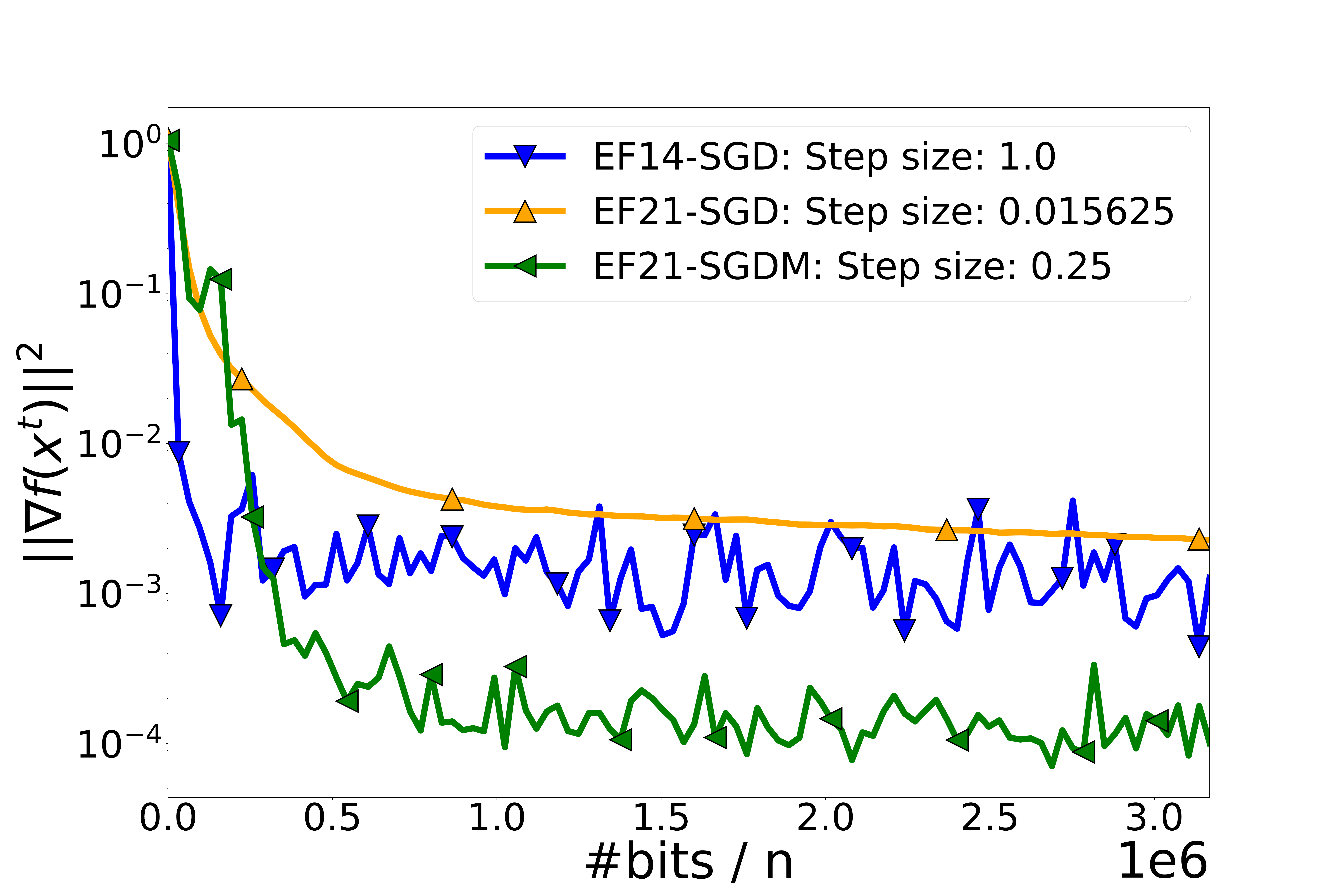}
		\caption{ $B = 128$}
	\end{subfigure}\hfill
	\caption{Performance of algorithms on \textit{MNIST} dataset with $n = 100$, and Top$10$ compressor.}
	\label{fig:log_reg_batch_increase_n_100}
\end{figure}

\begin{figure}[h]
	\centering
	\begin{subfigure}{.33\textwidth}
		\centering
		\includegraphics[width=\textwidth]{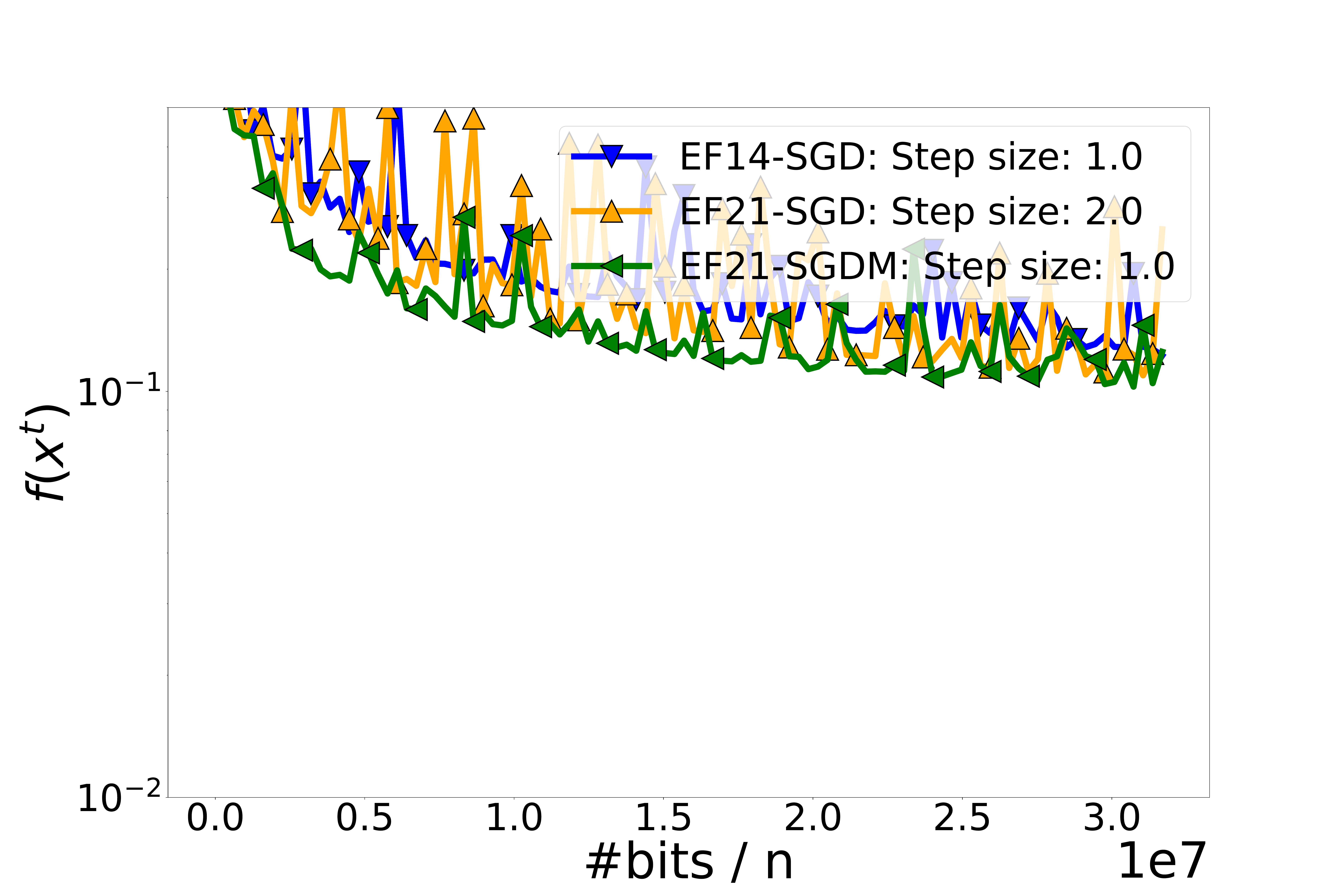}
		\caption{ $n = 1$}
	\end{subfigure}\hfill
	\begin{subfigure}{.33\textwidth}
		\centering
		\includegraphics[width=\textwidth]{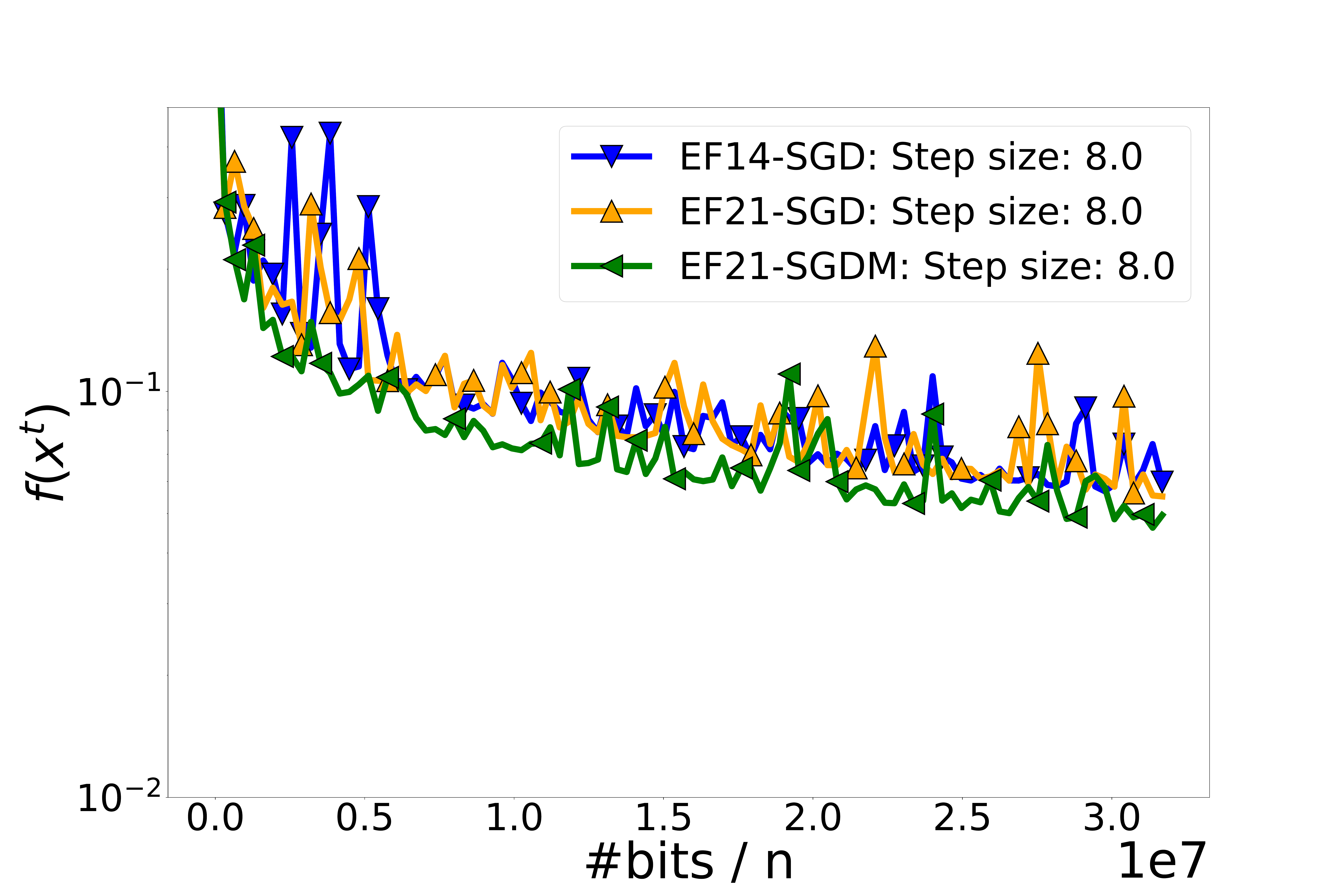}
		\caption{ $n = 10$}
	\end{subfigure}\hfill
	\begin{subfigure}{.33\textwidth}
		\centering
		\includegraphics[width=\textwidth]{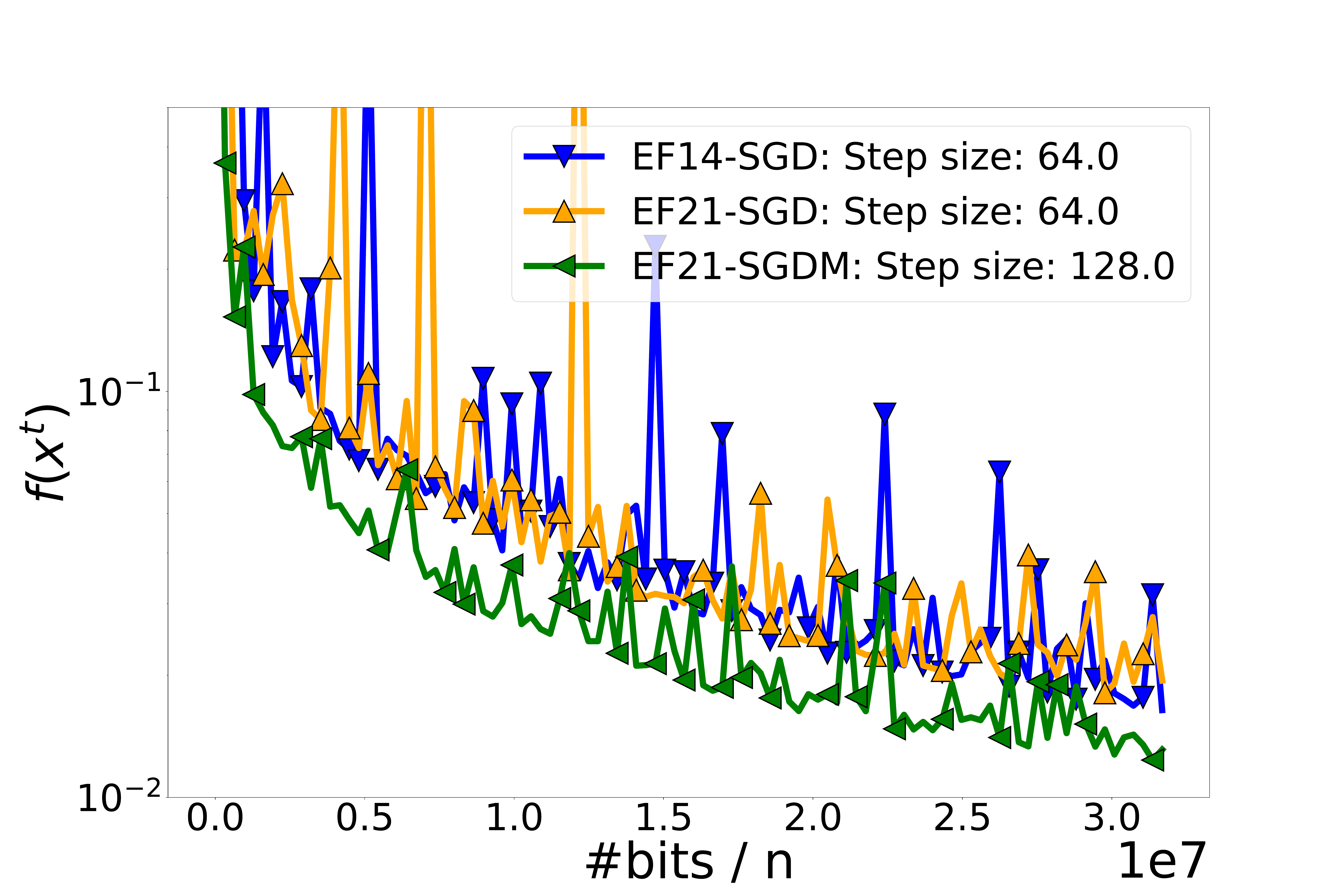}
		\caption{  $n = 100$}
	\end{subfigure}\hfill
	\caption{Performance of algorithms on \textit{real-sim} dataset with batch-size $B = 1$, and Top$100$ compressor.}
	\label{fig:log_reg_real_sim_mini_batch_B_1}
\end{figure}



\subsection{Experiment 3: stochastic quadratic optimization}
We now consider a synthetic $\lambda$--strongly convex quadratic function problem $f(x) = \frac{1}{n}\sum_{i=1}^{n} f_i(x),$ where the functions $\textstyle f_i(x) = \frac{1}{2}x^\top \mQ_i x - x^\top b_i
$ are (not necessarily convex) quadratic functions for all $i \in [n]$ and $x \in \R^d.$
The matrices $\mQ_1, \cdots, \mQ_n$, vectors $b_1, \cdots, b_n,$ and a starting point $x^0$ are generated by Algorithm~\ref{alg:quad_gen} with the number of nodes $n = 100,$ dimension $d = 1000,$ regularizer $\lambda = 0.01$, and scale $s = 1.$ For all $i \in [n]$ and $x \in \R^d,$ we consider stochastic gradients $\nabla f_i(x, \xi) = \nabla f_i(x) + \xi_i,$ where $\xi_i$ are i.i.d. samples from $\mathcal{N}(0, \sigma)$ with $\sigma \in \{0.001, 0.01\}.$ In Figure~\ref{fig:quad_exp}, we present the comparison of  \algname{EF21-SGDM}  and \algname{EF14-SGD} with three different step sizes. The behavior of methods for other step sizes from the set $\{2^k\,|\, k \in [-20, 20]\}$ follows a similar trend. For every step size, we observe that at the beginning, the methods have almost the same linear rates, but then \algname{EF14-SGD} gets stuck at high accuracies, while  \algname{EF21-SGDM}  continues converging to the lower accuracies.

\begin{figure}[h]
	\centering
	\begin{subfigure}{.49\textwidth}
		\centering
		\includegraphics[width=\textwidth]{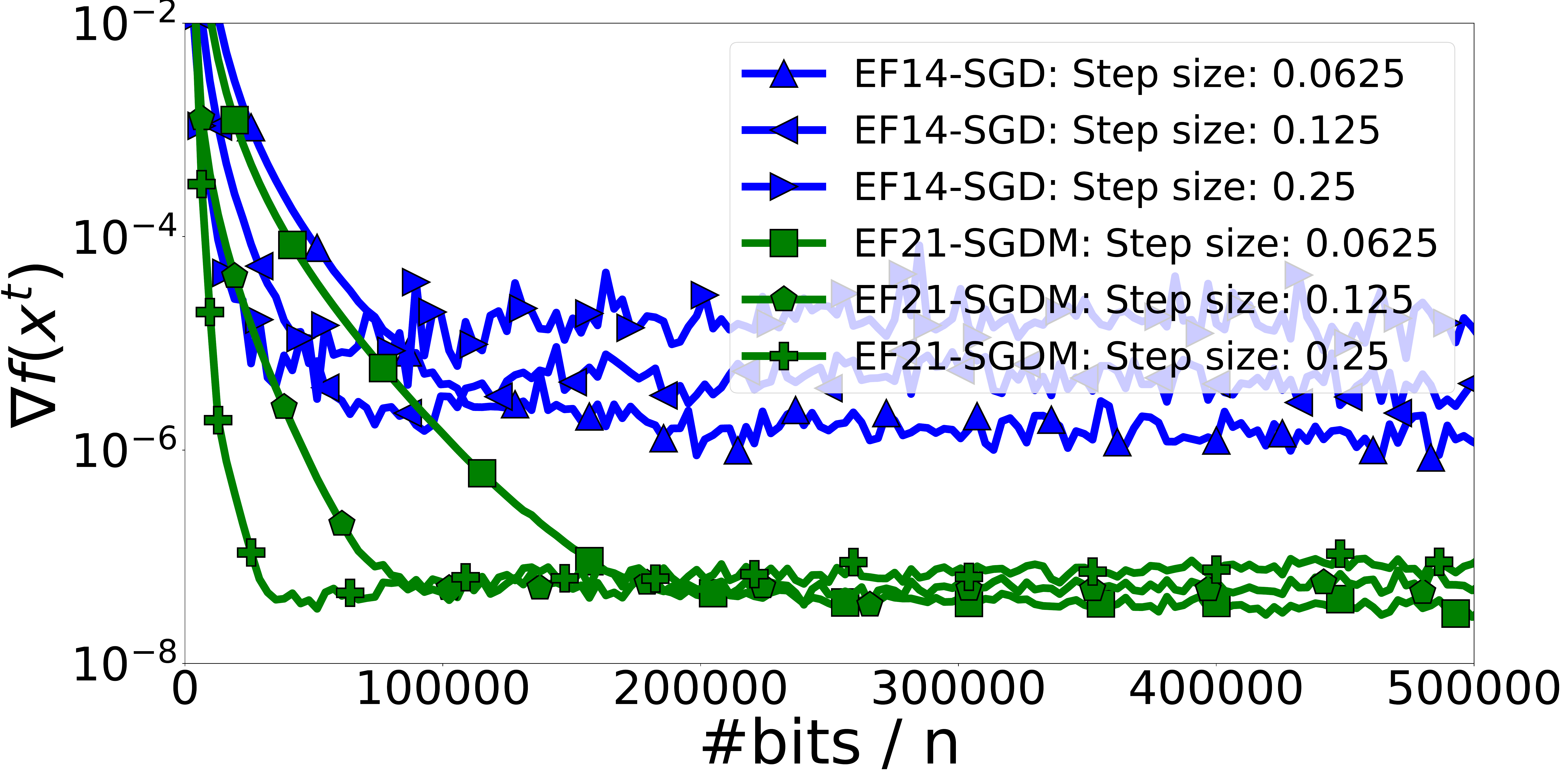}
	\end{subfigure}
	\begin{subfigure}{.49\textwidth}
		\centering
		\includegraphics[width=\textwidth]{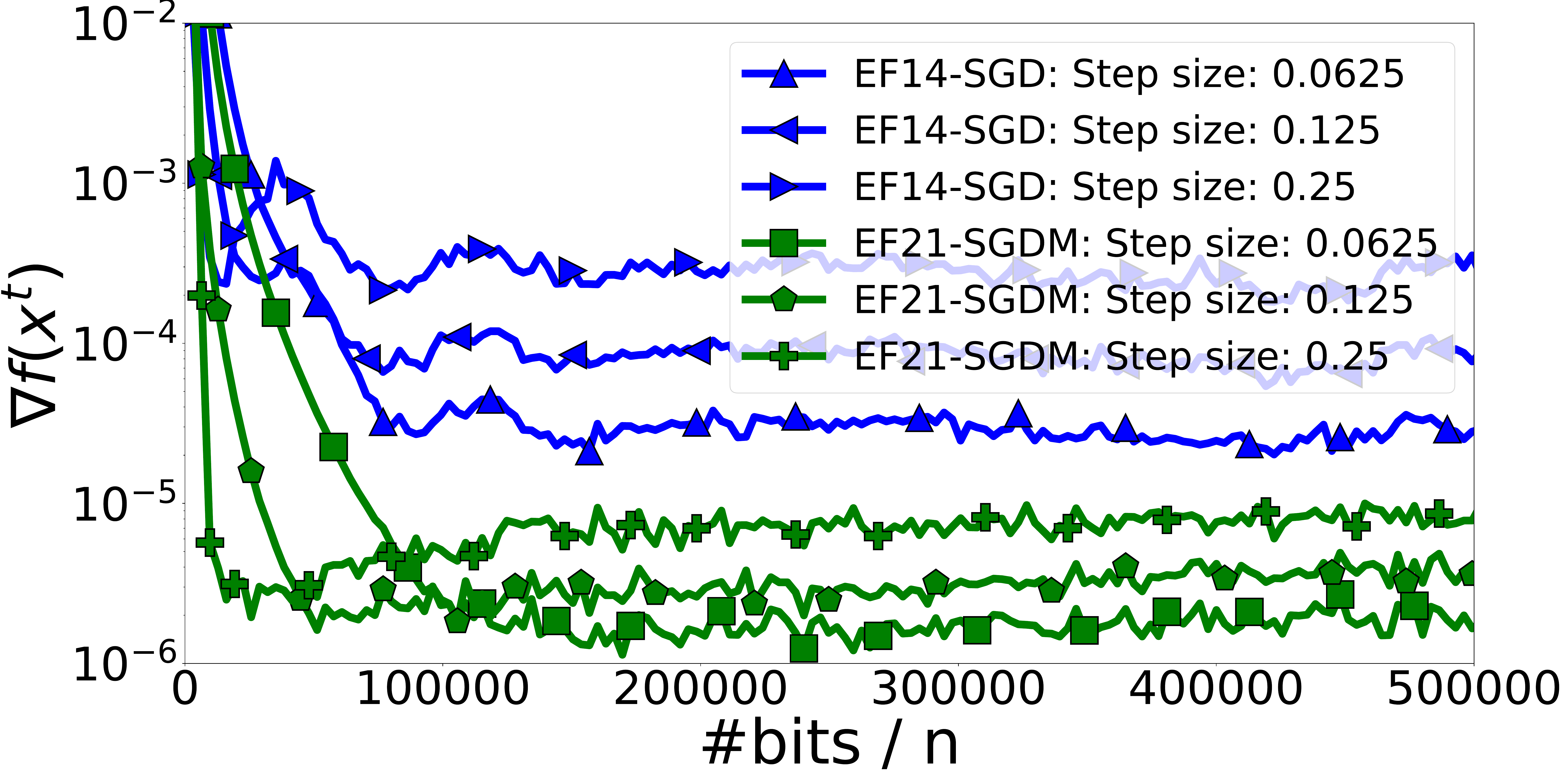}
	\end{subfigure}\hfill
	\caption{Stochastic Quadratic Optimization Problem with $\sigma = 0.001$ (left figure) and $\sigma = 0.01$ (right figure)}
	\label{fig:quad_exp}
\end{figure}

\begin{algorithm}[!h]
	\caption{Quadratic Optimization Task Generation Procedure}
	\label{alg:quad_gen}
	\begin{algorithmic}[1]
		\label{algorithm:matrix_generation}
		\State \textbf{Parameters:} number nodes $n$, dimension $d$, regularizer $\lambda$, and scale $s$.
		\For{$i = 1, \dots, n$}
		\State Calculate Guassian noises $\mu_i^s = 1 + s \xi_i^s$ and $\mu_i^b = s \xi_i^b,$ i.i.d. $\xi_i^s, \xi_i^b \sim \mathcal{N}(0, 1)$
		\State $b_i = \frac{\mu_i^s}{4}(-1 + \mu_i^b, 0, \cdots, 0) \in \R^{d}$
		\State Scale the predefined tridiagonal matrix
		\[\mQ_i = \frac{\mu_i^s}{4}\left( \begin{array}{cccc}
			2 & -1 & & 0\\
			-1 & \ddots & \ddots & \\
			& \ddots & \ddots & -1 \\
			0 & & -1 & 2 \end{array} \right) \in \R^{d \times d}\]
		\EndFor
		\State Find the mean of matrices $\mQ = \frac{1}{n}\sum_{i=1}^n \mQ_i$
		\State Find the minimum eigenvalue $\lambda_{\min}(\mQ)$
		\For{$i = 1, \dots, n$}
		\State Normalize matrix $\mQ_i = \mQ_i + (\lambda - \lambda_{\min}(\mQ)) \mI$
		\EndFor
		\State Find a starting point $x^0 = (\sqrt{d}, 0, \cdots, 0)$
		\State \textbf{Output a new problem:} matrices $\mQ_1, \cdots, \mQ_n$, vectors $b_1, \cdots, b_n$, starting point $x^0$
	\end{algorithmic}
\end{algorithm}

\paragraph{A procedure to generate stochastic quadratic optimization problems.}

In this section, we present an algorithm that generates quadratic optimization tasks. The formal description is provided in Algorithm~\ref{alg:quad_gen}. The idea is to take a predefined tridiagonal matrix and add noises to simulate the heterogeneous setting. Algorithm~\ref{alg:quad_gen} returns matrices $\mQ_1, \cdots, \mQ_n$, vectors $b_1, \cdots, b_n,$ and a starting point $x^0$ such that the matrix $\mQ = \frac{1}{n} \sum_{i=1}^n \mQ_i$ has the minimum eigenvalue $\lambda_{\min}(\mQ) = \lambda,$ where $\lambda \geq 0$ is a parameter. Next, we define the functions $f_i$ and stochastic gradients in the following way:
\begin{align*}
	f_i(x) \eqdef \frac{1}{2}x^\top \mQ_i x - x^\top b_i
\end{align*}
and
\begin{align*}
	\nabla f_i(x, \xi) \eqdef \nabla f_i(x) + \xi_i,
\end{align*}
for all $x \in \R^d$ and $i \in [n].$ The noises $\xi_i$ are i.i.d. samples from $\mathcal{N}(0, \sigma),$ where $\sigma$ is a parameter.

\begin{figure}
	\centering
	\begin{subfigure}{.49\textwidth}
		\centering
		\includegraphics[width=1.0\textwidth]{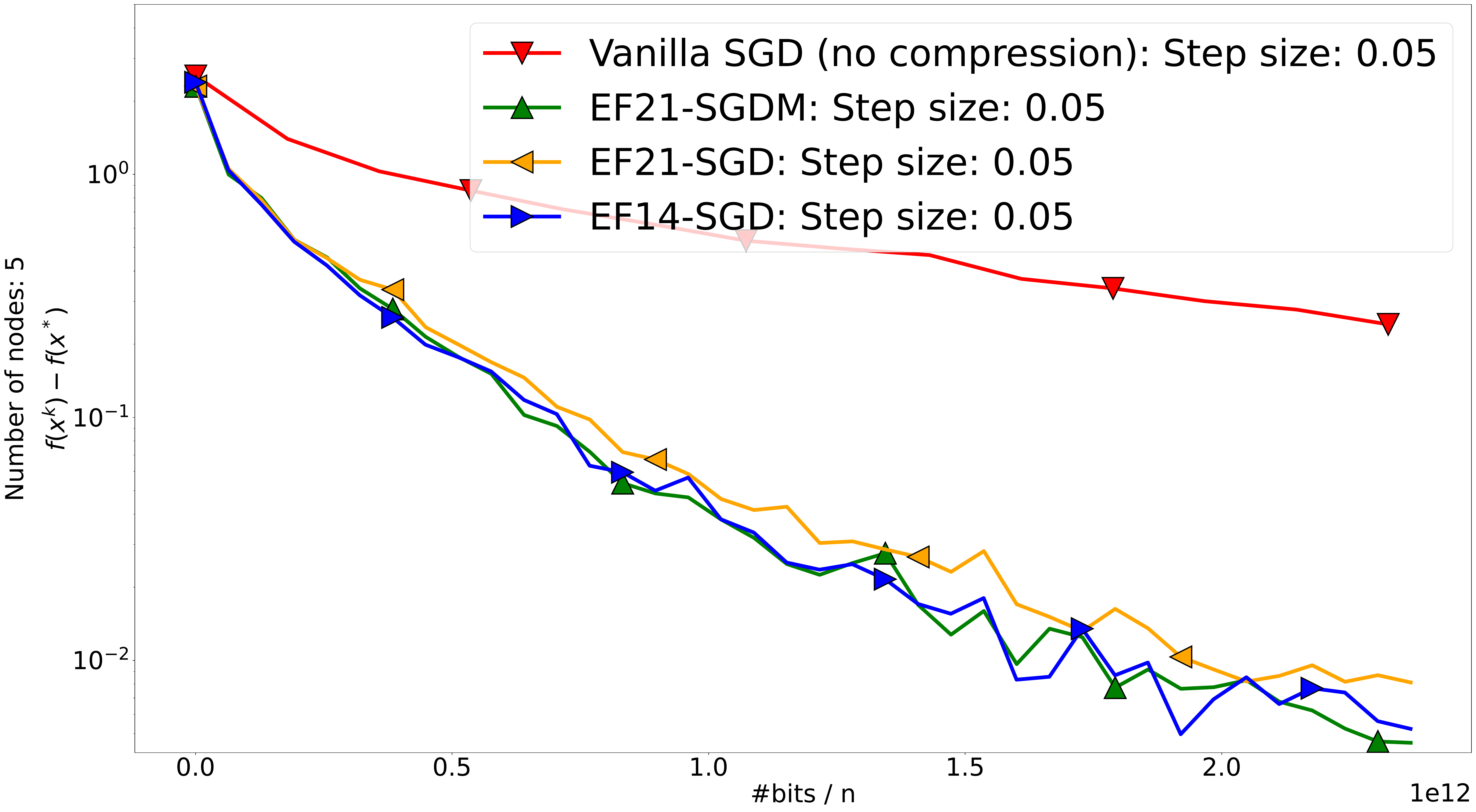}
		\caption{ $B = 8$}
	\end{subfigure}
	\begin{subfigure}{.49\textwidth}
		\centering
		\includegraphics[width=1.0\textwidth]{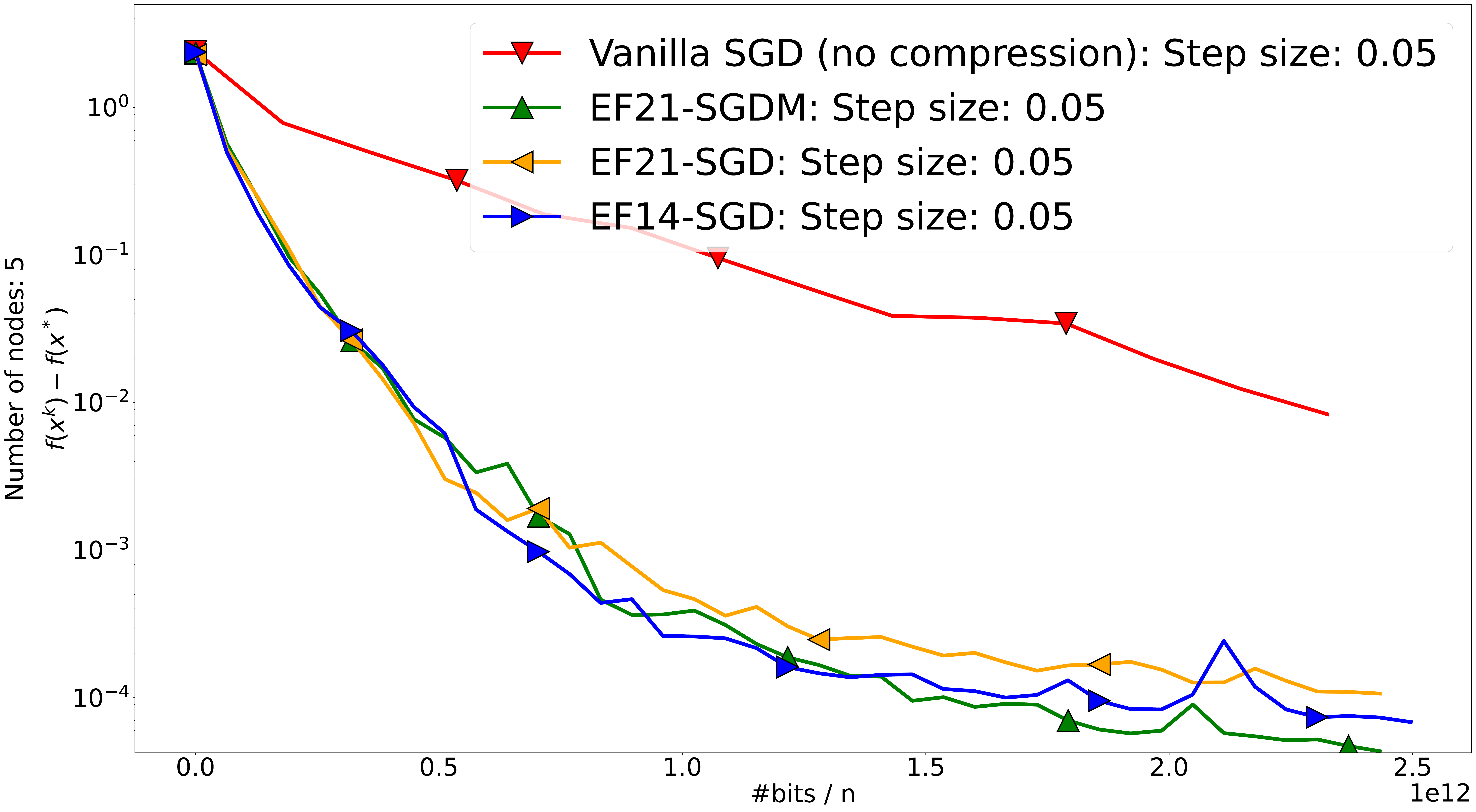}
		\caption{ $B = 25$}
		\label{fig:nn:b25}
	\end{subfigure}
	\caption{\textit{ResNet-18} on \textit{CIFAR10} dataset with $n = 5.$}
	\label{fig:nn}
	\end{figure}
	
	\begin{figure}
		\centering
		\begin{tabular}{|c c|} 
			\hline
			Algorithm & Test Accuracy \\ [0.5ex] 
			\hline
			\algname{SGD} & 81.5 \% \\ 
			\hline
			\algname{EF21-SGD} & 82.5 \% \\
			\hline
			\algname{EF14-SGD} & 83.1 \% \\
			\hline
			\algname{EF21-SGDM} & {\bf 83.3} \% \\
			\hline
		\end{tabular}
		\caption{Accuracy on the \textit{CIFAR10} test split.}
		\label{table:cifar10_test}
	\end{figure}

\subsection{Experiment 4: training neural network}

We test algorithms on an image recognition task, \textit{CIFAR10} \citep{krizhevsky2009learning}, with the \textit{ResNet-18} \citep{he2016deep} deep neural network (the number of parameters $d \approx 10^7$). We split \textit{CIFAR10} among 5 nodes, and take $K = 2 \times 10^6$ in Top$K$. In all methods we finetune the step sizes. One can see that our findings in the low-scale experiments translate into large-scale experiments in Figure~\ref{fig:nn}. With different batch sizes, \algname{EF21-SGD} converges slower than \algname{EF21-SGDM} and \algname{EF14-SGD}, and our new method \algname{EF21-SGDM} improves over \algname{EF14-SGD} in Figure~\ref{fig:nn:b25}. We checked the accuracies on the test dataset (see Table~\ref{table:cifar10_test}) and observed the same relations between algorithms (note that accuracies are far from the real SOTA because we turned off all augmentations and regularizations in training).

\newpage
\section{Descent Lemma}

Let us state the following lemma that is used in the analysis of nonconvex optimization methods.

\begin{lemma}[\citep{PAGE2021}]\label{le:descent}
	Let the function $f(\cdot)$ be $L$-smooth and let $x^{t+1} = x^t -\gamma g^t$ for some vector $g^t \in \R^d$ and a step-size $\gamma > 0$. Then we have 
	\begin{equation}\label{eq:descent}
		f(x^{t+1})  \leq f(x^{t}) - \fr{\g}{2} \sqnorm{\nabla f(x^t)} - \left(\frac{1}{2 \gamma} - \frac{L}{2}\right) \sqnorm{x^{t+1} - x^{t}}  + \frac{\gamma}{2} \sqnorm{g^t - \nabla f(x^t) } .
	\end{equation}
\end{lemma}

\section{EF21-SGDM-ideal (Proof of Theorem~\ref{thm:non_convergence_ef21_like_SGD} and  Proposition~\ref{prop:ef21_m_SGDv0})}
We now state a slighly more general result than Theorem~\ref{thm:non_convergence_ef21_like_SGD}, which holds for \algname{EF21-SGDM-ideal} method with any $\eta \in (0, 1]$. The statement of Theorem~\ref{thm:non_convergence_ef21_like_SGD} follows by setting $\eta = 1$, since in that case \algname{EF21-SGDM-ideal} coincides with \algname{EF21-SGD-ideal} \eqref{eq:ef21_like_sgd_1}, \eqref{eq:ef21_like_sgd_2}. Recall that \algname{EF21-SGDM-ideal} (distributed variant) has the following update rule:
\begin{equation}
	\label{eq:x_update_appendix}
	x^{t+1} = x^t - \gamma g^t, \qquad g^t = \suminn g_i^t , 
\end{equation}
\begin{align}
	\text{\algname{EF21-SGDM-ideal:}}&\qquad
	\begin{split}
		v_i^{t+1} &= \textcolor{mygreen}{\nabla f_i(x^{t+1}) } + \eta (\nabla f_i(x^{t+1}, \xi_i^{t+1})  - \textcolor{mygreen}{\nabla f_i(x^{t+1}) }  ) , \\
		g_i^{t+1} &= \textcolor{blue}{ \nabla f_i(x^{t+1}) } + \cC\rb{  v_i^{t+1}   - \textcolor{blue}{ \nabla f_i(x^{t+1}) } } . 
	\end{split} \label{eq:EF21-SGDMv0_dist_appendix} 
\end{align}
\begin{theorem}\label{thm:non_convergence_EF21_SGDMv0_appendix}
	Let $L, \sigma >0$, $0<\gamma\leq1/L$, $0 <\eta \leq 1$ and $n=1.$ There exists a convex, $L$-smooth function $f(\cdot)$, a contractive compressor $\cC(\cdot)$ satisfying Definition~\ref{def:contractive_compressor}, and an unbiased stochastic gradient with bounded variance $\sigma^2$ such that if the method \eqref{eq:x_update_appendix}, \eqref{eq:EF21-SGDMv0_dist_appendix} is  run with a step-size $\gamma$, then for all $T \geq 0$ and for all $x^0 \in \{(0,  x_{(2)}^{0})^{\top} \in \R^2 \,|\, x_{(2)}^0 < 0\},$ we have
	$$
	\Exp{\sqnorm{ \nabla f(x^T) } } \geq  \frac{1}{60} \min\left\{ \eta^2  \sigma^2, \sqnorm{ \nabla f(x^0) }\right\} . 
	$$
	Fix $0 < \varepsilon \leq L /\sqrt{60}$ and $x^0 = (0,  -1)^{\top}.$ Additionally assume that $n \geq 1$ 
	and the variance of unbiased stochastic gradient is controlled by $\nicefrac{\sigma^2}{B}$ for some $B\geq1$. If $B < \frac{ \eta^2 \sigma^2}{60 \varepsilon^2}$, then we have $\Exp{\norm{\nabla f(x^T)}} > \varepsilon $ for all $T \geq 0$. 
\end{theorem}	
\begin{proofof}{Theorem~\ref{thm:non_convergence_ef21_like_SGD}}
	\textbf{Part I.}
	Consider $f(x) = \frac{L}{2} \sqnorm{x}$, $x \in \R^2$. 
	For each iteration $t\geq 0$, let the random vector $\xi^{t+1}$ be sampled uniformly at random from the set of vectors: 
	$$
	z_1 = \begin{pmatrix} 2 \\  0 \end{pmatrix} \sqrt{\frac{3 \sigma^2}{10}}, \quad z_2 = \begin{pmatrix} 0 \\ 1 \end{pmatrix} \sqrt{\frac{3 \sigma^2}{10}} , \quad  z_3 = \begin{pmatrix} -2 \\  -1 \end{pmatrix} \sqrt{\frac{3 \sigma^2}{10}} .
	$$
	Define the stochastic gradient as $\nabla f(x^{t}, \xi^{t})  \eqdef  \nabla f(x^{t}) + \xi^{t}=  L x^{t} + \xi^{t}$. Notice that $\Exp{\nabla f(x^{t}, \xi^{t})} = \nabla f(x^{t})$, and  $\Exp{\sqnorm{\nabla f(x^{t}, \xi^{t}) - \nabla f(x^t)} } = \sigma^2 $. The update rule of method \eqref{eq:x_update_appendix}, \eqref{eq:EF21-SGDMv0_dist_appendix} with such estimator is $$x^{t+1} = x^t - \gamma  g^t = x^t - L \gamma  x^t - \gamma \cC\rb{ \eta \, \xi^{t} },$$ where we choose $\cC(\cdot)$ as a Top$1$ compressor. Notice that $\Exp{\xi^{t}} = (0, 0)^{\top}$, but $$\Exp{\cC(\xi^{t})} = \eta \, \sqrt{\frac{3 \sigma^2}{10}} (0, 1/3)^{\top} \neq (0, 0)^{\top}.$$ By setting the initial iterate to $x^0 = (0,  x_{(2)}^{0})^{\top}$ for any $x_{(2)}^0 < 0$, we can derive
	\begin{eqnarray}
		\Exp{x^T} &=& (1- L \gamma)^{T} x^0 - \eta \, \sqrt{\frac{3 \sigma^2}{10}}\begin{pmatrix} 0 \\  \fr{1}{3} \end{pmatrix} \gamma \sum_{t=0}^{T-1} (1-L\gamma)^t \notag \\
		& =& (1- L \gamma)^{T} \begin{pmatrix} 0 \\  x_{(2)}^0 \end{pmatrix} + \frac{\eta }{L} \sqrt{\frac{ \sigma^2}{30}} \begin{pmatrix} 0 \\  -1 \end{pmatrix} (1 -  (1-L\gamma)^T ) \neq \begin{pmatrix} 0 \\  0 \end{pmatrix}  \label{eq:exp_x^T}
	\end{eqnarray}
	for any $0 \leq  \gamma \leq 1/L$ and any $x_{(2)}^0 < 0$. The inequality in \eqref{eq:exp_x^T} is because the first vector has strictly negative component $x_{(2)}^0$, and the second vector has non-positive second component when $\gamma > 0$ and $\sigma^2 > 0$.  Therefore, since $\sqnorm{\nabla f(x) } = \sqnorm{L x}$, we have 
	\begin{eqnarray*}
		\Exp{\sqnorm{ \nabla f(x^T) } } &=&   \Exp{\sqnorm{L x^T } }  \\
		&=&   \sqnorm{\Exp{ L x^T } }  + \Exp{\sqnorm{ L x^T -  \Exp{ L x^T } }  }\\
		& \geq &   \sqnorm{\Exp{ L x^T } }   \\
		&\overset{(i)}{=}&  \rb{ (1-L\gamma)^{T}  \norm{L x^0} + \eta \, \sqrt{\frac{ \sigma^2}{30}} (1 -  (1- L\gamma)^T )  }^2  \\
		&\overset{(ii)}{\geq}&   (1-L\gamma)^{2T}  \sqnorm{\nabla f( x^0) } + \frac{ \eta^2 \sigma^2}{30} (1 -  (1- L\gamma)^T )^2  \\
		& \geq &   \frac{  \sqnorm{ \nabla f(x^0) } \eta^2 \sigma^2}{30  \sqnorm{ \nabla f(x^0) } + \eta^2 \sigma^2 }  
	\end{eqnarray*}
	for all $T \geq 1$, where in $(i)$ we used the form of vector $\Exp{x^T}$ in \eqref{eq:exp_x^T}, in $(ii)$ we drop a non-negative cross term, and use $\nabla f(x^0) =L x^0$. The last inequality follows by lower bounding a univariate quadratic function with respect to $z  \eqdef (1- L\gamma)^T$ for $0\leq z \leq 1$, where optimal choice is $z = \eta^2 \sigma^2 / (30  \sqnorm{ \nabla f(x^0) } + \eta^2 \sigma^2)  $. It is left to note that $\frac{x y}{x + y} \geq \frac{1}{2}\min\{x, y\}$ for all $x, y > 0.$
	
	\textbf{Part II.} Fix $n\geq 1$ and $B \geq 1$. Let at each node $i = 1, \ldots, n$, the random vectors $\xi_i^{t}$ be sampled independently and uniformly form the set of vectors:
	$$
	z_1 = \begin{pmatrix} 2 \\  0 \end{pmatrix} \sqrt{\frac{3 \sigma^2}{10 B }}, \quad z_2 = \begin{pmatrix} 0 \\ 1 \end{pmatrix} \sqrt{\frac{3 \sigma^2}{10 B }} , \quad  z_3 = \begin{pmatrix} -2 \\  -1 \end{pmatrix} \sqrt{\frac{3 \sigma^2}{10 B }} .
	$$
	Define a random matrix $\xi^t  \eqdef (\xi_1^{t}, \ldots, \xi_n^{t})^{\top}$. Then $\Exp{\sqnorm{\nabla f(x^{t}, \xi^{t}) - \nabla f(x^t)} } = \frac{\sigma^2}{B} $. The update of the method \eqref{eq:x_update_appendix}, \eqref{eq:EF21-SGDMv0_dist_appendix} on the same function instance will take the form 
	$$x^{t+1} = x^t - \gamma  \suminn g_i^t = x^t - L \gamma  x^t - \gamma \suminn \cC\rb{ \eta\, \xi_i^{t} } . 
	$$
	Notice that in this case, we still have 
	$$
	\Exp{\suminn \cC(\eta \, \xi_i^{t})} = \suminn \Exp{\cC( \eta \, \xi_i^{t})}  = \eta \, \sqrt{\frac{3 \sigma^2}{10}} (0, 1/3)^{\top} \neq (0, 0)^{\top},
	$$
	which is independent (!) of $n$. Therefore, by similar derivations, we can conclude that 
	\begin{eqnarray*}
		\Exp{\sqnorm{ \nabla f(x^T) } }   &\geq&  \frac{1}{60} \min\left\{\frac{\eta^2 \sigma^2}{B}, \sqnorm{ \nabla f(x^0) }\right\} > \varepsilon^2  ,
	\end{eqnarray*}
	where we use that $B < \frac{\eta^2 \sigma^2}{60 \varepsilon^2}$, $\varepsilon \leq L / \sqrt{60},$ and $x^0 = (0,  -1)^{\top}.$
\end{proofof}

\begin{proofof}{Proposition~\ref{prop:ef21_m_SGDv0}}
	By smoothness (Assumption~\ref{as:main}) of $f(\cdot)$ it follows from Lemma~\ref{le:descent} that for $\gamma \leq 1/L$ we have 
	\begin{eqnarray}\label{eq:descent_ef21-Mvo}
		f(x^{t+1})  \leq f(x^{t}) - \fr{\g}{2} \sqnorm{\nabla f(x^t)} + \frac{\gamma}{2} \sqnorm{g^t - \nabla f(x^t) } .
	\end{eqnarray}
	Now it remains to control the last term, which is due to the error introduced by a contractive compressor and stochastic gradients. We have 
	\begin{eqnarray*}
		\Exp{\sqnorm{g^t - \nabla f(x^t)}} &\overset{(i)}{=}& 	\Exp{\sqnorm{\cC\rb{v^t - \nabla f(x^t)}}} \overset{(ii)}{=}  \Exp{\sqnorm{\cC\rb{ \eta \rb{ \nabla f(x^t, \xi^{t}) - \nabla f(x^t)}}}} \\
		&\overset{(iii)}{\leq}&  2 \Exp{\sqnorm{\cC\rb{ \eta \rb{   \nabla f(x^t, \xi^{t}) - \nabla f(x^t) }} - \eta (\nabla f(x^t, \xi^{t}) - \nabla f(x^t))}} \\
		&& \qquad + 2\eta^2 \Exp{\sqnorm{ \nabla f(x^t, \xi^{t}) - \nabla f(x^t)}}\\
		& \leq  & 2 (2-\alpha) \eta^2 \Exp{\sqnorm{\nabla f(x^t, \xi^{t}) - \nabla f(x^t) }} \\
		& \leq  &4\eta^2 \sigma^2 , 
	\end{eqnarray*}	
	where $(i)$ and $(ii)$ use the update rule~\eqref{eq:EF21-SGDMv0}, $(iii)$ holds by Young's inequality, and the last two steps hold by Definition~\ref{def:contractive_compressor} and Assumption~\ref{as:BV}. 
	
	Subtracting $f^*$ from both sides of \eqref{eq:descent_ef21-Mvo}, taking expectation and defining $\delta_t  \eqdef \Exp{f(x^t) - f^*}$, we derive
	$$
	\Exp{\sqnorm{\nabla f(\hat x^T)}} = \frac{1}{T} \sum_{t=0}^{T-1} \Exp{\sqnorm{\nabla f(x^t)}} \leq  \fr{2 \delta_0 }{\gamma T} + 4\eta^2 \sigma^2  . 
	$$
\end{proofof}	

\clearpage
\section{EF21-SGDM (Proof of Theorems~\ref{thm:ef21-sgdm-one-node} and~\ref{thm:main-distrib})}
The statement of Theorem~\ref{thm:ef21-sgdm-one-node} follows directly from Theorem~\ref{thm:main-distrib} and Remark~\ref{rem:after_main_thm}. Let us prove Theorem~\ref{thm:main-distrib}.

\begin{proofof}{Theorem~\ref{thm:main-distrib}}
	In order to control the error between $g^t$ and $\nabla f (x^t)$, we decompose it into two terms
	$$
	\sqnorm{g^t - \nabla f(x^t)} \leq 2 \sqnorm{g^t - v^t} + 2 \sqnorm{v^t - \nabla f(x^t)} \leq 2 \suminn \sqnorm{g_i^t - v_i^t} + 2 \sqnorm{v^t - \nabla f(x^t)},
	$$
	and develop a recursion for each term above separately.
	
	\textbf{Part I (a). Controlling  the error of momentum estimator for each $v_i^t$.} Recall that by Lemma~\ref{le:key_HB_recursion}-\eqref{eq:wtP_rec}, we have for each $i = 1, \ldots, n$, and any $0 < \eta \leq 1$ and $t\geq 0$
	\begin{eqnarray}\label{eq:wtPi_rec}
		\Exp{ \sqnorm{v_i^{t+1} - \nabla f_i(x^{t+1})} }  \leq (1- \eta)   \Exp{ \sqnorm{v_i^{t} - \nabla f_i(x^{t})} }  +  \fr{3  L_i^2 }{\eta } \Exp{ \sqnorm{ x^{t+1} - x^{t}}  } + \eta^2 \sigma^2 .
	\end{eqnarray}
	Averaging inequalities \eqref{eq:wtPi_rec} over $i=1,\ldots,n$ and denoting $\wt P_t : = \suminn \Exp{ \sqnorm{v_i^{t} - \nabla f_i(x^{t})} } $, $R_t  \eqdef \Exp{\sqnorm{x^{t} - x^{t+1} } } $ we have 
	\begin{eqnarray*}
		\wt P_{t+1}  \leq (1- \eta)   \wt P_t  +  \fr{3  \wt L^2 }{\eta } R_t + \eta^2 \sigma^2 .
	\end{eqnarray*}
	
	Summing up the above inequality for $t=0, \ldots, T-1$, we derive 
	
	\begin{eqnarray}\label{eq:wt_mom_est_avg_dist}
		\frac{1}{T} \sum_{t=0}^{T-1} \wt P_t  \leq  \fr{3 \wt L^2 }{\eta^2 } \frac{1}{T} \sum_{t=0}^{T-1}  R_t +  \eta \sigma^2 + \frac{1}{\eta T} \wt P_0.
	\end{eqnarray}
	
	\textbf{Part I (b). Controlling  the error of momentum estimator for $v^t$ (on average).} Similarly by Lemma~\ref{le:key_HB_recursion}-\eqref{eq:P_rec}, we have for any $0 < \eta \leq 1$ and $t\geq 0$
	
	\begin{eqnarray*}
		\Exp{ \sqnorm{v^{t+1} - \nabla f(x^{t+1})} }  \leq (1- \eta)   \Exp{ \sqnorm{v^{t} - \nabla f(x^{t})} }  +  \fr{3 L^2 }{\eta } \Exp{ \sqnorm{ x^{t+1} - x^{t}}  } + \frac{\eta^2 \sigma^2 }{n} ,
	\end{eqnarray*}
	where $v^t  \eqdef \suminn v_i^t$ is an auxiliary sequence. 
	
	Summing up the above inequality for $t=0, \ldots, T-1$, and denoting $P_t  \eqdef   \Exp{ \sqnorm{ v^t - \nabla f(x^{t})} }$, we derive 
	
	\begin{eqnarray}\label{eq:mom_est_avg_dist}
		\frac{1}{T} \sum_{t=0}^{T-1} P_t  \leq  \fr{3 L^2 }{\eta^2 } \frac{1}{T} \sum_{t=0}^{T-1}  R_t +  \frac{\eta \sigma^2 }{n} + \frac{1}{\eta T} P_0.
	\end{eqnarray}

	\textbf{Part II. Controlling  the error of contractive compressor and momentum estimator.} By Lemma~\ref{le:EF21} we have for each $i = 1, \ldots, n$, and any $0 < \eta \leq 1$ and $t\geq 0$
	
	\begin{eqnarray}\label{eq:wtVi_rec}
		\Exp{ \sqnorm{g_i^{t+1} - v_i^{t+1}} }  &\leq& \rb{ 1-\fr{\al}{2} } \Exp{ \sqnorm{g_i^{t} - v_i^{t}} }  + \frac{4 \eta^2 }{\alpha}  \Exp{\sqnorm{v_i^{t} - \nabla f_i (x^{t})} }  \notag \\
		&& \qquad +  \frac{4 L_i^2 \eta^2 }{\alpha}   \Exp{\sqnorm{x^{t+1} - x^{t}} } +  \eta^2 \sigma^2 .
	\end{eqnarray}
	
	Averaging inequalities \eqref{eq:wtVi_rec} over $i=1,\ldots,n$,  denoting $ \wt V_t  \eqdef \suminn  \Exp{ \sqnorm{g_i^t - v_i^t} }$, and summing up the resulting inequality for $t=0, \ldots, T-1$, we obtain 
	\begin{eqnarray}\label{eq:ef21_mom_est_avg_dist}
		\frac{1}{T} \sum_{t=0}^{T-1} \wt V_t  &\leq&   \fr{8 \eta^2 }{\al^2} \frac{1}{T} \sum_{t=0}^{T-1}  \wt P_t +  \fr{8  \wt L^2 \eta^2 }{\al^2}  \frac{1}{T} \sum_{t=0}^{T-1} R_t +   \fr{2\eta^2 \sigma^2 }{\al} + \frac{2}{\alpha T} \wt V_0 .
	\end{eqnarray}
	
	\textbf{Part III. Combining steps I and II with descent lemma.}	
	By smoothness (Assumption~\ref{as:main}) of $f(\cdot)$ it follows from Lemma~\ref{le:descent} that for any $\gamma \leq 1/(2 L) $ we have 
	\begin{eqnarray}\label{eq:descent_ef21-M_dist}
		f(x^{t+1}) &\leq& f(x^{t}) - \fr{\g}{2} \sqnorm{\nabla f(x^t)}  - \frac{1}{4\gamma} \sqnorm{x^{t+1} - x^t } + \frac{\gamma}{2} \sqnorm{g^t - \nabla f(x^t) } \\
		&\leq& f(x^{t}) - \fr{\g}{2} \sqnorm{\nabla f(x^t)}  -  \frac{1}{4\gamma} \sqnorm{x^{t+1} - x^t } + \gamma \suminn\sqnorm{g_i^t - v_i^t } + \gamma \sqnorm{v^t - \nabla f(x^t) }  . \notag
	\end{eqnarray}
	Subtracting $f^*$ from both sides of \eqref{eq:descent_ef21-M_dist}, taking expectation and defining $\delta_t  \eqdef \Exp{f(x^t) - f^*}$, we derive
	\begin{eqnarray}
		\Exp{\sqnorm{\nabla f(\hat x^T)}} &=& \frac{1}{T} \sum_{t=0}^{T-1} \Exp{\sqnorm{\nabla f(x^t)}} \notag \\
		&\leq&  \fr{2 \delta_0 }{\gamma T} + 2 \frac{1}{T}\sum_{t=0}^{T-1} \wt V_t +  2 \frac{1}{T} \sum_{t=0}^{T-1} P_t  - \frac{1 }{2\gamma^2} \frac{1}{T} \sum_{t=0}^{T-1} R_t \notag  \\
		&\overset{(i)}{\leq}&  \fr{2 \delta_0 }{\gamma T} +  \fr{16 \eta^2 }{\al^2}  \frac{1}{T} \sum_{t=0}^{T-1}  \wt P_t + 2 \frac{1}{T} \sum_{t=0}^{T-1}  P_t   +   \fr{4\eta^2 \sigma^2 }{\al} \notag  \\
		&& \qquad  - \frac{\frac{1}{2} -  \fr{16 \gamma^2 \wL^2 \eta^2 }{\al^2}  }{\gamma^2} \frac{1}{T} \sum_{t=0}^{T-1} R_t \notag  \\
		&\overset{(ii)}{\leq}&  \fr{2 \delta_0 }{\gamma T} +   \fr{16 \eta^3 \sigma^2 }{\al^2}  +   \fr{4\eta^2 \sigma^2 }{\al}   + \frac{2 \eta \sigma^2}{n}  + \frac{4}{\alpha T} \wt V_0 \notag \\
		&& \qquad   - \frac{ \frac{1}{2} - \frac{16 \gamma^2 \wL^2 \eta^2}{\alpha^2} - \frac{6 \gamma^2 L^2}{\eta^2} - \frac{48 \gamma^2 \wL^2 }{\alpha^2} }{\gamma^2} \frac{1}{T} \sum_{t=0}^{T-1} R_t + \frac{2}{\eta T}P_0 + \frac{16 \eta}{\alpha^2 T} \widetilde{P}_0 \notag \\
		&\overset{(iii)}{\leq}&  \fr{2 \delta_0 }{\gamma T} +   \fr{16 \eta^3 \sigma^2 }{\al^2}  +   \fr{4\eta^2 \sigma^2 }{\al}   + \frac{2 \eta \sigma^2}{n}    - \frac{ \frac{1}{2} - \frac{6 \gamma^2 L^2}{\eta^2} - \frac{64 \gamma^2 \wL^2 }{\alpha^2} }{\gamma^2} \frac{1}{T} \sum_{t=0}^{T-1} R_t \notag \\
		&& \qquad + \frac{2}{\eta T}P_0 + \frac{16 \eta}{\alpha^2 T} \widetilde{P}_0 + \frac{4}{\alpha T} \wt V_0\notag  \\
		& \leq &  \fr{2 \delta_0 }{\gamma T} +   \fr{16 \eta^3 \sigma^2 }{\al^2}  +   \fr{4\eta^2 \sigma^2 }{\al}   + \frac{2 \eta \sigma^2}{n} + \frac{2}{\eta T}P_0 + \frac{16 \eta}{\alpha^2 T} \widetilde{P}_0 + \frac{4}{\alpha T} \wt V_0 .\notag
	\end{eqnarray}
	where $(i)$ holds due to \eqref{eq:ef21_mom_est_avg_dist}, $(ii)$ utilizes \eqref{eq:mom_est_avg_dist}, and $(iii)$ follows by $\eta \leq 1$, and the last step holds due to the assumption on the step-size. We proved \eqref{eq:main-distrib:general}.
	
	We now find the particular values of parameters. Since $g_i = v_i$ for all $i\in [n]$, we have $\wt V_0 = 0$. Using $v_i^0 = \frac{1}{B_{\textnormal{init}}} \sum_{j=1}^{B_{\textnormal{init}}} \nabla f_{i}(x^0, \xi_{i, j}^{0})$ for all $i = 1, \ldots,  n$, we have
	\begin{align*}
		P_0 = \Exp{\norm{v^0 - \nabla f(x^0)}^2} \leq \frac{\sigma^2}{n B_{\textnormal{init}}} \textnormal{ and } \widetilde{P}_0 = \frac{1}{n} \sum_{i=1}^n \Exp{\norm{v^0_i - \nabla f_i(x^0)}^2} \leq \frac{\sigma^2}{B_{\textnormal{init}}}.
	\end{align*}
	We can substitute the choice of $\gamma$ and obtain
	\begin{eqnarray*}
		\Exp{  \sqnorm{\nabla f(\hat x^T) }  } &= & \cO\rb{ \fr{\wL \delta_0 }{\alpha T} + \fr{L \delta_0 }{\eta T}  +   \fr{\eta^3 \sigma^2 }{\al^2}  +   \fr{\eta^2 \sigma^2 }{\al}   + \frac{\eta \sigma^2}{n} + \frac{\sigma^2}{\eta n B_{\textnormal{init}} T} + \frac{\eta \sigma^2}{\alpha^2 B_{\textnormal{init}} T}}.
	\end{eqnarray*}
	Since $B_{\textnormal{init}} \geq \frac{\sigma^2}{L \delta_0 n},$ we have
	\begin{eqnarray*}
		\Exp{  \sqnorm{\nabla f(\hat x^T) }  } &= & \cO\rb{ \fr{\wL \delta_0 }{\alpha T} + \fr{L \delta_0 }{\eta T}  +   \fr{\eta^3 \sigma^2 }{\al^2}  +   \fr{\eta^2 \sigma^2 }{\al}  + \frac{\eta \sigma^2}{n} + \frac{\eta \sigma^2}{\alpha^2 B_{\textnormal{init}} T}}.
	\end{eqnarray*}
	
	Notice that the choice of the momentum parameter such that $\eta \leq \rb{ \fr{ L \delta_0 \al^2 }{\sigma^2 T}}^{\nfr{1}{4} }$, $\eta \leq \rb{ \fr{ L \delta_0 \al }{\sigma^2 T}}^{\nfr{1}{3} }$, $\eta \leq \rb{ \fr{ L \delta_0 n  }{\sigma^2 T}}^{\nfr{1}{2} }$ and $\eta \leq \frac{\alpha \sqrt{L \delta_0 B_{\textnormal{init}}}}{\sigma}$ ensures that $\fr{ \eta^3 \sigma^2 }{\al^2}   \leq \fr{L \delta_0 }{\eta T}  $, $\fr{\eta^2 \sigma^2 }{\al} \leq \fr{L \delta_0 }{\eta T}$, $\frac{ \eta \sigma^2}{n} \leq \fr{L \delta_0 }{\eta T},$ and $\frac{\eta \sigma^2}{\alpha^2 B_{\textnormal{init}} T} \leq \fr{L \delta_0 }{\eta T}.$ Therefore, we have
	\begin{eqnarray*}
		\Exp{  \sqnorm{\nabla f(\hat x^T) }  } &\leq & \cO\rb{\fr{\wL \delta_0}{\alpha  T }  +  \rb{\fr{   L \delta_0 \sigma^{2/3} }{\al^{2/3}  T } }^{\nfr{3}{4}} +  \rb{\fr{   L \delta_0 \sigma }{\sqrt{\al}  T } }^{\nfr{2}{3}}  + \rb{\fr{  L \delta_0 \sigma^2}{ n T } }^{\nfr{1}{2}} + \frac{\sigma \sqrt{L \delta_0}}{\alpha \sqrt{B_{\textnormal{init}}} T}}.
	\end{eqnarray*}
	Using $B_{\textnormal{init}} \geq \min\left\{\frac{\sigma^2 L}{\widetilde{L}^2 \delta_0}, \frac{\sigma}{\alpha \sqrt{L \delta_0 T}}, \frac{\sigma^{2/3}}{\alpha^{4/3} T^{2/3} (L \delta_0)^{1/3}}, \frac{n}{\alpha^2 T}\right\},$ we obtain
	\begin{eqnarray*}
		\Exp{  \sqnorm{\nabla f(\hat x^T) }  } &\leq & \cO\rb{\fr{\wL \delta_0}{\alpha  T }  +  \rb{\fr{   L \delta_0 \sigma^{2/3} }{\al^{2/3}  T } }^{\nfr{3}{4}} +  \rb{\fr{   L \delta_0 \sigma }{\sqrt{\al}  T } }^{\nfr{2}{3}}  + \rb{\fr{  L \delta_0 \sigma^2}{ n T } }^{\nfr{1}{2}}}.
	\end{eqnarray*}
	It remains to notice that $\left\lceil\max\left\{\min\left\{\frac{\sigma^2 L}{\widetilde{L}^2 \delta_0}, \frac{\sigma}{\alpha \sqrt{L \delta_0 T}}, \frac{\sigma^{2/3}}{\alpha^{4/3} T^{2/3} (L \delta_0)^{1/3}}, \frac{n}{\alpha^2 T}\right\},  \frac{\sigma^2}{L \delta_0 n}\right\}\right\rceil \leq \left\lceil\frac{\sigma^2}{L \delta_0}\right\rceil $.
	
\end{proofof}

\subsection{Controlling  the error of momentum estimator }
\begin{lemma}\label{le:key_HB_recursion}
	Let Assumption~\ref{as:main} be satisfied, and suppose $0 < \eta \leq 1$. For every $i = 1,\ldots, n$, let the sequence $\cb{v_i^{t}}_{t\geq0}$ be updated via 
	$$v_i^{t} =  v_i^{t-1} + \eta \rb{ \nabla f_i (x^{t}, \xi_i^{t} ) -  v_i^{t-1} } , $$
	Define the sequence $v^t  \eqdef \suminn v_i^t$. Then for every $i = 1,\ldots, n$ and $t\geq0$ it holds
	\begin{eqnarray}\label{eq:wtP_rec}
		\Exp{ \sqnorm{v_i^{t} - \nabla f_i(x^{t})} }  \leq (1- \eta)   \Exp{ \sqnorm{v_i^{t-1} - \nabla f_i(x^{t-1})} }  +  \fr{3 L_i^2 }{\eta } \Exp{ \sqnorm{ x^t - x^{t-1}}  } + \eta^2 \sigma^2 ,
	\end{eqnarray}
	\begin{eqnarray}\label{eq:P_rec}
		\Exp{ \sqnorm{v^{t} - \nabla f(x^{t})} }  \leq (1- \eta)   \Exp{ \sqnorm{v^{t-1} - \nabla f(x^{t-1})} }  +  \fr{3 L^2 }{\eta } \Exp{ \sqnorm{ x^t - x^{t-1}}  } + \fr{\eta^2 \sigma^2}{n} .
	\end{eqnarray}
\end{lemma}
\begin{proof}
	By the update rule of $v_i^t$, we have
	\begin{eqnarray}
		\Exp{ \sqnorm{v_i^{t} - \nabla f_i(x^{t})} } 
		& = & \Exp{ \sqnorm{ v_i^{t-1} - \nabla f_i(x^{t}) + \eta ( \nabla f_{i}(x^{t}, \xi_i^{t})  - v_i^{t-1}  )}  } \notag \\
		& = & \Exp{ \Expu{\xi_i^{t}}{ \sqnorm{ (1-\eta) ( v_i^{t-1} - \nabla f_i(x^{t}) ) + \eta ( \nabla f_{i}(x^{t}, \xi_i^{t})  - \nabla f_i(x^{t})  )}  } } \notag \\
		& = & (1-\eta)^2 \Exp{ \sqnorm{  v_i^{t-1} - \nabla f_i(x^{t}) } } +  \eta^2 \Exp{\sqnorm{\nabla f_{i}(x^{t}, \xi_i^{t}) - \nabla f_i(x^{t}) }} \notag \\
		& \leq  & (1 - \eta)^2 \rb{ 1 + \fr{\eta}{2}} \Exp{ \sqnorm{  v_i^{t-1} - \nabla f_i(x^{t-1}) } }  \notag \\
		&& \qquad + \rb{1+\fr{2}{\eta} } \Exp{ \sqnorm{  \nabla f_i(x^{t-1}) - \nabla f_i(x^{t}) }  } + \eta^2 \sigma^2  \notag \\
		& \leq  & (1 - \eta) \Exp{ \sqnorm{  v_i^{t-1} - \nabla f_i(x^{t-1}) } }  + \fr{3 L_i^2}{\eta} \Exp{ \sqnorm{  x^{t} - x^{t+1}  } } + \eta^2 \sigma^2  \notag ,
	\end{eqnarray}
where the first inequality holds by Young's inequality, and the last step uses smoothness of $f_i(\cdot)$ (Assumption~\ref{as:main}), which concludes the proof of \eqref{eq:wtP_rec}. 

For each $t = 0, \ldots, T-1$, define a random vector $\xi^{t}  \eqdef (\xi_1^{t}, \ldots, \xi_n^{t})$ and denote by $\nabla f (x^{t}, \xi^{t})  \eqdef \suminn \nabla f_{i} (x^{t}, \xi_i^{t})$. Note that the entries of the random vector $\xi^{t}$ are independent and $\Expu{\xi^{t}}{ \nabla f (x^{t}, \xi^{t}) } = \nabla f (x^{t})$, then we have 
$$v^{t} =  v^{t-1} + \eta \rb{ \nabla f (x^{t}, \xi^{t} ) -  v^{t-1} } , $$
where $v^t  \eqdef \suminn v_i^t$ is an auxiliary sequence. Therefore, we can similarly derive
	
	\begin{eqnarray}
		\Exp{ \sqnorm{v^{t} - \nabla f(x^{t})} } 
		& = & \Exp{ \sqnorm{ v^{t-1}   - \nabla f (x^{t})  + \eta \rb{ \nabla f (x^{t}, \xi^{t}) - v^{t-1}} }  } \notag \\
		& = & \Exp{ \Expu{\xi^{t}}{ \sqnorm{ (1-\eta) (v^{t-1}   - \nabla f (x^{t}) ) + \eta \rb{ \nabla f (x^{t}, \xi^{t}) -  \nabla f (x^{t}) } }  } } \notag \\
		& = & (1-\eta)^2 \Exp{ \sqnorm{  v^{t-1} - \nabla f (x^{t}) } } +  \eta^2 \Exp{\sqnorm{\nabla f (x^{t}, \xi^{t}) - \nabla f (x^{t}) }} \notag \\
		& \leq  & (1 - \eta)^2 \rb{ 1 + \fr{\eta}{2}} \Exp{ \sqnorm{  v^{t-1} - \nabla f(x^{t-1}) } }  \notag \\
		&& \qquad + \rb{1+\fr{2}{\eta} } \Exp{ \sqnorm{  \nabla f (x^{t-1}) - \nabla f (x^{t}) }  } + \frac{\eta^2 \sigma^2}{n}  \notag \\
		& \leq  & (1 - \eta) \Exp{ \sqnorm{  v^{t-1} - \nabla f (x^{t-1}) } }  + \fr{3 L^2}{\eta} \Exp{ \sqnorm{  x^{t} - x^{t-1}  } } + \frac{\eta^2 \sigma^2}{n}  \notag ,
	\end{eqnarray}
	where the last step uses smoothness of $f (\cdot)$ (Assumption~\ref{as:main}), which concludes the proof of \eqref{eq:P_rec}.
\end{proof}

\subsection{Controlling  the error  of contractive compression and momentum estimator}

\begin{lemma}\label{le:EF21} 
	Let Assumption~\ref{as:main} be satisfied, and suppose $\cC$ is a contractive compressor with $\al \leq \fr{1}{2}$. For every $i = 1,\ldots, n$, let the sequences $\cb{v_i^{t}}_{t\geq0}$ and $\cb{g_i^{t}}_{t\geq0}$ be updated via 
	$$v_i^{t} =  v_i^{t-1} + \eta \rb{ \nabla f_i (x^{t}, \xi_i^{t} ) -  v_i^{t-1} } , $$
	$$g_i^{t} =  g_i^{t-1} +  \cC\rb{ v_i^t - g_i^{t-1}  }, $$
	Then for every $i = 1,\ldots, n$ and $t\geq0$ it holds
	\begin{eqnarray}\label{eq:rec_1_avg} 
		\Exp{ \sqnorm{g_i^t - v_i^{t}} }  &\leq& \rb{ 1-\fr{\al}{2} } \Exp{ \sqnorm{g_i^{t-1} - v_i^{t-1}} }  + \frac{4 \eta^2 }{\alpha}  \Exp{\sqnorm{v_i^{t-1} - \nabla f_i(x^{t-1})} }  \notag \\
		&& \qquad +  \frac{4 L_i^2 \eta^2 }{\alpha}   \Exp{\sqnorm{x^{t} - x^{t-1}} } +  \eta^2 \sigma^2 .
	\end{eqnarray}
\end{lemma}
\begin{proof}
	By the update rules of $g_i^t$ and $v_i^t$, we derive
	\begin{eqnarray*}
		\Exp{\sqnorm{g_i^{t} - v_i^{t}} } & =& \Exp{ \sqnorm{ g_i^{t-1}  - v_i^{t}  +  \cC(  v_i^{t} - g_i^{t-1} )  } }  \\
		& =&  \Exp{ \Expu{\cC}{\sqnorm{ \cC(  v_i^{t} - g_i^{t-1})   - ( v_i^{t} - g_i^{t-1})    } } }  \\
		&\overset{(i)}{ \leq} & (1-\al) \Exp{ \sqnorm{  v_i^{t} - g_i^{t-1}  } }  \\
		&\overset{(ii)}{ = } & (1-\al) \Exp{ \sqnorm{  v_i^{t-1} - g_i^{t-1} + \eta ( \nabla f_i(x^{t}, \xi_i^{t}) - v_i^{t-1} )  } }  \\
		&{ = } & (1-\al) \Exp{ \Expu{\xi_i^{t}}{ \sqnorm{  v_i^{t-1} - g_i^{t-1} + \eta ( \nabla f_i(x^{t}, \xi_i^{t}) - v_i^{t-1} )  } }  } \\
		& = & (1-\al) \Exp{ \sqnorm{ v_i^{t-1} - g_i^{t-1} + \eta ( \nabla f_i(x^{t}) - v_i^{t-1} ) }  } \\
		&& \qquad + (1 - \alpha) \eta^2 \Exp{\sqnorm{\nabla f_i(x^{t}, \xi_i^{t}) - \nabla f_i(x^{t})} } \\
		&\overset{(iii)}{ \leq} & \left(1 - \alpha\right) (1 + \rho ) \Exp{ \sqnorm{v_i^{t-1} -  g_i^{t-1}} } + \left(1 - \alpha\right) (1 + \rho^{-1} ) \eta^2 \Exp{\sqnorm{v_i^{t-1} - \nabla f_i(x^{t})} } \notag \\
		&& \qquad  + (1 - \alpha) \eta^2 \sigma^2\\
		&\overset{(iv)}{ = } & (1-\theta)\Exp{ \sqnorm{ g_i^{t-1} - v_i^{t-1} } } + \beta \eta^2 \Exp{\sqnorm{v_i^{t-1} - \nabla f_i(x^{t})} }  + (1 - \alpha) \eta^2 \sigma^2\\
		&\overset{(v)}{ \leq} & (1-\theta)\Exp{ \sqnorm{ g_i^{t-1} - v_i^{t-1} } } + 2 \beta \eta^2 \Exp{\sqnorm{v_i^{t-1} - \nabla f_i(x^{t-1})} }  \notag \\
		&& \qquad + 2 \beta \eta^2 \Exp{\sqnorm{\nabla f_i(x^{t}) - \nabla f_i(x^{t-1})} } +  \eta^2 \sigma^2\\
		& \leq  & (1-\theta)\Exp{ \sqnorm{ g_i^{t-1} - v_i^{t-1} } } + 2 \beta \eta^2 \Exp{\sqnorm{v_i^{t-1} - \nabla f_i(x^{t-1})} }  \notag \\
		&& \qquad + 2 \beta L_i^2 \eta^2  \Exp{\sqnorm{x^{t} - x^{t-1}} } +  \eta^2 \sigma^2 ,
	\end{eqnarray*}
where $(i)$ is due to the definition of a contractive compressor (Definition~\ref{def:contractive_compressor}), $(ii)$ follows by the update rule of $v_i^t$, $(iii)$ and $(v)$ hold by Young's inequality for any $\rho >0$. In $(iv)$, we introduced the notation $\theta  \eqdef 1 - (1-\alpha)(1 + \rho)$, and $\beta  \eqdef (1-\alpha) (1 + \rho^{-1})$. The last step follows by smoothness of $f_i(\cdot)$ (Assumption~\ref{as:main}). The proof is complete by the choice $\rho = \alpha / 2$, which guarantees $1-\theta \leq 1-\alpha/2$, and $2 \beta \leq 4/\alpha$ . 
\end{proof}

\clearpage
\section{Further Improvement Using Double Momentum (Proof of Corollary~\ref{cor:EF21-DM_conv})}\label{sec:appendix_DM}


\begin{algorithm}[H]
	\centering
	\caption{\algname{EF21-SGD2M} (Error Feedback  2021 Enhanced with \textit{Double Momentum})}\label{alg:EF21-DM}
	\begin{algorithmic}[1]
		\State \textbf{Input:} starting point $x^{0}$, step-size $\gamma>0$, parameter $\eta \in (0, 1]$, initial batch size $B_{\textnormal{init}}$
		\State Initialize $u_i^{0} = v_i^{0} = g_i^{0} = \frac{1}{B_{\textnormal{init}}} \sum_{j=1}^{B_{\textnormal{init}}} \nabla f_{i}(x^0, \xi_{i, j}^{0})$ for  $ i = 1, \ldots,  n$; $g^{0} = \frac{1}{n} \sum_{i=1}^n  g_i^{0}$
		\For{$t=0,1, 2, \dots , T-1 $}
		\State Master computes $x^{t+1} = x^t - \gamma g^t$ and broadcasts $x^{t+1}$ to all nodes 
		\For{{\bf all nodes $i =1,\dots, n$ in parallel}}
		\State \textcolor{mygreen}{ Compute the first momentum estimator $v_i^{t+1} = (1-\eta) v_i^{t}+ \eta  \nabla f_{i}(x^{t+1}, \xi_{i}^{t+1})  $ }
		\State \textcolor{mygreen}{ Compute the second momentum estimator $u_i^{t+1} = (1-\eta) u_i^{t}+ \eta  v_i^{t+1}  $ }
		\State Compress $c_i^{t+1} =  \cC(  u_i^{t+1} - g_i^{t}   )$ and send $c_i^{t+1} $ to the master
		\State Update local state $g_i^{t+1} = g_i^{t} +  c_i^{t+1}$
		\EndFor
		\State Master computes $g^{t+1} = \frac{1}{n} \sum_{i=1}^n  g_i^{t+1}$ via  $g^{t+1} = g^{t} + \frac{1}{n} \sum_{i=1}^n   c_i^{t+1} $
		\EndFor
	\end{algorithmic}
\end{algorithm}

In this section, we state the detailed version of Corollary~\ref{cor:EF21-DM_conv} in Theorem~\ref{thm:EF21-DM-SGD}, followed by its formal proof. Notice that the key reason for the sample complexity improvement of the double momentum variant compared to \algname{EF21-SGDM} (Theorem~\ref{thm:main-distrib}) is that in \eqref{eq:EF21-DM-SGD:general}, one of the terms has better dependence on $\eta$  compared to \eqref{eq:main-distrib:general} in Theorem~\ref{thm:main-distrib}, i.e., $\eta^4 \sigma^2 / \alpha$ instead of $\eta^2 \sigma^2 / \alpha$. As a result, this term is dominated by other terms and vanishes in Corollary~\ref{cor:EF21-DM_conv}.

\begin{theorem}\label{thm:EF21-DM-SGD}
		Let Assumptions~\ref{as:main} and \ref{as:BV} hold. Let $\hat x^T$ be sampled uniformly at random from the iterates of the method. Let Algorithm~\ref{alg:EF21-DM} run with a contractive compressor. For all $\eta \in (0, 1]$ and $B_{\textnormal{init}} \geq 1$, with $\gamma \leq \min\cb{ \fr{\al}{60 \wL}, \fr{\eta}{16 L}  },$ we have
	\begin{align}
		\label{eq:EF21-DM-SGD:general}
		\Exp{  \sqnorm{\nabla f(\hat x^T) }  } \leq \cO\rb{\fr{\Psi_0}{\gamma T} + \fr{\eta^3 \sigma^2 }{\al^2}  +  \fr{\eta^4 \sigma^2 }{\al} + \frac{\eta \sigma^2}{n}},
	\end{align}
	where $\Psi_0 \eqdef \delta_0 + \frac{\gamma}{\eta}\Exp{\norm{v^0 - \nabla f(x^0)}^2} + \frac{\gamma \eta^4}{\alpha^2} \frac{1}{n} \sum_{i=1}^n \Exp{\norm{v^0_i - \nabla f_i(x^0)}^2}$. Setting initial batch size $B_{\textnormal{init}} =  \left\lceil\frac{\sigma^2}{L \delta_0}\right\rceil $, step-size and momentum parameters
	\begin{equation} \label{eqEF21-DM-SGD-stepsize}
		\gamma = \min\cb{ \fr{\al}{60 \wL}, \fr{\eta}{16 L}  } , \qquad  \eta = \min\cb{ 1 ,  \rb{ \fr{ L \delta_0 \al^2 }{\sigma^2 T}}^{\nfr{1}{4} }, \rb{ \fr{ L \delta_0 n  }{\sigma^2 T}}^{\nfr{1}{2} } , \frac{\alpha \sqrt{L \delta_0 B_{\textnormal{init}}}}{\sigma} } ,
	\end{equation}
	we get
	\begin{eqnarray*}
		\fr{1}{T}\sum_{t=0}^{T-1} \Exp{  \sqnorm{\nabla f(x^t) }  } &\leq&  
		\cO\rb{ \fr{\wL \delta_0}{\alpha  T }  +     \rb{\fr{   L \delta_0 \sigma^{2/3} }{\al^{2/3}  T } }^{\nfr{3}{4}}  + \rb{\fr{   L \delta_0 \sigma^{2} }{ n  T } }^{\nfr{1}{2}}  }.
	\end{eqnarray*}
	
\end{theorem}

\begin{proof}
	In order to control the error between $g^t$ and $\nabla f (x^t)$, we decompose it into three terms
	\begin{eqnarray*}
	\sqnorm{g^t - \nabla f(x^t)} &\leq& 3 \sqnorm{g^t - u^t} + 3 \sqnorm{u^t - v^t}  + 3 \sqnorm{v^t - \nabla f(x^t)} \\
	&\leq& 3 \suminn \sqnorm{g_i^t - u_i^t} + 3 \sqnorm{u^t - v^t}  + 3 \sqnorm{v^t - \nabla f(x^t)}, 
	\end{eqnarray*}
	where we define the sequences $v^t  \eqdef \suminn v_i^t$ and $u^t  \eqdef \suminn u_i^t$. In the following, we develop a recursion for each term above separately.
	
	\textbf{Part I. Controlling  the error  of momentum estimator for each $v_i^t$ and on average for $v^t$.}
	Denote $P_t  \eqdef   \Exp{ \sqnorm{ v^t - \nabla f(x^{t})} }$, $\wt P_t : = \suminn \Exp{ \sqnorm{v_i^{t} - \nabla f_i(x^{t})} } $, $R_t  \eqdef \Exp{\sqnorm{x^{t} - x^{t+1} } } $. Similarly to Part I  of the proof of Theorem~\ref{thm:main-distrib}, we have 
	\begin{eqnarray}\label{eq:wt_mom_est_avg_dist_DM}
		\frac{1}{T} \sum_{t=0}^{T-1} \wt P_t  \leq  \fr{3 \wt L^2 }{\eta^2 } \frac{1}{T} \sum_{t=0}^{T-1}  R_t +  \eta \sigma^2  + \frac{1}{\eta T} \wt P_0 ,
	\end{eqnarray}
	\begin{eqnarray}\label{eq:mom_est_avg_dist_DM}
	\frac{1}{T} \sum_{t=0}^{T-1} P_t  \leq  \fr{3 L^2 }{\eta^2 } \frac{1}{T} \sum_{t=0}^{T-1}  R_t +  \frac{\eta \sigma^2 }{n}  + \frac{1}{\eta T}  P_0 .
\end{eqnarray}

	\textbf{Part II (a). Controlling  the error of the second momentum estimator for each $u_i^t$.} Recall that by Lemma~\ref{le:key_HB_recursion_DM}-\eqref{eq:lem_wtQ_rec_DM}, we have for each $i = 1, \ldots, n$, and any $0 < \eta \leq 1$ and $t\geq 0$
	\begin{eqnarray}\label{eq:wtQi_rec_DM}
		\Exp{ \sqnorm{u_i^{t+1} - v_i^{t+1}  } }  &\leq& (1- \eta)   \Exp{ \sqnorm{u_i^{t} - v_i^{t} } }  + 6 \eta \Exp{\sqnorm{v_i^{t} - \nabla f_i(x^{t}) } } \notag \\
		&&	\qquad +  6 L_i^2 \eta  \Exp{ \sqnorm{ x^{t+1} - x^t }  } + \eta^2 \sigma^2 ,
	\end{eqnarray}
	Averaging inequalities \eqref{eq:wtQi_rec_DM} over $i=1,\ldots,n$ and denoting $\wt Q_t : = \suminn \Exp{ \sqnorm{ u_i^{t} - v_i^{t} } } $, we have 
	\begin{eqnarray*}
		\wt Q_{t+1}  &\leq& (1- \eta)   \wt Q_t  + 6 \eta \wt P_t  +  6 \wt L^2 \eta  R_t + \eta^2 \sigma^2 .
	\end{eqnarray*}
	Summing up the above inequalities for $t=0, \ldots, T-1$, we derive 
	\begin{eqnarray}\label{eq:wt_doub_mom_est_avg_dist_DM}
		\frac{1}{T} \sum_{t=0}^{T-1} \wt Q_t  &\leq& \frac{6}{T} \sum_{t=0}^{T-1} \wt P_t + 6 \wt L^2  \frac{1}{T} \sum_{t=0}^{T-1}  R_t +  \eta \sigma^2 + \frac{1}{\eta T} \wt Q_0 \notag \\
		&\leq& \rb{\fr{6 \cdot 3 \wt L^2 }{\eta^2 }   + 6 \wt L^2 } \frac{1}{T}  \sum_{t=0}^{T-1}  R_t +  7 \eta \sigma^2  + \frac{1}{\eta T} \wt Q_0 + \frac{6}{\eta T} \wt P_0 \notag \\
		&\leq& \fr{19 \wt L^2 }{\eta^2 }   \frac{1}{T}  \sum_{t=0}^{T-1}  R_t +  7 \eta \sigma^2 + \frac{6}{\eta T} \wt P_0 ,
	\end{eqnarray}
	where we used \eqref{eq:wt_mom_est_avg_dist_DM}, the bound $\eta \leq 1$, and $u_i^0 = v_i^0$ for $i = 1, \ldots, n$.  
	
		\textbf{Part II (b). Controlling  the error of the second momentum estimator  $u^t$ (on average).} Similarly by Lemma~\ref{le:key_HB_recursion_DM}-\eqref{eq:lem_Q_rec_DM}, we have for any $0 < \eta \leq 1$ and $t\geq 0$
	\begin{eqnarray*}
		\Exp{ \sqnorm{u^{t+1} - v^{t+1}  } }  &\leq& (1- \eta)   \Exp{ \sqnorm{u^{t} - v^{t} } }  + 6 \eta \Exp{\sqnorm{v^{t} - \nabla f(x^{t}) } } \notag \\
	&&	\qquad +  6 L^2 \eta  \Exp{ \sqnorm{ x^{t+1} - x^t }  } + \frac{\eta^2 \sigma^2}{n} ,
	\end{eqnarray*}
	Summing up the above inequalities for $t=0, \ldots, T-1$, and denoting  $ Q_t : =  \Exp{ \sqnorm{ u^{t} - v^{t} } } $, we derive
	\begin{eqnarray}\label{eq:doub_mom_est_avg_dist_DM}
		\frac{1}{T} \sum_{t=0}^{T-1} Q_t  &\leq& \frac{6}{T} \sum_{t=0}^{T-1} P_t + 6  L^2  \frac{1}{T} \sum_{t=0}^{T-1}  R_t +  \eta \sigma^2  + \frac{1}{\eta T} Q_0  \notag \\
		&\leq& \rb{\fr{6 \cdot 3  L^2 }{\eta^2 }   + 6  L^2 } \frac{1}{T}  \sum_{t=0}^{T-1}  R_t +   \frac{7\eta \sigma^2 }{n} + \frac{1}{\eta T}  Q_0 + \frac{6}{\eta T}  P_0 \notag \\
		&\leq& \fr{19  L^2 }{\eta^2 }   \frac{1}{T}  \sum_{t=0}^{T-1}  R_t +   \frac{7\eta \sigma^2 }{n}  + \frac{6}{\eta T}  P_0 ,
	\end{eqnarray}
	where we used \eqref{eq:mom_est_avg_dist_DM}, the bound $\eta \leq 1$, and $u^0 = v^0$.

	\textbf{Part III. Controlling  the error of contractive compressor and the double momentum estimator.} By Lemma~\ref{le:EF21-DM} we have for each $i = 1, \ldots, n$, and any $0 < \eta \leq 1$ and $t\geq 0$
	\begin{eqnarray}\label{eq:wtVi_rec_DM}
		\Exp{ \sqnorm{g_i^{t+1} - u_i^{t+1}} }  &\leq& \rb{ 1-\fr{\al}{2} } \Exp{ \sqnorm{g_i^{t} - u_i^{t}} }  + \frac{6 \eta^2 }{\alpha}  \Exp{\sqnorm{u_i^{t} - v_i^{t} } }   \\
	&& \quad + \frac{6 \eta^4 }{\alpha}  \Exp{\sqnorm{v_i^{t} - \nabla f_i(x^{t}, \xi_i^{t})} }  +  \frac{6 L_i^2 \eta^4 }{\alpha}   \Exp{\sqnorm{x^{t} - x^{t+1}} } +  \eta^4 \sigma^2 . \notag
	\end{eqnarray}
	
	Averaging inequalities \eqref{eq:wtVi_rec_DM} over $i=1,\ldots,n$,  denoting $ \wt V_t  \eqdef \suminn  \Exp{ \sqnorm{g_i^t - u_i^t} }$, and summing up the resulting inequality for $t=0, \ldots, T-1$, we obtain 
	\begin{eqnarray}\label{eq:ef21_mom_est_avg_dist_DM}
		\frac{1}{T} \sum_{t=0}^{T-1} \wt V_t  &\leq&   \fr{12 \eta^2 }{\al^2} \frac{1}{T} \sum_{t=0}^{T-1}  \wt Q_t  + \fr{12 \eta^4 }{\al^2} \frac{1}{T} \sum_{t=0}^{T-1}  \wt P_t +  \fr{12  \wt L^2 \eta^4 }{\al^2}  \frac{1}{T} \sum_{t=0}^{T-1} R_t + \fr{2\eta^4 \sigma^2 }{\al} \notag \\
		&\leq&   \fr{12 \eta^2 }{\al^2} \rb{ \fr{19 \wt L^2 }{\eta^2 }   \frac{1}{T}  \sum_{t=0}^{T-1}  R_t +  7 \eta \sigma^2 }  + \fr{12 \eta^4 }{\al^2} \rb{ \fr{3 \wt L^2 }{\eta^2 } \frac{1}{T} \sum_{t=0}^{T-1}  R_t +  \eta \sigma^2 } \notag \\
		&& \qquad +  \fr{12  \wt L^2 \eta^4 }{\al^2}  \frac{1}{T} \sum_{t=0}^{T-1} R_t + \fr{2\eta^4 \sigma^2 }{\al} +  \frac{12 \eta^4}{\alpha^2  T} \wt P_0 \notag \\
		&\leq&    \fr{12 \cdot 19 \wt L^2 }{\alpha^2 }   \frac{1}{T}  \sum_{t=0}^{T-1}  R_t  +     \fr{12 \cdot 7 \eta^3 \sigma^2 }{\al^2}  +   \fr{12 \cdot 3 \wt L^2  \eta^2 }{\alpha^2} \frac{1}{T} \sum_{t=0}^{T-1}  R_t +  \frac{12 \eta^5 \sigma^2}{\alpha^2}  \notag \\
		&& \qquad +  \fr{12  \wt L^2 \eta^4 }{\al^2}  \frac{1}{T} \sum_{t=0}^{T-1} R_t + \fr{2\eta^4 \sigma^2 }{\al} +  \frac{12 \eta^4}{\alpha^2  T} \wt P_0\notag \\
		&\leq&    \fr{276 \wt L^2 }{\alpha^2 }   \frac{1}{T}  \sum_{t=0}^{T-1}  R_t  +     \fr{84 \eta^3 \sigma^2 }{\al^2}  +  \frac{12 \eta^5 \sigma^2}{\alpha^2}   + \fr{2\eta^4 \sigma^2 }{\al} +  \frac{12 \eta^4}{\alpha^2  T} \wt P_0 \notag \\
	\end{eqnarray}
	
	\textbf{Part IV. Combining steps I, II and III with descent lemma.}	
	By smoothness (Assumption~\ref{as:main}) of $f(\cdot)$ it follows from Lemma~\ref{le:descent} that for any $\gamma \leq 1/(2 L) $ we have 
	\begin{eqnarray}\label{eq:descent_ef21-DM_dist}
		f(x^{t+1}) &\leq& f(x^{t}) - \fr{\g}{2} \sqnorm{\nabla f(x^t)}  - \frac{1}{4\gamma} \sqnorm{x^{t+1} - x^t } + \frac{\gamma}{2} \sqnorm{g^t - \nabla f(x^t) } \\
		&\leq& f(x^{t}) - \fr{\g}{2} \sqnorm{\nabla f(x^t)}  -  \frac{1}{4\gamma} \sqnorm{x^{t+1} - x^t } \notag \\
		&& \qquad + \frac{3 \gamma}{2} \suminn\sqnorm{g_i^t - u_i^t } + \frac{3 \gamma}{2} \sqnorm{u^t - v^t  }  + \frac{3 \gamma}{2} \sqnorm{v^t - \nabla f(x^t) }  . \notag
	\end{eqnarray}
	Subtracting $f^*$ from both sides of \eqref{eq:descent_ef21-DM_dist}, taking expectation and defining $\delta_t  \eqdef \Exp{f(x^t) - f^*}$, we derive
	\begin{eqnarray}
		\Exp{\sqnorm{\nabla f(\hat x^T)}} &=& \frac{1}{T} \sum_{t=0}^{T-1} \Exp{\sqnorm{\nabla f(x^t)}} \notag \\
		&\leq&  \fr{2 \delta_0 }{\gamma T} + 3 \frac{1}{T}\sum_{t=0}^{T-1} \wt V_t + 3 \frac{1}{T} \sum_{t=0}^{T-1} Q_t +  3 \frac{1}{T} \sum_{t=0}^{T-1} P_t    - \frac{1 }{2\gamma^2} \frac{1}{T} \sum_{t=0}^{T-1} R_t \notag  \\
		&\overset{(i)}{\leq}&  \fr{2 \delta_0 }{\gamma T} + \fr{3 \cdot 276 \wt L^2 }{\alpha^2 }  \frac{1}{T}  \sum_{t=0}^{T-1}  R_t  +    \fr{3\cdot 84 \eta^3 \sigma^2 }{\al^2}  +  \frac{3\cdot 12 \eta^5 \sigma^2}{\alpha^2}   + \fr{3\cdot 2\eta^4 \sigma^2 }{\al}  \notag  \\
		&& \qquad + \fr{3\cdot 19  L^2 }{\eta^2 }   \frac{1}{T}  \sum_{t=0}^{T-1}  R_t +   \frac{3\cdot 7\eta \sigma^2 }{n} \notag  \\ 
		&& \qquad + \fr{3 L^2 }{\eta^2 } \frac{1}{T} \sum_{t=0}^{T-1}  R_t +  \frac{\eta \sigma^2 }{n} - \frac{1 }{2\gamma^2} \frac{1}{T} \sum_{t=0}^{T-1} R_t  \notag \\
		&& \qquad +  \frac{36  \eta^4}{\alpha^2  T} \wt P_0 + \frac{18}{\eta T}  P_0 + \frac{3}{\eta T}  P_0 \notag \\
		& = &  \fr{2 \delta_0 }{\gamma T}   +    \fr{3\cdot 84 \eta^3 \sigma^2 }{\al^2}  +  \frac{3\cdot 12 \eta^5 \sigma^2}{\alpha^2}   + \fr{3\cdot 2\eta^4 \sigma^2 }{\al} +   \frac{22\eta \sigma^2 }{n}   \notag  \\
		&& \qquad + \rb{ \fr{60  L^2 }{\eta^2 }   + \fr{3 \cdot 276 \wt L^2 }{\alpha^2 }  - \frac{1 }{2\gamma^2} }  \frac{1}{T} \sum_{t=0}^{T-1} R_t  \notag \\
		&& \qquad +  \frac{36  \eta^4}{\alpha^2  T} \wt P_0 + \frac{21}{\eta T}  P_0 \notag \\
		& = &  \fr{2 \delta_0 }{\gamma T}   +    \fr{288 \eta^3 \sigma^2 }{\al^2}     + \fr{ 6  \eta^4 \sigma^2 }{\al} +   \frac{22\eta \sigma^2 }{n}   +  \frac{36  \eta^4}{\alpha^2  T} \wt P_0 + \frac{21}{\eta T}  P_0 \notag . \label{eq:conv_ef21_sgdm_DM_eta}
	\end{eqnarray}
	where $(i)$ holds due to \eqref{eq:mom_est_avg_dist_DM}, \eqref{eq:ef21_mom_est_avg_dist_DM} and \eqref{eq:doub_mom_est_avg_dist_DM}, the last two steps hold because of the assumption on the step-size, and $\eta \leq 1$, which completes the proof of the first part of Theorem.  
	
	Notice that it suffices to take the same initial batch-size as in the proof of the Theorem~\ref{thm:main-distrib} in order to "remove" $\wt P_0$ and $P_0$ terms, since the power of $\eta$ in front of $\wt P_0$ is larger here compared to the proof of Theorem~\ref{thm:main-distrib}. The choice of the momentum parameter such that $\eta \leq \rb{ \fr{ L \delta_0 \al^2 }{\sigma^2 T}}^{\nfr{1}{4} }$, $\eta \leq \rb{ \fr{ L \delta_0 n  }{\sigma^2 T}}^{\nfr{1}{2} }$ ensures that $\fr{ \eta^3 \sigma^2 }{\al^2}   \leq \fr{L \delta_0 }{\eta T}  $,  and $\frac{ \eta \sigma^2}{n} \leq \fr{L \delta_0 }{\eta T} $. Therefore, we can guarantee that the choice $\eta =  \min\cb{ \frac{\alpha \sqrt{L \delta_0 B_{\textnormal{init}}}}{\sigma},  \rb{ \fr{ L \delta_0 \al^2 }{\sigma^2 T}}^{\nfr{1}{4} }, \rb{ \fr{ L \delta_0 n  }{\sigma^2 T}}^{\nfr{1}{2} } }$ ensures that 
	\begin{eqnarray*}
		\Exp{  \sqnorm{\nabla f(\hat x^T) }  } &\leq & \cO\rb{ \fr{\wL \delta_0}{\alpha  T }  +  \rb{\fr{   L \delta_0 \sigma^{2/3} }{\al^{2/3}  T } }^{\nfr{3}{4}} + \rb{\fr{  L \delta_0 \sigma^2}{ n T } }^{\nfr{1}{2}} }  .
	\end{eqnarray*}

\end{proof}

\subsection{Controlling  the error of second momentum estimator }
\begin{lemma}\label{le:key_HB_recursion_DM}
	Let Assumption~\ref{as:main} be satisfied, and suppose $0 < \eta \leq 1$. For every $i = 1,\ldots, n$, let the sequences $\cb{v_i^{t}}_{t\geq0}$ and  $\cb{u_i^{t}}_{t\geq0}$ be updated via 
	$$v_i^{t} =  v_i^{t-1} + \eta \rb{ \nabla f_i (x^{t}, \xi_i^{t} ) -  v_i^{t-1} } , $$
	$$u_i^{t} =  u_i^{t-1} + \eta \rb{ v_i^t -  u_i^{t-1} } . $$
	Define the sequences $v^t  \eqdef \suminn v_i^t$ and $u^t  \eqdef \suminn u_i^t$. Then for every $i = 1,\ldots, n$ and $t\geq0$ it holds
	\begin{eqnarray}\label{eq:lem_wtQ_rec_DM}
		\Exp{ \sqnorm{u_i^{t} - v_i^{t}  } }  \leq (1- \eta)   \Exp{ \sqnorm{u_i^{t-1} - v_i^{t-1} } }  + 6 \eta \Exp{\sqnorm{v_i^{t-1} - \nabla f_i(x^{t-1}) } } \notag \\
		\qquad +  6 L_i^2 \eta  \Exp{ \sqnorm{ x^t - x^{t-1}}  } + \eta^2 \sigma^2 ,
	\end{eqnarray}
	\begin{eqnarray}\label{eq:lem_Q_rec_DM}
		\Exp{ \sqnorm{u^{t} - v^{t}  } }  \leq (1- \eta)   \Exp{ \sqnorm{u^{t-1} - v^{t-1} } }  + 6 \eta \Exp{\sqnorm{v^{t-1} - \nabla f (x^{t-1}) } } \notag \\
		\qquad +  6 L^2 \eta  \Exp{ \sqnorm{ x^t - x^{t-1}}  } + \fr{\eta^2 \sigma^2}{n} .
	\end{eqnarray}
\end{lemma}
\begin{proof}
	By the update rule of $v_i^t$, we have
	\begin{eqnarray}
		\Exp{ \sqnorm{u_i^{t} - v_i^{t}  } } 
		& = & \Exp{ \sqnorm{ u_i^{t-1} - v_i^{t} + \eta ( v_i^{t} - u_i^{t-1}  )}  } \notag \\
		& = & (1-\eta)^2\Exp{ \sqnorm{  v_i^{t} - u_i^{t-1} }  } \notag \\
		& = & (1-\eta)^2\Exp{ \sqnorm{   (1-\eta) v_i^{t-1}  + \eta  \nabla f_{i}(x^{t}, \xi_i^{t}) - u_i^{t-1}  }  } \notag \\
		& = & (1-\eta)^2\Exp{ \sqnorm{   ( u_i^{t-1} - v_i^{t-1} ) + \eta (v_i^{t-1} - \nabla f_{i}(x^{t}, \xi_i^{t}) ) }  } \notag \\
		& = & (1-\eta)^2\Exp{ \sqnorm{   ( u_i^{t-1} - v_i^{t-1} ) + \eta (v_i^{t-1} - \nabla f_{i}(x^{t}) ) + \eta (\nabla f_{i}(x^{t}) - \nabla f_{i}(x^{t}, \xi_i^{t})) }  } \notag \\
		& = & (1-\eta)^2\Exp{ \Expu{\xi_i^{t}}{  \sqnorm{   ( u_i^{t-1} - v_i^{t-1} ) + \eta (v_i^{t-1} - \nabla f_{i}(x^{t}) ) + \eta (\nabla f_{i}(x^{t}) - \nabla f_{i}(x^{t}, \xi_i^{t})) }  } } \notag \\
		& = & (1-\eta)^2 \left(\Exp{ \sqnorm{  u_i^{t-1} - v_i^{t-1}  + \eta (v_i^{t-1} - \nabla f_{i}(x^{t}) )  } } +  \eta^2 \Exp{\sqnorm{\nabla f_{i}(x^{t}, \xi_i^{t}) - \nabla f_i(x^{t}) }} \right)  \notag \\
		& \leq & (1-\eta)^2 \Exp{ \sqnorm{  u_i^{t-1} - v_i^{t-1}  + \eta (v_i^{t-1} - \nabla f_{i}(x^{t}) )  } } +  \eta^2 \sigma^2 \notag \\
		& \leq  & (1 - \eta)^2 \rb{ 1 + \fr{\eta}{2}} \Exp{ \sqnorm{  u_i^{t-1} - v_i^{t-1}  } }  \notag \\
		&& \qquad + \rb{1+\fr{2}{\eta} } \eta^2 \Exp{ \sqnorm{  v_i^{t-1} - \nabla f_i(x^{t}) }  } + \eta^2 \sigma^2  \notag \\
		& \leq  & (1 - \eta) \Exp{ \sqnorm{  u_i^{t-1} - v_i^{t-1} } }  + 3  \eta \Exp{ \sqnorm{  v_i^{t-1} - \nabla f_i(x^{t}) }  }  + \eta^2 \sigma^2  \notag  \\
		& \leq  & (1 - \eta) \Exp{ \sqnorm{  u_i^{t-1} - v_i^{t-1} } }  + 6 \eta \Exp{ \sqnorm{  v_i^{t-1} - \nabla f_i(x^{t-1}) }  } \notag \\
		&& \qquad  + 6 \eta \Exp{ \sqnorm{  \nabla f_i(x^{t})  - \nabla f_i(x^{t-1}) }  } + \eta^2 \sigma^2  \notag \\
		& \leq  & (1 - \eta) \Exp{ \sqnorm{  u_i^{t-1} - v_i^{t-1} } }  + 6 \eta \Exp{ \sqnorm{  v_i^{t-1} - \nabla f_i(x^{t-1}) }  } \notag \\
		&& \qquad  + 6 L_i^2 \eta \Exp{ \sqnorm{  x^t - x^{t-1} }  } + \eta^2 \sigma^2  \notag ,
	\end{eqnarray}
	where the first inequality holds Assumption~\ref{as:BV}, the second inequality holds by Young's inequality, and the last step uses smoothness of $f_i(\cdot)$ (Assumption~\ref{as:main}), which concludes the proof of \eqref{eq:lem_wtQ_rec_DM}. 
	
	For each $t = 0, \ldots, T-1$, define a random vector $\xi^{t}  \eqdef (\xi_1^{t}, \ldots, \xi_n^{t})$ and denote by $\nabla f (x^{t}, \xi^{t+1})  \eqdef \suminn \nabla f_{i} (x^{t}, \xi_i^{t})$. Note that the entries of the random vector $\xi^{t}$ are independent and $\Expu{\xi^{t}}{ \nabla f (x^{t}, \xi^{t}) } = \nabla f (x^{t})$, then we have 
	$$v^{t} =  v^{t-1} + \eta \rb{ \nabla f (x^{t}, \xi^{t} ) -  v^{t-1} } , $$
	$$u_i^{t} =  u_i^{t-1} + \eta \rb{ v_i^t -  u_i^{t-1} } , $$
	where $v^t  \eqdef \suminn v_i^t$, $u^t  \eqdef \suminn u_i^t$ are auxiliary sequences. Therefore, we can similarly derive
	
	\begin{eqnarray}
		\Exp{ \sqnorm{u^{t} - v^{t}  } } 
		& = & \Exp{ \sqnorm{ u^{t-1} - v^{t} + \eta ( v^{t} - u^{t-1}  )}  } \notag \\
		& = & (1-\eta)^2\Exp{ \sqnorm{  v^{t} - u^{t-1} }  } \notag \\
		& = & (1-\eta)^2\Exp{ \sqnorm{   (1-\eta) v^{t-1}  + \eta  \nabla f(x^{t}, \xi^{t}) - u^{t-1}  }  } \notag \\
		& = & (1-\eta)^2\Exp{ \sqnorm{   ( u^{t-1} - v^{t-1} ) + \eta (v^{t-1} - \nabla f(x^{t}, \xi^{t}) ) }  } \notag \\
		& = & (1-\eta)^2\Exp{ \sqnorm{   ( u^{t-1} - v^{t-1} ) + \eta (v^{t-1} - \nabla f(x^{t}) ) + \eta (\nabla f(x^{t}) - \nabla f(x^{t}, \xi^{t})) }  } \notag \\
		& = & (1-\eta)^2\Exp{ \Expu{\xi^{t}}{  \sqnorm{   ( u^{t-1} - v^{t-1} ) + \eta (v^{t-1} - \nabla f(x^{t}) ) + \eta (\nabla f(x^{t}) - \nabla f(x^{t}, \xi^{t})) }  } } \notag \\
		& = & (1-\eta)^2 \left(\Exp{ \sqnorm{  u^{t-1} - v^{t-1}  + \eta (v^{t-1} - \nabla f(x^{t}) )  } } +  \eta^2 \Exp{\sqnorm{\nabla f(x^{t}, \xi^{t}) - \nabla f(x^{t}) }} \right)  \notag \\
		& \leq & (1-\eta)^2 \Exp{ \sqnorm{  u^{t-1} - v^{t-1}  + \eta (v^{t-1} - \nabla f(x^{t}) )  } } +  \frac{\eta^2 \sigma^2}{n}  \notag \\
		& \leq  & (1 - \eta)^2 \rb{ 1 + \fr{\eta}{2}} \Exp{ \sqnorm{  u^{t-1} - v^{t-1}  } }  \notag \\
		&& \qquad + \rb{1+\fr{2}{\eta} } \eta^2 \Exp{ \sqnorm{  v^{t-1} - \nabla f(x^{t}) }  } + \frac{\eta^2 \sigma^2}{n}   \notag \\
		& \leq  & (1 - \eta) \Exp{ \sqnorm{  u^{t-1} - v^{t-1} } }  + 3  \eta \Exp{ \sqnorm{  v^{t-1} - \nabla f(x^{t}) }  }  + \frac{\eta^2 \sigma^2}{n}   \notag  \\
		& \leq  & (1 - \eta) \Exp{ \sqnorm{  u^{t-1} - v^{t-1} } }  + 6 \eta \Exp{ \sqnorm{  v^{t-1} - \nabla f(x^{t-1}) }  } \notag \\
		&& \qquad  + 6 \eta \Exp{ \sqnorm{  \nabla f(x^{t})  - \nabla f(x^{t-1}) }  } + \frac{\eta^2 \sigma^2}{n}   \notag \\
		& \leq  & (1 - \eta) \Exp{ \sqnorm{  u^{t-1} - v^{t-1} } }  + 6 \eta \Exp{ \sqnorm{  v^{t-1} - \nabla f(x^{t-1}) }  } \notag \\
		&& \qquad  + 6 L^2 \eta \Exp{ \sqnorm{  x^t - x^{t-1} }  } + \frac{\eta^2 \sigma^2}{n}  \notag ,
	\end{eqnarray}
	where the first inequality holds Assumption~\ref{as:BV}, the second inequality holds by Young's inequality, and the last step uses smoothness of $f(\cdot)$ (Assumption~\ref{as:main}), which concludes the proof of \eqref{eq:lem_Q_rec_DM}.
\end{proof}

\subsection{Controlling the error of contractive compression and double  momentum estimator}

\begin{lemma}\label{le:EF21-DM} 
	Let Assumption~\ref{as:main} be satisfied, and suppose $\cC$ is a contractive compressor. For every $i = 1,\ldots, n$, let the sequences $\cb{v_i^{t}}_{t\geq0}$, $\cb{u_i^{t}}_{t\geq0}$, and $\cb{g_i^{t}}_{t\geq0}$ be updated via 
	\begin{eqnarray*}
		v_i^{t} &=&  v_i^{t-1} + \eta \rb{ \nabla f_i (x^{t}, \xi_i^{t} ) -  v_i^{t-1} } , \\
	u_i^{t} &=&  u_i^{t-1} + \eta \rb{ v_i^{t}  -  u_i^{t-1} } ,  \\
	g_i^{t} &=&  g_i^{t-1} +  \cC\rb{  u_i^t - g_i^{t-1}  } .
\end{eqnarray*}
	Then for every $i = 1,\ldots, n$ and $t\geq0$ it holds
	\begin{eqnarray}\label{eq:wtV_rec_DM} 
		\Exp{ \sqnorm{g_i^t - u_i^{t}} }  &\leq& \rb{ 1-\fr{\al}{2} } \Exp{ \sqnorm{g_i^{t-1} - u_i^{t-1}} }  + \frac{6 \eta^2 }{\alpha}  \Exp{\sqnorm{u_i^{t-1} - v_i^{t-1} } }   \\
		&& \qquad + \frac{6 \eta^4 }{\alpha}  \Exp{\sqnorm{v_i^{t-1} - \nabla f_i(x^{t-1}, \xi_i^{t-1})} }  +  \frac{6 L_i^2 \eta^4 }{\alpha}   \Exp{\sqnorm{x^{t} - x^{t-1}} } +  \eta^4 \sigma^2 . \notag
	\end{eqnarray}
\end{lemma}

\begin{proof}
	By the update rules of $g_i^t$, $u_i^t$ and $v_i^t$, we derive
	\begin{eqnarray*}
		\Exp{\sqnorm{g_i^{t} - u_i^{t}} } & =& \Exp{ \sqnorm{ g_i^{t-1}  - u_i^{t}  +  \cC(  u_i^{t} - g_i^{t-1}  )  } }  \\
		&\overset{(i)}{ \leq} & (1-\al) \Exp{ \sqnorm{  u_i^{t} - g_i^{t-1}  } }  \\
		&\overset{(ii)}{ = } & (1-\al) \Exp{ \sqnorm{  u_i^{t-1} - g_i^{t-1} + \eta ( v_i^{t-1} - u_i^{t-1} ) + \eta^2 (\nabla f_i (x^t, \xi_i^{t}) - v_i^{t-1} )  } }  \\
		& = & (1-\al) \mathbb E \big[ \|  u_i^{t-1} - g_i^{t-1} + \eta ( v_i^{t-1} - u_i^{t-1} ) + \eta^2 (\nabla f_i (x^t) - v_i^{t-1} )   \\
		&& \qquad  + \eta^2 (\nabla f_i (x^t, \xi_i^{t}) - \nabla f_i (x^t) ) \|^2  \big] \\
		& = & (1-\al) \mathbb E \big[ \| \mathbb E_{\xi_i^{t}}\big[ \|  u_i^{t-1} - g_i^{t-1} + \eta ( v_i^{t-1} - u_i^{t-1} ) + \eta^2 (\nabla f_i (x^t) - v_i^{t-1} )   \\
		&& \qquad  + \eta^2 (\nabla f_i (x^t, \xi_i^{t}) - \nabla f_i (x^t) ) \|^2  \big] \big] \\
		& = & (1-\al) \Exp{ \sqnorm{ u_i^{t-1} - g_i^{t-1} + \eta ( v_i^{t-1} - u_i^{t-1} ) + \eta^2 (\nabla f_i (x^t) - v_i^{t-1} )   }  } \\
		&& \qquad + (1 - \alpha) \eta^4 \Exp{\sqnorm{\nabla f_i(x^{t}, \xi_i^{t}) - \nabla f_i(x^{t})} } \\
		& \overset{(iii)}{\leq} & (1-\al) (1+\rho) \Exp{ \sqnorm{ u_i^{t-1} - g_i^{t-1}}} \notag \\
		&& \qquad +(1-\alpha) (1+\rho^{-1}) \Exp{\sqnorm{ \eta ( v_i^{t-1} - u_i^{t-1} ) + \eta^2 (\nabla f_i (x^t) - v_i^{t-1} )   }  } \\
		&& \qquad +  \eta^4 \sigma^2  \\
		&\overset{(iv)}{ = } & (1-\theta) \Exp{ \sqnorm{ u_i^{t-1} - g_i^{t-1}}}  +  \eta^4 \sigma^2 \notag \\
		&& \qquad + \beta \Exp{\sqnorm{ \eta ( v_i^{t-1} - u_i^{t-1} ) + \eta^2 (\nabla f_i (x^{t-1}) - v_i^{t-1} ) +  \eta^2 (\nabla f_i (x^{t}) - \nabla f_i (x^{t-1} )   }  }   \\
		&\overset{(v)}{ \leq} & \left(1 - \theta\right)  \Exp{ \sqnorm{u_i^{t-1} -  g_i^{t-1}} } + 3\beta \eta^2 \Exp{\sqnorm{ v_i^{t-1} - u_i^{t-1} } } \notag \\
		&& \qquad  + 3 \beta \eta^4 \Exp{\sqnorm{ v_i^{t-1} - \nabla f_i(x^{t-1}) } } \notag \\
		&& \qquad + 3 \beta \eta^4 \Exp{\sqnorm{ \nabla f_i(x^{t}) - \nabla f_i(x^{t-1}) } } + \eta^4 \sigma^2\\
		& \leq & \left(1 - \theta\right)  \Exp{ \sqnorm{u_i^{t-1} -  g_i^{t-1}} } + 3\beta \eta^2 \Exp{\sqnorm{ v_i^{t-1} - u_i^{t-1} } } \notag \\
		&& \qquad  + 3 \beta \eta^4 \Exp{\sqnorm{ v_i^{t-1} - \nabla f_i(x^{t-1}) } } \notag \\
		&& \qquad + 3 \beta L_i^2 \eta^4 \Exp{\sqnorm{ x^t - x^{t-1} } } + \eta^4 \sigma^2\\
	\end{eqnarray*}
	where $(i)$ is due to definition of a contractive compressor (Definition~\ref{def:contractive_compressor}), $(ii)$ follows by the update rule of $v_i^t$ and $u_i^t$, $(iii)$ and $(v)$ hold by Young's inequality for any $\rho >0$. In $(iv)$, we introduced the notation $\theta  \eqdef 1 - (1-\alpha)(1 + \rho)$, and $\beta  \eqdef (1-\alpha) (1 + \rho^{-1})$. The last step follows by smoothness of $f_i(\cdot)$ (Assumption~\ref{as:main}). The proof is complete by the choice $\rho = \alpha / 2$, which guarantees $1-\theta \leq 1-\alpha/2$, and $3 \beta \leq 6/\alpha$ . 
\end{proof}

\clearpage
\section{EF21-SGDM with Absolute Compressor}\label{sec:appendix_abs}

In this section, we complement our theory by analyzing \algname{EF21-SGDM} under a different class of widely used biased compressors, namely, absolute compressors, which are defined as follows.

\begin{definition}[Absolute compressors]\label{def:absolute_compressor}
	We say that a (possibly randomized) map $\cC: \R^{d} \rightarrow \R^{d}$ is an {\em absolute compression operator} if   there exists a constant $\Delta > 0$ such that	
	\begin{eqnarray}\label{eq:abs_compressor}
		\Exp{\|\cC(x) - x\|^{2}} \leq \Delta^{2}, \qquad \forall x\in \R^d.
	\end{eqnarray}
\end{definition}

This class includes important examples of compressors such as hard-threshold sparsifier \citep{Sahu_abs_comp_2021}, (stochatsic) rounding schemes with bounded error \citep{Gupta_DL_limited_presition_2015} and scaled integer rounding \citep{Sapio_Scaling_Dist_ML_2021}.

 \begin{algorithm}[H]
	\centering
	\caption{\algname{EF21-SGDM (abs)}}\label{alg:EF21-SGDM-ABS}
	\begin{algorithmic}[1]
		\State \textbf{Input:} starting point $x^{0}$, step-size $\gamma>0$, momentum $\eta \in (0, 1],$ initial batch size $B_{\textnormal{init}} $
		\State Initialize $v_i^{0} = g_i^{0} = \frac{1}{B_{\textnormal{init}}} \sum_{j=1}^{B_{\textnormal{init}}} \nabla f_{i}(x^0, \xi_{i, j}^{0})$ for  $ i = 1, \ldots,  n$; $g^{0} = \frac{1}{n} \sum_{i=1}^n  g_i^{0}$
		\For{$t=0,1, 2, \dots , T-1 $}
		\State Master computes $x^{t+1} = x^t - \gamma g^t$ and broadcasts $x^{t+1}$ to all nodes 
		\For{{\bf all nodes $i =1,\dots, n$ in parallel}}
		\State \textcolor{mygreen}{ Compute momentum estimator $v_i^{t+1} = (1-\eta) v_i^{t}+ \eta  \nabla f_{i}(x^{t+1}, \xi_{i}^{t+1})  $ }
		\State Compress $c_i^{t+1} =  \cC\rb{ \frac{ v_i^{t+1} - g_i^{t}} {\gamma} } $ and send $c_i^{t+1} $ to the master
		\State Update local state $g_i^{t+1} = g_i^{t} + \gamma c_i^{t+1}$
		\EndFor
		\State Master computes $g^{t+1} = \frac{1}{n} \sum_{i=1}^n  g_i^{t+1}$ via  $g^{t+1} = g^{t} + \frac{1}{n} \sum_{i=1}^n \gamma  c_i^{t+1} $
		\EndFor
	\end{algorithmic}
\end{algorithm}

To accomodate absolute compressors into our \algname{EF21-SGDM} method, we need to make a slight modification to our algorithm, see Algorithm~\ref{alg:EF21-SGDM-ABS}. At each iteration, before compressing the difference $ v_i^{t+1} - g_i^{t}$, we divide it by the step-size $\gamma $. Later, we multiply the compressed vector $c_i^{t+1}$ by $\gamma$, i.e., have $$g_i^{t+1} = g_i^{t} + \gamma \, \cC\rb{ \fr{ v_i^{t+1} - g_i^{t} } {\gamma} } .$$ Such modification is necessary for absolute compressors because by Definition~\ref{def:absolute_compressor} the compression error is not proportional to $\sqnorm{x}$, but merely an absolute constant $\Delta^2$. In fact, Algorithm~\ref{alg:EF21-SGDM-ABS} is somewhat more universal in the sense that it can be also applied for contractive compressors.\footnote{It is straightforward to modify the proof of our Theorem~\ref{thm:main-distrib} for the case when Algorithm~\ref{alg:EF21-SGDM-ABS} is applied with a contractive compressor.} We derive the following result for \algname{EF21-SGDM (abs)}.

\begin{theorem}\label{thm:ef21-m-abs-comp}
	Let Assumptions~\ref{as:main} and \ref{as:BV} hold. Let $\hat x^T$ be sampled uniformly at random from the iterates of the method. Let Algorithm~\ref{alg:EF21-SGDM-ABS} run with an absolute compressor (Definition~\ref{def:absolute_compressor}). For all $\eta \in (0, 1]$ and $B_{\textnormal{init}} \geq 1$, with $\gamma \leq  \fr{\eta}{4 L},$ we have
	\begin{align}
		\label{eq:main-distrib:general-ABS}
		\Exp{  \sqnorm{\nabla f(\hat x^T) }  } \leq \cO\rb{\fr{\Psi_0}{\gamma T} + \gamma^2 \Delta^2  + \frac{\eta \sigma^2}{n}},
	\end{align}
	where $\Psi_0 \eqdef \delta_0 + \frac{\gamma}{\eta}\Exp{\norm{v^0 - \nabla f(x^0)}^2} $ is a Lyapunov function. With the following step-size, momentum parameter, and initial batch size 
	\begin{equation} \label{eq:ef21-m-abs-stepsize}
		\gamma =  \fr{\eta}{4 L}  , \qquad  \eta = \min\cb{ 1 , \rb{ \fr{  L^3 \delta_0  }{\Delta^2 T}}^{\nfr{1}{3}}, \rb{ \fr{ L \delta_0 n  }{\sigma^2 T}}^{\nfr{1}{2} } }, \qquad B_{init} = \frac{\sigma^2}{L \delta_0 n }
	\end{equation}
we have 
	\begin{eqnarray*}
		\Exp{  \sqnorm{\nabla f(\hat x^T) }  } &\leq&  
		\cO\rb{ \fr{ L \delta_0}{ T }  +  \rb{\fr{\delta_0 \Delta }{T }}^{\nfr{2}{3}}   + \rb{\fr{  L \delta_0 \sigma^2}{ n T } }^{\nfr{1}{2}} }  .
	\end{eqnarray*}
	
\end{theorem}

\begin{corollary}\label{cor:ef21sgdm-abs}
	Under the setting of Theorem~\ref{thm:ef21-m-abs-comp}, we have  $\Exp{\norm{\nabla f(\hat{x}^T) }} \leq \varepsilon$ after $T = \cO\rb{ \fr{L \delta_0 }{ \varepsilon^2} + \fr{\Delta \delta_0 }{ \varepsilon^{3}}  +  \fr{ \sigma^2 L \delta_0   }{n \varepsilon^4} }$ iterations.
\end{corollary}	

	\begin{remark}
		The sample complexity result in Corollary~\ref{cor:ef21sgdm-abs} matches the one derived for \algname{DoubleSqueeze} algorithm~\citep{DoubleSqueeze}, which is different from Algorithm~\ref{alg:EF21-SGDM-ABS}. 
	\end{remark}

\begin{proof}
	Similarly to the proof of Theorem~\ref{thm:main-distrib}, we control the error between $g^t$ and $\nabla f (x^t)$ by decomposing it into two terms
	$$
	\sqnorm{g^t - \nabla f(x^t)} \leq 2 \sqnorm{g^t - v^t} + 2 \sqnorm{v^t - \nabla f(x^t)} \leq 2 \suminn \sqnorm{g_i^t - v_i^t} + 2 \sqnorm{v^t - \nabla f(x^t)} . 
	$$
	Again, for the second term above we can use the recursion developed for momentum estimator Lemma~\ref{le:key_HB_recursion}. However, since we use a different compressor here, we need to bound $\sqnorm{g_i^t - v_i^t}$ term differently, thus we invoke Lemma~\ref{le:EF21_Cabs} for absolute compressor.  
	
	\textbf{Part I. Controlling  the error of momentum estimator on average for $v^t$.}
	Denote $P_t  \eqdef   \Exp{ \sqnorm{ v^t - \nabla f(x^{t})} }$, $R_t  \eqdef \Exp{\sqnorm{x^{t} - x^{t+1} } } $. Similarly to Part I  of the proof of Theorem~\ref{thm:main-distrib}, we have by Lemma~\ref{le:key_HB_recursion}
	\begin{eqnarray}\label{eq:mom_est_avg_dist_ABS}
		\frac{1}{T} \sum_{t=0}^{T-1} P_t  \leq  \fr{3 L^2 }{\eta^2 } \frac{1}{T} \sum_{t=0}^{T-1}  R_t +  \frac{\eta \sigma^2 }{n}  + \frac{1}{\eta T}  P_0 .
	\end{eqnarray}
	
	\textbf{Part II. Controlling the error of absolute compressor and momentum estimator.} By Lemma~\ref{le:EF21_Cabs} we have for any $0 < \eta \leq 1$ and $t\geq 0$
	\begin{eqnarray}\label{eq:var_abs_compressor}
		\wt V_t  \eqdef \suminn \Exp{\sqnorm{g_i^{t} - v_i^{t}} }  \leq \gamma^2 \Delta^2  .
	\end{eqnarray}
	\textbf{Part III. Combining steps I and II with descent lemma.}	
	By smoothness (Assumption~\ref{as:main}) of $f(\cdot)$ it follows from Lemma~\ref{le:descent} that for any $\gamma \leq 1/(2 L) $ we have 
	\begin{eqnarray}\label{eq:descent_ef21-M_dist_ABS}
		f(x^{t+1}) &\leq& f(x^{t}) - \fr{\g}{2} \sqnorm{\nabla f(x^t)}  - \frac{1}{4\gamma} \sqnorm{x^{t+1} - x^t } + \frac{\gamma}{2} \sqnorm{g^t - \nabla f(x^t) } \\
		&\leq& f(x^{t}) - \fr{\g}{2} \sqnorm{\nabla f(x^t)}  -  \frac{1}{4\gamma} \sqnorm{x^{t+1} - x^t } + \gamma \wt V_t + \gamma P_t   . \notag
	\end{eqnarray}
	Subtracting $f^*$ from both sides of \eqref{eq:descent_ef21-M_dist_ABS}, taking expectation and defining $\delta_t  \eqdef \Exp{f(x^t) - f^*}$, we derive
	\begin{eqnarray}
		\Exp{\sqnorm{\nabla f(\hat x^T)}} &=& \frac{1}{T} \sum_{t=0}^{T-1} \Exp{\sqnorm{\nabla f(x^t)}} \notag \\
		&\leq&  \fr{2 \delta_0 }{\gamma T} + 2 \frac{1}{T}\sum_{t=0}^{T-1} \wt V_t +  2 \frac{1}{T} \sum_{t=0}^{T-1} P_t  - \frac{1 }{2\gamma^2} \frac{1}{T} \sum_{t=0}^{T-1} R_t \notag  \\
		&\overset{(i)}{\leq}&  \fr{2 \delta_0 }{\gamma T} + 2 \gamma^2 \Delta^2 +  2 \frac{1}{T} \sum_{t=0}^{T-1} P_t  - \frac{1 }{2\gamma^2} \frac{1}{T} \sum_{t=0}^{T-1} R_t \notag  \\
		&\overset{(ii)}{\leq}&  \fr{2 \delta_0 }{\gamma T} +  2 \gamma^2 \Delta^2 + \rb{  \fr{6 L^2 }{\eta^2 }   - \frac{1 }{2\gamma^2} }  \frac{1}{T} \sum_{t=0}^{T-1} R_t  +  \frac{2\eta \sigma^2 }{n}  + \frac{1}{\eta T}  P_0 \notag  \\
		&\leq&  \fr{2 \delta_0 }{\gamma T} +  2 \gamma^2 \Delta^2  +  \frac{2\eta \sigma^2 }{n}  + \frac{1}{\eta T}  P_0 \notag  \\
	\end{eqnarray}
	where in $(i)$ and $(ii)$ we apply \eqref{eq:mom_est_avg_dist_ABS}, \eqref{eq:var_abs_compressor}, and in the last step we use the assumption on the step-size $\gamma \leq \eta / (4 L)$.
	
	Setting $\gamma =  \fr{\eta}{4 L} $, and taking $\eta \leq  \rb{ \fr{  L^3 \delta_0  }{\Delta^2 T}}^{\nfr{1}{3}}$ we can ensure that $\frac{\eta^2 \Delta^2}{L^2} \leq \frac{L \delta_0}{\eta T}$, since $ \eta \leq \rb{ \fr{ L \delta_0 n  }{\sigma^2 T}}^{\nfr{1}{2} } $ we have $\frac{\eta \sigma^2}{n} \leq \frac{L \delta_0}{\eta T}$. Finally, by setting the initial batch-size to $B_{init} = \frac{\sigma^2}{L \delta_0 n }$, we have $\frac{1}{\eta T} P_0 = \frac{\sigma^2}{\eta T n B_{init}} \leq \frac{L \delta_0}{\eta T}$. Therefore, we derive 
	\begin{eqnarray}
		\fr{1}{T}\sum_{t=0}^{T-1} \Exp{  \sqnorm{\nabla f(x^t) }  } &\leq&  
		\fr{2\delta_0}{\gamma T }  + 2 \gamma^2 \Delta^2      + \fr{2 \eta \sigma^2}{n}  + \frac{1}{\eta T}  P_0 \notag \\
		&=&  
		\fr{8 L \delta_0}{\eta T }  +  \fr{\eta^2 \Delta^2}{8 L^2 }      + \fr{2 \eta \sigma^2}{n}  + \frac{\sigma^2}{\eta T B_{init}}   \notag \\
		&=&  
		\cO\rb{ \fr{ L \delta_0}{ T }  +  \fr{\delta_0^{\nfr{2}{3}} \Delta^{\nfr{2}{3}}}{T^{\nfr{2}{3}} }   + \fr{\sigma (L\delta_0)^{\nfr{1}{2}} }{(n T)^{\nfr{1}{2}}} }  . \notag \\
	\end{eqnarray}
	
\end{proof}

\subsection{Controlling the error of absolute compression}

\begin{lemma}\label{le:EF21_Cabs} 
	Let $\cC$ be an absolute compressor and $g_i^{t+1}$ be updated according to Algorithm~\ref{alg:EF21-SGDM-ABS}, then for $t \geq 0$, we have  $\suminn \Exp{\sqnorm{g_i^{t} - v_i^{t}} }  \leq \gamma^2 \Delta^2  $.
\end{lemma}

\begin{proof}
	By the update rule for $g_i^{t+1}$ in Algorithm~\ref{alg:EF21-SGDM-ABS} and  Definition~\ref{def:absolute_compressor}, we can bound
	\begin{eqnarray*}
		\Exp{\sqnorm{g_i^{t+1} - v_i^{t+1}} } & =& \Exp{ \sqnorm{ \gamma \cC\rb{ \fr{v_i^{t+1} - g_i^t}{\gamma} } - (v_i^{t+1} - g_i^t ) } }  \\
		& =& \gamma^2 \Exp{ \sqnorm{  \cC\rb{ \fr{v_i^{t+1} - g_i^t}{\gamma} } - \fr{ v_i^{t+1} - g_i^t }{\gamma} } }  \leq  \gamma^2 \Delta^2 .
	\end{eqnarray*}
\end{proof}

\clearpage
\section{EF21-STORM/MVR }\label{sec:appendix_STORM}

\begin{algorithm}[H]
	\centering
	\caption{\algname{EF21-STORM/MVR}}\label{alg:EF21-STORM}
	\begin{algorithmic}[1]
		\State \textbf{Input:} $x^{0}$,  step-size $\gamma>0$, parameter $\eta \in (0, 1]$, $B_{\textnormal{init}}\geq1$
		\State Initialize $w_i^{0} = g_i^{0} = \frac{1}{B_{\textnormal{init}}} \sum_{j=1}^{B_{\textnormal{init}}} \nabla f_{i}(x^0, \xi_{i, j}^{0})$ for  $ i = 1, \ldots,  n$; $g^{0} = \frac{1}{n} \sum_{i=1}^n  g_i^{0}$
		\For{$t=0,1, 2, \dots , T-1 $}
		\State Master computes $x^{t+1} = x^t - \gamma g^t$ and broadcasts $x^{t+1}$ to all nodes 
		\For{{\bf all nodes $i =1,\dots, n$ in parallel}}
		\State Draw $\xi_i^{t+1}$ and compute two (stochastic) gradients $\nabla f_i(x^{t}, \xi_i^{t+1})$ and $\nabla f_i(x^{t+1}, \xi_i^{t+1})$
		\State \textcolor{mygreen}{ Compute variance reduced STORM/MVR estimator 
		\State $w_i^{t+1} =   \nabla f_{i}(x^{t+1}, \xi_{i}^{t+1})  +  (1-\eta) ( w_i^{t} - \nabla f_i(x^{t}, \xi_i^{t+1}) )  $ }
		\State Compress $c_i^{t+1} =  \cC( w_i^{t+1} - g_i^{t} )$ and send $c_i^{t+1} $ to the master
		\State Update local state $g_i^{t+1} = g_i^{t} +  c_i^{t+1}$
		\EndFor
		\State Master computes $g^{t+1} = \frac{1}{n} \sum_{i=1}^n  g_i^{t+1}$ via  $g^{t+1} = g^{t} + \frac{1}{n} \sum_{i=1}^n   c_i^{t+1} $
		\EndFor
	\end{algorithmic}
\end{algorithm}

\begin{assumption}[Individual smoothness\footnote{This assumption can be also relaxed to so-called expected smoothness.}]\label{as:ind_smoothness}	 
	For each $i = 1, \ldots, n$, every realization of $\xi_i \sim \cD_i $, the stochastic gradient $\nabla f_{i}(x,\xi_i)$ is $\ell_i$-Lipschitz, i.e.,  for all $  x, y \in \R^d$
	$$\norm{\nabla f_{i}(x,\xi_i) - \nabla f_{i}(y, \xi_i)} \le \ell_i\norm{x - y} .$$
	We denote $\wt \ell^2 \eqdef \suminn \ell_i^2$	
\end{assumption}

\begin{theorem}\label{thm:ef21_storm}
	Let Assumptions~\ref{as:main}, \ref{as:BV} and \ref{as:ind_smoothness} hold. Let $\hat x^T$ be sampled uniformly at random from the iterates of the method. Let Algorithm~\ref{alg:EF21-STORM} run with a contractive compressor. For all $\eta \in (0, 1]$ and $B_{\textnormal{init}} \geq 1$, with $\gamma \leq \min\cb{\fr{\alpha}{8 \wL}, \fr{\sqrt{\al}}{6\wt{\ell}},  \fr{\sqrt{n \eta}}{8 \wt \ell}  },$ we have
	\begin{align}
		\label{eq:ef21_storm_eta}
		\Exp{  \sqnorm{\nabla f(\hat x^T) }  } \leq \cO\rb{\fr{\Psi_0}{\gamma T} + \fr{\eta^3 \sigma^2 }{\al^2}  +  \fr{\eta^2 \sigma^2 }{\al} + \frac{\eta \sigma^2}{n}},
	\end{align}
	where $\Psi_0 \eqdef \delta_0 + \frac{\gamma}{\eta}\Exp{\norm{v^0 - \nabla f(x^0)}^2} + \frac{\gamma \eta}{\alpha^2} \frac{1}{n} \sum_{i=1}^n \Exp{\norm{v^0_i - \nabla f_i(x^0)}^2}$. 
	With the following step-size, momentum parameter, and initial batch size  
	\begin{equation} \label{eq:storm-stepsize}
		\gamma = \min\cb{\fr{\alpha}{8 \wL}, \fr{\sqrt{\al}}{6\wt{\ell}}, \fr{\sqrt{n \eta}}{8 \wt \ell}  } , \quad  \eta = \min\cb{ \alpha  ,  \rb{ \fr{  \wt \ell \delta_0 \al^2 }{\sigma^2 \sqrt{n} T}}^{\nfr{2}{7} }, \rb{ \fr{  \wt \ell \delta_0 \al }{\sigma^2 \sqrt{n} T}}^{\nfr{2}{5} }, \rb{ \fr{  \wt \ell \delta_0 \sqrt{n}  }{\sigma^2 T}}^{\nfr{2}{3} } }  \notag , 
	\end{equation}
  	and $B_{\textnormal{init}} = \max\cb{ \frac{\sigma^2}{L \delta_0 n} , \frac{\alpha n }{T}}$, we have 
	\begin{eqnarray*}
		\Exp{  \sqnorm{\nabla f(\hat x^T) }  } &\leq&  
		\cO\rb{ \fr{ \wL \delta_0}{ \al  T } + \fr{ \wt{\ell} \delta_0}{ \sqrt{\al}  T } + \rb{\fr{  \wt \ell \delta_0 \sigma^{\nicefrac{1}{3}} }{\alpha^{\nicefrac{1}{3}} \sqrt{n}  T } }^{\nfr{6}{7}}  + \rb{\fr{  \wt \ell \delta_0 \sigma^{\nfr{1}{2}} }{\alpha^{\nicefrac{1}{4}} \sqrt{n}  T } }^{\nfr{4}{5}} + \rb{\fr{  \wt \ell \delta_0 \sigma }{ n  T } }^{\nfr{2}{3}}  } .
	\end{eqnarray*}
\end{theorem}

\begin{corollary}\label{cor:ef21_storm}
	Under the setting of Theorem~\ref{thm:ef21_storm}.  we have $\Exp{\norm{\nabla f(\hat{x}^T) }} \leq \varepsilon$ after $T = \cO\rb{ \fr{\wt \ell \delta_0 }{\al \varepsilon^2} + \fr{  \wt \ell \delta_0 \sigma^{\nicefrac{1}{3}} }{\alpha^{\nicefrac{1}{3}} \sqrt{n}  \varepsilon^{\nicefrac{7}{3}} } + \fr{\wt \ell \delta_0 \sigma^{\nfr{1}{2}}  }{  \alpha^{\nfr{1}{4}} \sqrt{n} \varepsilon^{\nfr{5}{2}}}  +  \fr{\wt \ell \delta_0 \sigma  }{n \varepsilon^3} }$ iterations.
\end{corollary}	
Recently, \citet{Yau_DoCoM_SGT_2022} propose and analyze a DoCoM-SGT algorithm for decentralized optimization with contractive compressor under the above Assumption~\ref{as:ind_smoothness}. When their method is specialized to centralized setting (with mixing constant $\rho = 1$), their total sample complexity becomes $\cO\rb{ \frac{\wt \ell }{\al \varepsilon^2} + \frac{n^{4/5} \sigma^{3/2} }{\alpha^{9/4} \varepsilon^{3/2}} + \frac{\sigma^3}{n \varepsilon^3} } $ (see Table 1 or Theorem 4.1 in \citep{Yau_DoCoM_SGT_2022}). In contrast, the sample complexity given in our Corollary~\ref{cor:ef21_storm} improves the dependence on $\sigma$ in the last term and, moreover, achieves the linear speedup in terms of $n$ for all stochastic terms in the sample complexity. 

\begin{proof}
	Similarly to the proof of Theorem~\ref{thm:main-distrib}, we control the error between $g^t$ and $\nabla f (x^t)$ by decomposing it into two terms
	$$
	\sqnorm{g^t - \nabla f(x^t)} \leq 2 \sqnorm{g^t - w^t} + 2 \sqnorm{w^t - \nabla f(x^t)} \leq 2 \suminn \sqnorm{g_i^t - w_i^t} + 2 \sqnorm{w^t - \nabla f(x^t)} .
	$$
	In the following, we develop a recursive bound for each term above separately.
	
		\textbf{Part I. Controlling  the variance of STORM/MVR estimator for each $w_i^t$ and on average $w^t$.} Recall that by Lemma~\ref{le:key_STORM_recursion}-\eqref{eq:wtP_rec_storm}, we have for each $i = 1, \ldots, n$, and any $0 < \eta \leq 1$ and $t\geq 0$
	\begin{eqnarray}\label{eq:wtPi_rec_storm}
			\Exp{ \sqnorm{w_i^{t+1} - \nabla f_i(x^{t+1})} } \leq (1- \eta) 
		\Exp{ \sqnorm{w_i^{t} - \nabla f_i(x^{t})} }  + 2 \ell_i^2  \Exp{\sqnorm{x^{t} - x^{t+1}}} + 2 \eta^2 \sigma^2 .
	\end{eqnarray}
	Averaging inequalities \eqref{eq:wtPi_rec_storm} over $i=1,\ldots,n$ and denoting $\wt P_t : = \suminn \Exp{ \sqnorm{w_i^{t} - \nabla f_i(x^{t})} } $, $R_t  \eqdef \Exp{\sqnorm{x^{t} - x^{t+1} } } $ we have 
	\begin{eqnarray*}
		\wt P_{t+1}  \leq (1- \eta)   \wt P_t  +  2  \wt \ell^2  R_t + 2 \eta^2 \sigma^2 .
	\end{eqnarray*}
	Summing up the above inequality for $t=0, \ldots, T-1$, we derive 
	\begin{eqnarray}\label{eq:wt_mom_est_avg_dist_storm}
		\frac{1}{T} \sum_{t=0}^{T-1} \wt P_t  \leq  \fr{2 \wt \ell^2 }{\eta } \frac{1}{T} \sum_{t=0}^{T-1}  R_t +  2 \eta \sigma^2 + \frac{1}{\eta T} \wt P_0.
	\end{eqnarray}
	Similarly by Lemma~\ref{le:key_STORM_recursion}-\eqref{eq:P_rec_storm} denoting $ P_t : =  \Exp{ \sqnorm{w^{t} - \nabla f(x^{t})} } $, we have 
	\begin{eqnarray}\label{eq:mom_est_avg_dist_storm}
		\frac{1}{T}\sum_{t=0}^{T-1} P_t  \leq  \fr{2 \wt \ell^2 }{\eta n } \frac{1}{T}\sum_{t=0}^{T-1}  R_t +  \fr{2 \eta \sigma^2}{n}  + \fr{1}{\eta T } P_0 .
	\end{eqnarray}

	\textbf{Part II. Controlling  the variance of contractive compressor and STORM/MVR estimator.} By Lemma~\ref{le:EF21-storm} we have for each $i = 1, \ldots, n$, and any $0 < \eta \leq 1$ and $t\geq 0$
	\begin{eqnarray}\label{eq:wtVi_rec_storm}
		\Exp{ \sqnorm{ g_i^{t+1} - w_i^{t+1} }}   &\leq& \rb{ 1-\fr{\al}{2} } \Exp{ \sqnorm{ g_i^{t} - w_i^{t} }}    +  \frac{4 \eta^2}{\alpha} \Exp{\sqnorm{w_i^{t} - \nabla f_i(x^{t})} } \notag \\
	&& \qquad + \rb{ \frac{4 L_i^2}{\alpha} + \ell_i^2} \Exp{ \sqnorm{ x^{t+1} -x^{t} } }  + 2 \eta^2 \sigma^2 .
	\end{eqnarray}
	
	Averaging inequalities \eqref{eq:wtVi_rec_storm} over $i=1,\ldots,n$,  denoting $ \wt V_t  \eqdef \suminn  \Exp{ \sqnorm{g_i^t - w_i^t} }$, and summing up the resulting inequality for $t=0, \ldots, T-1$, we obtain 
	\begin{eqnarray}\label{eq:ef21_mom_est_avg_dist_storm}
		\frac{1}{T} \sum_{t=0}^{T-1} \wt V_t  &\leq&  \fr{8 \eta^2 }{\al^2} \frac{1}{T} \sum_{t=0}^{T-1} \wt P_t +   \rb{ \frac{8 \wL^2}{\alpha^2} + \fr{2\wt{\ell}^2}{\al}  } \frac{1}{T} \sum_{t=0}^{T-1} R_t +   \fr{2\eta^2 \sigma^2 }{\al} \notag \\
		&\leq&    \rb{ \frac{8 \wL^2}{\alpha^2} + \fr{2\wt{\ell}^2}{\al} +  \fr{ 16 \eta \wt \ell^2 }{ \alpha^2   } } \frac{1}{T} \sum_{t=0}^{T-1} R_t  \notag \\ 
		&& \qquad +  \frac{16  \eta^3 \sigma^2}{\alpha^2}  +   \fr{2\eta^2 \sigma^2 }{\al} + \frac{8 \eta }{\alpha^2 T} \wt P_0  .
	\end{eqnarray}
	
		\textbf{Part III. Combining steps I and II with descent lemma.}	
	By smoothness (Assumption~\ref{as:main}) of $f(\cdot)$ it follows from Lemma~\ref{le:descent} that for any $\gamma \leq 1/(2 L) $ we have 
	\begin{eqnarray}\label{eq:descent_ef21_strom}
		f(x^{t+1}) &\leq& f(x^{t}) - \fr{\g}{2} \sqnorm{\nabla f(x^t)}  - \frac{1}{4\gamma} \sqnorm{x^{t+1} - x^t } + \frac{\gamma}{2} \sqnorm{g^t - \nabla f(x^t) } \\
		&\leq& f(x^{t}) - \fr{\g}{2} \sqnorm{\nabla f(x^t)}  -  \frac{1}{4\gamma} \sqnorm{x^{t+1} - x^t } + \gamma \suminn\sqnorm{g_i^t - w_i^t } + \gamma \sqnorm{w^t - \nabla f(x^t) }  . \notag
	\end{eqnarray}
	Subtracting $f^*$ from both sides of \eqref{eq:descent_ef21_strom}, taking expectation and defining $\delta_t  \eqdef \Exp{f(x^t) - f^*}$, we derive
	\begin{eqnarray}
		\Exp{\sqnorm{\nabla f(\hat x^T)}} &=& \frac{1}{T} \sum_{t=0}^{T-1} \Exp{\sqnorm{\nabla f(x^t)}} \notag \\
		&\leq&  \fr{2 \delta_0 }{\gamma T} + 2 \frac{1}{T}\sum_{t=0}^{T-1} \wt V_t +  2 \frac{1}{T} \sum_{t=0}^{T-1} P_t  - \frac{1 }{2\gamma^2} \frac{1}{T} \sum_{t=0}^{T-1} R_t \notag  \\
		&\overset{(i)}{\leq}&  \fr{2 \delta_0 }{\gamma T} +  \rb{ \frac{16 \wL^2}{\alpha^2} + \fr{4\wt{\ell}^2}{\al} +  \fr{ 32 \eta \wt \ell^2 }{ \alpha^2   } } \frac{1}{T} \sum_{t=0}^{T-1} R_t  +  2 \frac{1}{T} \sum_{t=0}^{T-1} P_t  - \frac{1 }{2\gamma^2} \frac{1}{T} \sum_{t=0}^{T-1} R_t \notag  \\
		&&\qquad  + \frac{32  \eta^3 \sigma^2}{\alpha^2}  +   \fr{4\eta^2 \sigma^2 }{\al} + \frac{16 \eta }{\alpha^2 T} \wt P_0  \notag \\
		&\overset{(ii)}{\leq}&  \fr{2 \delta_0 }{\gamma T} +  \rb{ \frac{16 \wL^2}{\alpha^2} + \fr{4\wt{\ell}^2}{\al} +  \fr{ 32 \eta \wt \ell^2 }{ \alpha^2   } + \fr{4 \wt \ell^2 }{\eta n } - \frac{1 }{2\gamma^2}  } \frac{1}{T} \sum_{t=0}^{T-1} R_t  \notag  \\
		&&\qquad  + \frac{32  \eta^3 \sigma^2}{\alpha^2}  +   \fr{4\eta^2 \sigma^2 }{\al}  +  \fr{4 \eta \sigma^2}{n} + \frac{16 \eta }{\alpha^2 T} \wt P_0  + \fr{2}{\eta T } P_0    \notag \\
		& \leq & \fr{2 \delta_0 }{\gamma T} +   \frac{32  \eta^3 \sigma^2}{\alpha^2}  +   \fr{4\eta^2 \sigma^2 }{\al}  +  \fr{4 \eta \sigma^2}{n} + \frac{16 \eta }{\alpha^2 T} \wt P_0  + \fr{2}{\eta T } P_0  ,  \notag 
	\end{eqnarray}
		where in $(i)$ we apply \eqref{eq:ef21_mom_est_avg_dist_storm}, in $(ii)$ we use \eqref{eq:mom_est_avg_dist_storm}, and the last step follows by assumption on the step-size, which proves \eqref{eq:ef21_storm_eta}. 
	
	
	We now find the particular values of parameters. Using $w_i^0 = \frac{1}{B_{\textnormal{init}}} \sum_{j=1}^{B_{\textnormal{init}}} \nabla f_{i}(x^0, \xi_{i, j}^{0})$ for all $i = 1, \ldots,  n$, we have
	\begin{align*}
		P_0 = \Exp{\norm{w^0 - \nabla f(x^0)}^2} \leq \frac{\sigma^2}{n B_{\textnormal{init}}} \textnormal{ and } \widetilde{P}_0 = \frac{1}{n} \sum_{i=1}^n \Exp{\norm{w^0_i - \nabla f_i(x^0)}^2} \leq \frac{\sigma^2}{B_{\textnormal{init}}}.
	\end{align*}
	We can substitute the choice of $\gamma$ and obtain
\begin{eqnarray*}
	\Exp{\sqnorm{\nabla f(\hat x^T)}}  & = &  
	\cO\rb{ \fr{\delta_0}{\gamma T }  +  \fr{  \eta^3  \sigma^2 }{\al^2}     +    \fr{\eta^2 \sigma^2 }{\al}     + \fr{ \eta \sigma^2}{n} + \frac{\sigma^2}{\eta n B_{\textnormal{init}} T} + \frac{\eta \sigma^2}{\alpha^2 B_{\textnormal{init}} T} } \\
	& = &  
	\cO\rb{  \fr{ \wL \delta_0}{ \al  T } + \fr{ \wt{\ell} \delta_0}{ \sqrt{\al}  T } + \fr{ \wt \ell \delta_0}{\sqrt{n \eta }  T }  +  \fr{  \eta^3  \sigma^2 }{\al^2}     +    \fr{\eta^2 \sigma^2 }{\al}     + \fr{ \eta \sigma^2}{n}  + \frac{\sigma^2}{\eta n B_{\textnormal{init}} T} + \frac{\eta \sigma^2}{\alpha^2 B_{\textnormal{init}} T} } . 
\end{eqnarray*}

	Since $B_{\textnormal{init}} \geq \frac{\sigma^2}{L \delta_0 n},$ we have
	\begin{eqnarray*}
		\Exp{  \sqnorm{\nabla f(\hat x^T) }  } &= & \cO\rb{  \fr{ \wL \delta_0}{ \al  T } + \fr{ \wt{\ell} \delta_0}{ \sqrt{\al}  T } + \fr{ \wt \ell \delta_0}{\sqrt{n \eta }  T }  +  \fr{  \eta^3  \sigma^2 }{\al^2}     +    \fr{\eta^2 \sigma^2 }{\al}     + \fr{ \eta \sigma^2}{n}  + \frac{\eta \sigma^2}{\alpha^2 B_{\textnormal{init}} T} } . 
	\end{eqnarray*}
	
	Notice that the choice of the momentum parameter such that $\eta \leq \rb{ \fr{  \wt \ell \delta_0 \al^2 }{\sigma^2 \sqrt{n} T}}^{\nfr{2}{7} }$,  $\eta \leq \rb{ \fr{  \wt \ell \delta_0 \al }{\sigma^2 \sqrt{n} T}}^{\nfr{2}{5} }$, $\eta \leq \rb{ \fr{  \wt \ell \delta_0 \sqrt{n}  }{\sigma^2 T}}^{\nfr{2}{3} } $,  and $\eta \leq \rb{\frac{\wt \ell \delta_0 \alpha^2 B_{\textnormal{init}}}{\sigma^2 \sqrt{n}}}^{\nfr{2}{3} }  $ ensures that $\fr{ \eta^3 \sigma^2 }{\al^2}   \leq\fr{ \wt \ell \delta_0}{\sqrt{n \eta }  T }  $, $\fr{\eta^2 \sigma^2 }{\al} \leq \fr{ \wt \ell \delta_0}{\sqrt{n \eta }  T } $, $\frac{ \eta \sigma^2}{n} \leq \fr{ \wt \ell \delta_0}{\sqrt{n \eta }  T } ,$ and $\frac{\eta \sigma^2}{\alpha^2 B_{\textnormal{init}} T} \leq \fr{ \wt \ell \delta_0}{\sqrt{n \eta }  T }  .$ Therefore, we have
		\begin{eqnarray*}
		\Exp{\sqnorm{\nabla f(\hat x^T)}}   = 
		\cO\rb{ \fr{ \wL \delta_0}{ \al  T } + \fr{ \wt{\ell} \delta_0}{ \sqrt{\al}  T } + \rb{\fr{  \wt \ell \delta_0 \sigma^{\nicefrac{1}{3}} }{\alpha^{\nicefrac{1}{3}} \sqrt{n}  T } }^{\nfr{6}{7}}  + \rb{\fr{  \wt \ell \delta_0 \sigma^{\nfr{1}{2}} }{\alpha^{\nicefrac{1}{4}} \sqrt{n}  T } }^{\nfr{4}{5}} + \rb{\fr{  \wt \ell \delta_0 \sigma }{ n  T } }^{\nfr{2}{3}}  + \rb{\frac{\wt \ell \delta_0 \sigma  }{ \sqrt{n}}}^{\nfr{2}{3} } \frac{\alpha^{1/3}}{B_{\textnormal{init}}^{1/3} T } }.
	\end{eqnarray*}

	Using $B_{\textnormal{init}} \geq \frac{\alpha n }{T },$ we obtain
	\begin{eqnarray*}
		\Exp{\sqnorm{\nabla f(\hat x^T)}}  & = & 
		\cO\rb{ \fr{ \wL \delta_0}{ \al  T } + \fr{ \wt{\ell} \delta_0}{ \sqrt{\al}  T } + \rb{\fr{  \wt \ell \delta_0 \sigma^{\nicefrac{1}{3}} }{\alpha^{\nicefrac{1}{3}} \sqrt{n}  T } }^{\nfr{6}{7}}  + \rb{\fr{  \wt \ell \delta_0 \sigma^{\nfr{1}{2}} }{\alpha^{\nicefrac{1}{4}} \sqrt{n}  T } }^{\nfr{4}{5}} + \rb{\fr{  \wt \ell \delta_0 \sigma }{ n  T } }^{\nfr{2}{3}} }.
	\end{eqnarray*}

\end{proof}

\subsection{Controlling the variance of STORM/MVR estimator}
\begin{lemma}\label{le:key_STORM_recursion}
	Let Assumptions~\ref{as:BV} and \ref{as:ind_smoothness} be satisfied, and suppose $0 < \eta \leq 1$. For every $i = 1,\ldots, n$, let the sequence $\cb{w_i^{t}}_{t\geq 0}$ be updated via $w_i^{t+1} =   \nabla f_{i}(x^{t+1}, \xi_{i}^{t+1})  +  (1-\eta) ( w_i^{t} - \nabla f_i(x^{t}, \xi_i^{t+1}) )  $ 
	Define the sequence $w^t  \eqdef \suminn w_i^t$. Then for every $i = 1,\ldots, n$ and $t\geq0$ it holds
	
	\begin{eqnarray}\label{eq:wtP_rec_storm}
		\Exp{ \sqnorm{w_i^{t+1} - \nabla f_i(x^{t+1})} } \leq (1- \eta) 
	  \Exp{ \sqnorm{w_i^{t} - \nabla f_i(x^{t})} }  + 2 \ell_i^2  \Exp{\sqnorm{x^{t} - x^{t+1}}} + 2 \eta^2 \sigma^2 ,
	\end{eqnarray}
	\begin{eqnarray}\label{eq:P_rec_storm}
			\Exp{ \sqnorm{w^{t+1} - \nabla f(x^{t+1})} } \leq (1- \eta) 
		\Exp{ \sqnorm{w^{t} - \nabla f(x^{t})} }  +\frac{ 2 \wt \ell^2 }{n} \Exp{\sqnorm{x^{t} - x^{t+1}}} + \frac{2 \eta^2 \sigma^2}{n} .
	\end{eqnarray}
\end{lemma}

\begin{proof}
	For each $t = 0, \ldots, T-1$, define a random vector $\xi^{t}  \eqdef (\xi_1^{t}, \ldots, \xi_n^{t})$ and denote by $\nabla f (x^{t}, \xi^{t})  \eqdef \suminn \nabla f_{i} (x^{t}, \xi_i^{t})$. Note that the entries of the random vector $\xi^{t}$ are independent and $\Expu{\xi^{t}}{ \nabla f (x^{t}, \xi^{t}) } = \nabla f (x^{t})$, then we have 
	$$w^{t+1} =  \nabla f (x^{t+1}, \xi^{t+1} ) + (1 - \eta ) \rb{  w^{t} - \nabla f (x^{t}, \xi^{t+1} ) } , $$
	where $w^t = \suminn w_i^t$ is an auxiliary sequence. 
	
	We define
	$$
	\cV_i^t \eqdef  \nabla f_{i}(x^{t}, \xi_i^{t}) - \nabla f_i(x^{t}) , \qquad \cV^t \eqdef  \suminn \cV_i^t,
	$$  
	$$
	\cW_i^t \eqdef  \nabla f_i(x^t)  - \nabla f_{i}(x^t, \xi_i^{t+1}) +  \nabla f_{i}(x^{t+1}, \xi_i^{t+1}) - \nabla f_i(x^{t+1}) , \qquad \cW^t \eqdef  \suminn \cW_i^t .
	$$
	Then by Assumptions~\ref{as:BV}, we have 
	\begin{equation}\label{eq:storm_unbiased_V_W}
	\Exp{ \cV_i^t } = \Exp{ \cW_i^t  } =\Exp{ \cV^t  } = \Exp{ \cW^t  } = 0, 
	\end{equation}
	\begin{equation}\label{eq:storm_variance}
	\Exp{ \sqnorm{\cV_i^t}  } \leq {\sigma^2},  \qquad \Exp{ \sqnorm{\cV^t}  } \leq \fr{\sigma^2}{n}.
	\end{equation}
	Furthermore, we can derive
	\begin{eqnarray*}
	\Exp{ \sqnorm{\cW^t} } & = & \Exp{ \sqnorm{ \suminn \cW_i^t }} \\
	&=&   \frac{1}{n^2} \Exp{ \sqnorm{ \sumin \cW_i^t } } \\
	&=&   \frac{1}{n^2} \sumin \Exp{ \sqnorm{ \cW_i^t } } + \frac{1}{n^2} \sum_{i\neq j} \Exp{ \langle \cW_i^t, \cW_j^t \rangle   } \\
		&\overset{(i)}{=}&   \frac{1}{n^2} \sumin \Exp{ \sqnorm{ \cW_i^t } } + \frac{1}{n^2} \sum_{i\neq j} \langle \Exp{ \cW_i^t}, \Exp{ \cW_j^t}  \rangle    \\
	&=&   \frac{1}{n^2} \sumin \Exp{ \sqnorm{ \cW_i^t } }  \\
	&\leq&   \frac{1}{n^2} \sumin \Exp{ \sqnorm{ \nabla f_{i}(x^{t+1}, \xi_i^{t+1})  -  \nabla f_{i}(x^t, \xi_i^{t+1}) } }  \\
	&\leq&   \frac{1}{n^2} \sumin \ell_i^2 \Exp{ \sqnorm{ x^{t+1} - x^t } }   = \frac{\wt \ell^2}{n} \Exp{ \sqnorm{ x^{t+1} - x^t } } ,
	\end{eqnarray*}
where $(i)$ holds by the conditional independence of $\cW_i^t$ and  $\cW_j^t$, and the last inequality follows by the individual smoothness of stochastic functions (Assumption~\ref{as:ind_smoothness}). Therefore, we have
\begin{equation}\label{eq:storm_W_variance}
\Exp{ \sqnorm{\cW_i^t} } \leq {\ell_i^2} \Exp{\sqnorm{ x^{t+1} - x^t } } ,  \qquad \Exp{ \sqnorm{\cW^t} } \leq \fr{\wt \ell^2}{n} \Exp{ \sqnorm{ x^{t+1} - x^t } } ,
\end{equation}
where the first inequality is obtained by using a similar derivation.

	By the update rule for $w^t$, we can also derive 
	\begin{eqnarray*}
		w^{t+1} - \nabla f(x^{t+1})  &= & (1 - \eta) \rb{ w^t - \nabla f(x^t, \xi^{t+1}) } +  \rb{ \nabla f(x^{t+1}, \xi^{t+1}) - \nabla f(x^{t+1}) } \notag \\
		&= & (1 - \eta) \rb{ w^t - \nabla f(x^t) } + \eta  \rb{ \nabla f(x^{t+1}, \xi^{t+1} ) - \nabla f(x^{t+1}) } \notag \\
		&& \quad + (1 - \eta) \rb{  \rb{ \nabla f(x^t)  - \nabla f(x^t, \xi^{t+1}) +  \nabla f(x^{t+1}, \xi^{t+1}) - \nabla f(x^{t+1}) } } \notag \\
		&= & (1 - \eta) \rb{ w^t - \nabla f(x^t) } + \eta \cV^{t+1} + (1 - \eta) \cW^t  .
	\end{eqnarray*}
	
	Therefore, we have
	\begin{eqnarray*}
		\Exp{\sqnorm{ w^{t+1} - \nabla f(x^{t+1}) } } & \leq & \Exp{\Expu{\xi^{t+1}}{\sqnorm{ (1 - \eta) \rb{ w^t - \nabla f(x^t) } + \eta \cV_{t+1} + (1 - \eta) \cW_t }  } }  \notag \\
		& = &  (1 - \eta)^2 \Exp{  \sqnorm{ w^t - \nabla f(x^t) } }  + \Exp{ \sqnorm{ \eta \cV^{t+1} + (1 - \eta) \cW^t } } \notag \\
		& \leq &  (1 - \eta) \sqnorm{ w^t - \nabla f(x^t) } + 2 \eta^2 \Exp{ \sqnorm{  \cV^{t+1} } }  + 2 \Exp{ \sqnorm{ \cW^t } } \notag \\
		& \leq &  (1 - \eta) \Exp{ \sqnorm{ w^t - \nabla f(x^t) } } +\fr{ 2 \sigma^2 \eta^2}{n}  + \fr{ 2 \wt\ell^2}{n}  \Exp{\sqnorm{x^{t+1} - x^t}}, 
	\end{eqnarray*}
where the last inequality holds by \eqref{eq:storm_variance} and \eqref{eq:storm_W_variance}.  
	Similarly for each $i = 1, \ldots, n$, we have
	\begin{eqnarray}\label{le:storm_vec_recursion_i}
		w_i^{t+1} - \nabla f_i(x^{t+1})  &= &  (1 - \eta) \rb{ w_i^t - \nabla f_i(x^t) } + \eta \cV_i^{t+1} + (1 - \eta) \cW_i^t  .
	\end{eqnarray}
	Thus,
	\begin{eqnarray*}
		\Exp{\sqnorm{ w_i^{t+1} - \nabla f_i(x^{t+1}) } } & \leq &   (1 - \eta) \Exp{ \sqnorm{ w_i^t - \nabla f_i(x^t) } } + 2 \sigma^2 \eta^2   + 2 \ell_i^2   R_t. 
	\end{eqnarray*}
\end{proof}	

\subsection{Controlling  the variance of contractive compression and STORM/MVR estimator}
\begin{lemma}\label{le:EF21-storm} 
Let Assumptions~\ref{as:main}, \ref{as:BV} and \ref{as:ind_smoothness} be satisfied, and suppose $0 < \eta \leq 1$. For every $i = 1,\ldots, n$, let the sequences $\cb{w_i^{t}}_{t\geq 0}$ and $\cb{g_i^t}_{t\geq 0}$ be updated via 
\begin{eqnarray}
	w_i^{t+1}& =&   \nabla f_{i}(x^{t+1}, \xi_{i}^{t+1})  +  (1-\eta) ( w_i^{t} - \nabla f_i(x^{t}, \xi_i^{t+1}) )  , \notag \\
	g_i^{t+1} &=&  g_i^{t} +  \cC\rb{  w_i^{t+1} - g_i^{t}  } . \notag  
\end{eqnarray}
 Then for every $i = 1,\ldots, n$ and $t\geq0$ it holds
	\begin{eqnarray}\label{eq:rec_ef21_storm_avg} 
		\Exp{ \sqnorm{ g_i^{t+1} - w_i^{t+1} }}   &\leq& \rb{ 1-\fr{\al}{2} } \Exp{ \sqnorm{ g_i^{t} - w_i^{t} }}    +  \frac{4 \eta^2}{\alpha} \Exp{\sqnorm{w_i^{t} - \nabla f_i(x^{t})} } \notag \\
		&& \qquad + \rb{ \frac{4 L_i^2}{\alpha} + \ell_i^2} \Exp{ \sqnorm{ x^{t+1} -x^{t} } }  + 2 \eta^2 \sigma^2 .
	\end{eqnarray}
\end{lemma}

\begin{proof}
	By the update rule of $w_i^{t}$, $g_i^t$, and definition of $\cV_i^t $, $\cW_i^t$ given in the proof of Lemma~\ref{le:key_STORM_recursion}, we can derive 
	\begin{eqnarray*}
		\Exp{\sqnorm{g_i^{t+1} - w_i^{t+1}} } & =& \Exp{ \sqnorm{ \cC( w_i^{t+1} - g_i^t) - (w_i^{t+1} - g_i^t ) } }  \\
		& \overset{(i)}{\leq} & (1-\al) \Exp{ \sqnorm{  w_i^{t+1} - g_i^t  } }  \\
		& \overset{(ii)}{=} & (1-\al) \Exp{ \sqnorm{  (1 - \eta) \rb{ w_i^t - \nabla f_i(x^t) } + \eta \cV_i^{t+1} + (1 - \eta) \cW_i^t  + \nabla f_i(x^{t+1}) - g_i^t  } }  \\
		& =& (1-\al) \Exp{\Expu{\xi_i^{t+1}}{ \sqnorm{  (1 - \eta) \rb{ w_i^t - \nabla f_i(x^t) } + \eta \cV_i^{t+1} + (1 - \eta) \cW_i^t  + \nabla f_i(x^{t+1}) - g_i^t  } } } \\
		& \overset{(iii)}{=} & (1-\al) \Exp{ \sqnorm{ (1 - \eta) \rb{ w_i^t - \nabla f_i(x^t) }  + \nabla f_i(x^{t+1}) - g_i^t } }  \\
		&& \qquad + (1 - \alpha)\Exp{ \sqnorm{ \eta \cV_i^{t+1} + (1 - \eta) \cW_i^t } } \\
		& = & (1-\al) \Exp{ \sqnorm{ \rb{ w_i^t - g_i^t }  + \rb{ \nabla f_i(x^{t+1}) -  \nabla f_i(x^t) } - \eta \rb{ w_i^t - \nabla f_i(x^t) }  } } \\
		&& \qquad + (1 - \alpha) \Exp{ \sqnorm{ \eta \cV_i^{t+1} + (1 - \eta) \cW_i^t } } \\
		& \overset{(iv)}{\leq}&  ( 1- \alpha) \left(1 + \rho \right) \Exp{ \sqnorm{w_i^t - g_i^t} } \\
		&& \qquad  + ( 1- \alpha) \left(1 + \rho^{-1} \right) \Exp{ \sqnorm{  \rb{ \nabla f_i(x^{t+1}) -  \nabla f_i(x^t) } - \eta \rb{ w_i^t - \nabla f_i(x^t) }  } } \\ 
		&& \qquad + 2 (1 - \alpha) \eta^2 \Exp{\sqnorm{\cV_i^{t+1}}}+  2 (1 - \alpha) (1-\eta)^2 \Exp{ \sqnorm{\cW_i^t}  }\\
		& \overset{(v)}{=}&  (1-\theta) \Exp{ \sqnorm{w_i^t - g_i^t} } \\
		&& \qquad  + \beta \Exp{ \sqnorm{  \rb{ \nabla f_i(x^{t+1}) -  \nabla f_i(x^t) } - \eta \rb{ w_i^t - \nabla f_i(x^t) }  } } \\ 
		&& \qquad + 2 (1 - \alpha) \eta^2 \Exp{\sqnorm{\cV_i^{t+1}}}+ 2(1 - \alpha) (1-\eta)^2 \Exp{ \sqnorm{\cW_i^t}  }\\
		& \overset{(vi)}{\leq}&(1-\theta) \Exp{\sqnorm{w_i^t - g_i^t} } + 2 \beta \eta^2 \Exp{\sqnorm{  w_i^t - \nabla f_i(x^t) }  } \\
		&& \qquad +  2 \beta  \Exp{\sqnorm{  \nabla f_i(x^{t+1}) -  \nabla f_i(x^t) }}  +  2 \ell_i^2 \Exp{\sqnorm{x^{t+1} - x^t}} +  2 \eta^2 \sigma^2 \\
		& \leq & (1-\theta) \Exp{\sqnorm{w_i^t - g_i^t}}  + 2\beta \eta^2 \Exp{\sqnorm{w_i^t - \nabla f_i(x^{t})} } \\
		&& \qquad + \rb{ 2\beta L_i^2 + \ell_i^2} \Exp{ \sqnorm{ x^{t+1} -x^{t} } }  + 2 \eta^2 \sigma^2 ,
	\end{eqnarray*}
where $(i)$ holds by Definition~\ref{def:contractive_compressor}, $(ii)$ follows from \eqref{le:storm_vec_recursion_i}, $(iii)$ holds by unbiasedness of $\cV_i^{t+1}$ and $\cW_i^t$ \eqref{eq:storm_unbiased_V_W}. In $(iv)$ we use Young's inequality twice, in $(v)$ we introduce the notation $\theta  \eqdef 1 - (1-\alpha) (1+\rho)$ and $\beta  \eqdef (1-\alpha) (1+\rho^{-1})$, in $(vi)$ we again use Young's inequality and the bound \eqref{eq:storm_variance} and \eqref{eq:storm_W_variance}. The last step holds by smoothness of $f_i(\cdot)$ (Assumption~\ref{as:main}). The proof is complete by the choice $\rho = \alpha / 2$, which guarantees $1-\theta \leq 1-\alpha/2$, and $2 \beta \leq 4/\alpha$ . 
\end{proof}

\clearpage
\section{Simplified Proof of SGDM: Time Varying Parameters and No Tuning for Momentum Sequence}\label{sec:appendix_mom_simple}
In this section, we give a simplified proof of \algname{SGDM} in the single node setting ($n=1$) without compression ($\alpha = 1$). The following theorem shows that the momentum parameter can be chosen in a parameter agnostic\footnote{ That is, independently of the problem specific parameters} way as $\eta_t = 1 / \sqrt{t+1}$ (or $\eta_t = 1/ \sqrt{T+1}$), instead of being a constant depending on problem parameters as it is suggested in our main Theorem~\ref{thm:main-distrib}. In other words, using \algname{SGDM} with time varying momentum does not introduce any additional tuning of hyper-parameters.  

\begin{theorem}\label{thm:SGDM_simple}
	Let Assumptions~\ref{as:main}, \ref{as:BV} hold. Let $n=1$ and  Algorithm~\ref{alg:EF21-M} run with identity compressor $\cC$, i.e., $\alpha = 1$, and (possibly) time varying momentum $\eta_t \in (0, 1]$ and step-size paramters $\gamma_t = \gamma \eta_t$ with $\gamma \in (0,  \nfr{1}{(3L)} ]$. Let $\hat x^T$ be sampled from the iterates of the algorithm with probabilities $p_t = \eta_t / (\sum_{t=0}^{T-1} \eta_t)$, then
	\begin{eqnarray*}
		\Exp{\sqnorm{\nabla f(\hat x^T)}}  
		&\leq& \frac{2 \Lambda_0 \gamma^{-1} + 2 \sigma^2 \sum_{t=0}^{T-1} \eta_t^2}{\sum_{t=0}^{T-1} \eta_t} ,
	\end{eqnarray*}
where $\Lambda_0  \eqdef f(x^0) - f^* + \gamma \Exp{\sqnorm{v^0 - \nabla f(x^{0})}}$ is the Lyapunov function.  
\end{theorem}	
\begin{proof}
By Lemma~\ref{le:key_HB_recursion} denoting $P_t  \eqdef \Exp{ \sqnorm{v^{t} - \nabla f(x^{t})} } $, $R_t  \eqdef \Exp{\sqnorm{x^t - x^{t+1}}}$,   we have
\begin{eqnarray}\label{eq:P_rec_SGDM_simpl}
	P_{t+1} \leq P_t - \eta_t P_t +  \fr{3 L^2 }{\eta_t } R_t + \eta_t^2 \sigma^2 .
\end{eqnarray}
By descent Lemma~\ref{le:descent}, we have for any $\gamma_t > 0 $
\begin{eqnarray}\label{eq:SGDM_descent_simpl}
	\delta_{t+1} \leq \delta_t - \frac{\gamma_t}{2} \Exp{\sqnorm{\nabla f(x^t)}} -  \fr{1}{2 \gamma_t } \rb{1 - \gamma_t L } R_t + \frac{\gamma_t}{2} P_t ,
\end{eqnarray}
where $\delta_0  \eqdef \Exp{f(x^t) - f^*}$. Define the Lyapunov function as $\Lambda_t = \delta_0 + \gamma P_t$. Then summing up \eqref{eq:SGDM_descent_simpl} with a $\gamma$ multiple of \eqref{eq:P_rec_SGDM_simpl} and noticing that $\gamma_t \leq \gamma$, we get
$$
\Lambda_{t+1} \leq \Lambda_t - \frac{\gamma_t}{2} \Exp{\sqnorm{\nabla f(x^t)}} - \frac{1}{2 \gamma_t}\rb{ 1 -  \gamma L - 6 \gamma^2 L^2} R_t + \gamma \eta_t^2 \sigma^2 .
$$
Since $\gamma \leq 1 / (3 L)$, we have $ 1 -  \gamma L - 6 \gamma^2 L^2\leq 0$, and, therefore, by telescoping we can derive
\begin{eqnarray*}
\Exp{\sqnorm{\nabla f(\hat x^T)}}  &=& \rb{\sum_{t=0}^{T-1} \eta_t }^{-1}\sum_{t=0}^{T-1} \eta_t \Exp{\sqnorm{\nabla f( x^t)}} \\
&\leq& \frac{2 \Lambda_0 \gamma^{-1} + 2 \sigma^2 \sum_{t=0}^{T-1} \eta_t^2}{\sum_{t=0}^{T-1} \eta_t} . 
\end{eqnarray*}
\end{proof}

The above theorem suggests that to ensure convergence, we can select any momentum sequence such that $\sigma^2 \sum_{t=0}^{\infty} \eta_t^2 < \infty$, and $ \sum_{t=0}^{\infty} \eta_t^2 \rightarrow \infty$ for $t\rightarrow\infty$. The parameter $\gamma$, which determines the step-size $\gamma_t = \gamma \eta_t$, should be set to $\gamma = 1 / (3 L)$ (to minimize the upper bound). Let us now consider some special cases. 
\paragraph{Deterministic case.}
If $\sigma = 0$, we can set it to be any constant $\eta_t = \eta \in (0, 1]$ and derive 
$$\Exp{\sqnorm{\nabla f(\hat x^T)}}  \leq \frac{2 \delta_0 }{\gamma \eta T} = \cO\rb{ \frac{L \delta_0 }{ \eta T} }. $$

\paragraph{Stochastic case.}
For $\sigma^2 > 0$, we can select time-varying $\eta_t = \frac{1}{\sqrt{t+1}}$ or constant $\eta_t = \frac{1}{\sqrt{T+1}}$, which gives $\sum_{t=0}^{T-1} \eta_t^2 = \cO\rb{\log(T)}$, and $\sum_{t=0}^{T-1} \eta_t = \Omega\rb{\sqrt{T}}$. Thus
$$
\Exp{\sqnorm{\nabla f(\hat x^T)}}  = \wt \cO\rb{ \frac{L \Lambda_0 + \sigma^2  }{ \sqrt{T}} }. 
$$
Notice that if we set $\eta_t$ as above, we do not need any tuning of momentum parameter. Only tuning of paramter $\gamma$ is required to ensure convergence with optimal dependence on $T$, as in \algname{SGD} without momentum. Of course, this rate is not yet optimal in other parameters, e.g., $\sigma^2$ and $L$. To make it optimal in all problem parameters, we can set $\eta = \max\cb{1, \rb{\frac{L\Lambda_0}{\sigma^2 T}}^{\nfr{1}{2}} }$ similarly to the statement of Theorem~\ref{thm:ef21-sgdm-one-node}.

\clearpage
\section{Revisiting EF14-SGD Analysis under BG and BGS Assumptions}\label{sec:revisiting_EF14}
In this section, we revisit the analysis of the original variant of error feedback (\algname{EF14-SGD}) to showcase the difficulty in avoiding BG/BGS assumptions commonly used in the nonconvex analysis of this variant. In summary, the key reason for BG/BGS assumption is to bound the second term in \eqref{eq:appendix_error_control_EF14_BG} or \eqref{eq:appendix_error_control_EF14_BGS}.  

Recall that \algname{EF14-SGD} has the update rule~\citep{Stich-EF-NIPS2018}
\begin{equation}
	\label{eq:x_update_ef_appendix}
	x^{t+1} = x^t -  g^t, \qquad  g^t = \suminn g_i^t , 
\end{equation}
\begin{align}
	\text{\algname{EF14-SGD}:}\qquad 
	\begin{split}
		e_i^{t+1} &= e_i^{t} +  \gamma \nabla f_i(x^{t}, \xi_i^{t})  -  g_i^t , \\
		g_i^{t+1} &= \cC\rb{  e_i^{t+1}   +  \gamma\nabla f_i(x^{t+1}, \xi_i^{t+1})  } ,
	\end{split} \label{eq:EF14-SGD} 
\end{align}
where $\cb{e_i^t}_{t\geq 0}$ are error/memory sequences with $e_i^0 = 0$ for each $i = 1, \ldots, n$. Let $e^t := \suminn e_i^t$. The proof of this method relies on so called perturbed iterate analysis, for which one defines a "virtual sequence": $\tilde x^t := x^t -  e^t$. Then it is verified by direct substitution that for any $t\geq 0$
$$
\tilde x^{t+1} = \tilde x^t - \gamma \suminn \nabla f_i(x^t, \xi_i^t) . 
$$
If follows from Lemma 9 in \citep{EF_delay_2021} that for any $\gamma \leq \nfr{1}{2L}$ and $t\geq 0$
$$
\Exp{f(\tilde x^{t+1})} \leq \Exp{f(\tilde x^{t})} - \frac{\gamma}{4} \Exp{\sqnorm{\nabla f(x^t)}} + \frac{\gamma L \sigma^2 }{2 n } + \frac{ L^2 }{2} \Exp{\sqnorm{e^t}}.
$$
Telescoping the recursion above and setting $\delta_0 := f(x^0) - f^*$, 
we have
\begin{eqnarray}\label{eq:appendix_EF14-SGD_descent}
\frac{1}{T} \sum_{t=0}^{T-1}\Exp{\sqnorm{\nabla f( x^t)}} \leq \frac{4 \delta_0}{\gamma T } + \frac{2\gamma L \sigma^2 }{n} + 2  L^2 \frac{1}{T} \sum_{t=0}^{T-1} \Exp{\sqnorm{e^t}} .
\end{eqnarray}

Now it remains to bound efficiently the average error term $\Exp{\sqnorm{e^t}} =  \Exp{\sqnorm{\suminn e_i^t}} .$ By Jensen's inequality, we have
\begin{eqnarray*}
\Exp{\sqnorm{\suminn e_i^t}} \leq \suminn \Exp{\sqnorm{ e_i^t}} ,
\end{eqnarray*}
and develop a bound for each $\Exp{\sqnorm{ e_i^t}}$ individually. Denote by $z := e_i^t + \gamma \nabla f_i(x^t, \xi_i^t)$, then
\begin{eqnarray}
\Exp{\sqnorm{e_i^{t+1}}} &\leq& \Exp{\sqnorm{ \cC(z) - z }}  \notag \\
&\leq & (1-\alpha ) \Exp{\sqnorm{  e_i^t + \gamma \nabla f_i(x^t, \xi_i^t) }} \notag \\
&\leq & (1-\alpha )\rb{ 1 + \frac{\alpha}{2} } \Exp{\sqnorm{  e_i^t }} + \rb{1+\frac{2}{\alpha}}  \Exp{\sqnorm{ \gamma \nabla f_i(x^t, \xi_i^t) }} \notag  \\
&\leq & \rb{ 1 - \frac{\alpha}{2} } \Exp{\sqnorm{  e_i^t }} + \frac{3 \gamma^2 }{\alpha}  \Exp{\sqnorm{ \nabla f_i(x^t, \xi_i^t) }} , \label{eq:appendix_error_control_EF14_BG}
\end{eqnarray}
where we used Definition~\ref{def:contractive_compressor} and Young's inequality.
\paragraph{BG asssumption.} If we assume bounded (stochastic) gradients (BG), i.e., $\Exp{\sqnorm{\nabla f_i(x, \xi_i)}} \leq G^2$ for all $i = 1,\ldots, n$, then using \eqref{eq:appendix_error_control_EF14_BG} we can derive 
$$
\frac{1}{T} \sum_{t=0}^{T-1} \Exp{\sqnorm{e^t}} \leq \frac{6\gamma^2 G^2}{\alpha^2 } . 
$$
Combining this bound with \eqref{eq:appendix_EF14-SGD_descent}, we have 
$$
\frac{1}{T} \sum_{t=0}^{T-1}\Exp{\sqnorm{\nabla f( x^t)}} \leq \frac{4 \delta_0}{\gamma T } + \frac{2\gamma L \sigma^2 }{n} +  \frac{12 L^2\gamma^2 G^2}{\alpha^2 }  .
$$
The step-size choice $\gamma = \min\cb{\frac{1}{L}, \rb{ \frac{\delta_0 \alpha^2 }{T L^2 \sigma^2 } }^{\nfr{1}{3}} , \rb{ \frac{n \delta_0 }{T L \sigma^2 } }^{\nfr{1}{2}}  }$, allows us to bound the RHS by $\frac{12 \delta_0}{\gamma T }$, and guarantees   
$$
\Exp{\sqnorm{\nabla f( \hat x^T)}}  = \cO \rb{ \frac{L \delta_0 }{T} + \rb{\frac{L \delta_0  G}{\alpha T}}^{\nfr{2}{3}} + \rb{ \frac{L \delta_0 }{n T  } }^{\nfr{1}{2}}  }, 
$$
or, equivalently,  $T = \cO\rb{ \frac{L \delta_0 }{\varepsilon^2} + \frac{L \delta_0  G}{\alpha \varepsilon^3} +  \frac{L \delta_0 }{n \varepsilon^4  } }  $ sample complexity to find a stationary point. This analysis using BG assumption and derived sample complexity is essentially a simplified version of the one by \citet{Koloskova2019DecentralizedDL}.\footnote{Up to a smoothness constant and the fact that \citet{Koloskova2019DecentralizedDL} works in a more general decentralized setting. }  

\paragraph{BGS asssumption.} If we assume bounded gradient similarity (BGS), i.e., $\suminn \Exp{\sqnorm{\nabla f_i(x) - \nabla f(x)}} \leq G^2$, we can slightly modify the derivation in \eqref{eq:appendix_error_control_EF14_BG} as follows 
\begin{eqnarray}
	\Exp{\sqnorm{e_i^{t+1}}} 
	&\leq & (1-\alpha ) \Exp{\sqnorm{  e_i^t + \gamma \nabla f_i(x^t, \xi_i^t) }} \notag \\
	& = & (1-\alpha ) \Exp{\sqnorm{  e_i^t + \gamma \nabla f_i(x^t) }} + (1-\alpha ) \gamma^2 \Exp{\sqnorm{  \nabla f_i(x^t, \xi_i^t) - \nabla f_i(x^t) }} \notag \\
	&\leq & (1-\alpha )\rb{ 1 + \frac{\alpha}{2} } \Exp{\sqnorm{  e_i^t }} + \rb{1+\frac{2}{\alpha}}  \Exp{\sqnorm{ \gamma \nabla f_i(x^t) }} + \gamma^2 \sigma^2 \notag  \\
	&\leq & \rb{ 1 - \frac{\alpha}{2} } \Exp{\sqnorm{  e_i^t }} + \frac{3 \gamma^2 }{\alpha}  \Exp{\sqnorm{ \nabla f_i(x^t ) }} + \gamma^2 \sigma^2   . \label{eq:appendix_error_control_EF14_BGS}
\end{eqnarray}
Averaging the above inequalities over $i = 1, \ldots, n$ and using BGS assumption, i.e., $\suminn \Exp{\sqnorm{\nabla f_i(x) }} \leq  \sqnorm{\nabla f(x) } + G^2 $ , we can derive via averaging over $t = 0, \ldots, T-1$ 
\begin{eqnarray*}
\frac{1}{T} \sum_{t=0}^{T-1} \Exp{\sqnorm{e^t}} 
&\leq&  \frac{6\gamma^2 }{\alpha^2 } \frac{1}{T} \sum_{t=0}^{T-1} \Exp{\sqnorm{\nabla f(x^t)}} + \frac{6\gamma^2 G^2   }{\alpha^2 } + \frac{2\gamma^2 \sigma^2   }{\alpha },
\end{eqnarray*}
 Combining the above inequality with \eqref{eq:appendix_EF14-SGD_descent}, we have 
$$
\rb{1 - \frac{12 L^2 \gamma^2}{\alpha^2}} \frac{1}{T} \sum_{t=0}^{T-1}\Exp{\sqnorm{\nabla f( x^t)}} \leq \frac{4 \delta_0}{\gamma T } + \frac{2\gamma L \sigma^2 }{n}  + \frac{12  \gamma^2 L^2  G^2  }{\alpha^2 }  + \frac{4 \gamma^2 L^2 \sigma^2 }{\alpha}.
$$
By setting $\gamma = \min\cb{ \frac{\alpha}{4 L}, \rb{\frac{n \delta_0 }{L \sigma^2 T}}^{\nfr{1}{2}}, \rb{\frac{\alpha^2 \delta_0 }{L^2 G^2 T}}^{\nfr{1}{3}}, \rb{\frac{\alpha \delta_0 }{L^2 \sigma^2 T}}^{\nfr{1}{3}} }$, we have $\rb{1 - \frac{12 L^2 \gamma^2}{\alpha^2}} \geq \frac{1}{4}$ , and the RHS is at most $\frac{16 \delta_0}{\gamma T }$. Therefore,  
$$
\Exp{\sqnorm{\nabla f( \hat x^T)}}  = \cO \rb{ \frac{L \delta_0 }{\alpha T} + \rb{\frac{L \delta_0  G}{\alpha T}}^{\nfr{2}{3}} + \rb{\frac{L \delta_0  \sigma}{\sqrt{\alpha} T}}^{\nfr{2}{3}} + \rb{ \frac{L \delta_0 }{n T  } }^{\nfr{1}{2}}  } , 
$$
or, equivalently,  $T = \cO\rb{ \frac{L \delta_0 }{\alpha\varepsilon^2} + \frac{L \delta_0  G}{\alpha \varepsilon^3} + \frac{L \delta_0  \sigma}{\sqrt{\alpha} \varepsilon^3}  +  \frac{L \delta_0 }{n \varepsilon^4  } }  $ sample complexity. Notice that in the single node case ($n=1$), we have $G=0$, and by Young's inequality $\rb{\fr{   L \delta_0 \sigma }{\sqrt{\al}  T } }^{\nfr{2}{3}} \leq \frac{1}{3}\frac{L \delta_0}{\alpha T} + \frac{2}{3} \left(\frac{L\delta_0 \sigma^2}{T}\right)^{\nicefrac{1}{2}}$. Therefore, the above rate recovers the one by \citet{EF_delay_2021} in the single node setting. 

\end{document}